\documentclass[preprint,3p,times]{elsarticle}
\RequirePackage{multirow,booktabs,subfigure,color,array,hhline,makecell}
\usepackage{amssymb}
\usepackage{amsmath}
\usepackage{graphicx}
\usepackage{amsthm}
\usepackage{mathrsfs}
\usepackage{indentfirst}
\usepackage[colorlinks,citecolor=blue,urlcolor=blue]{hyperref}
\allowdisplaybreaks
\usepackage[table]{xcolor}
\usepackage{tikz-network}
\usepackage{pgf}
\usepackage{tikz}
\usepackage{subfigure}
\usetikzlibrary{arrows, decorations.pathmorphing, backgrounds, positioning, fit, petri, automata}
\usetikzlibrary{shadows,arrows,positioning}
\RequirePackage{algorithm,algpseudocode}
\allowdisplaybreaks
\usepackage{changes}

\usepackage{setspace}
\newtheorem{thm}{Theorem}

\newtheorem{lem}{Lemma}
\newtheorem{assum}{Assumption}
\newtheorem{rem}{Remark}
\newtheorem{cor}{Corollary}

\allowdisplaybreaks[4]
\AtBeginDocument{%
	\providecommand\BibTeX{{%
			\normalfont B\kern-0.5em{\scshape i\kern-0.25em b}\kern-0.8em\TeX}}}
\journal{~}
\begin{document}
\begin{frontmatter}
\title{Latent class analysis by regularized spectral clustering}

\author[label1]{Huan Qing\corref{cor1}}
\ead{qinghuan@u.nus.edu\&qinghuan@cqut.edu.cn}
\cortext[cor1]{Corresponding author.}
\address[label1]{School of Economics and Finance, Chongqing University of Technology, Chongqing, 400054, China}
\begin{abstract}
The latent class model is a highly effective tool in the analysis of categorical data from social, psychological, and behavioral sciences, where populations often share hidden common characteristics. In this article, we introduce two new algorithms for estimating the parameters of a latent class model for ordered categorical data with polytomous responses. These algorithms are based on a newly defined regularized Laplacian matrix derived from the response matrix. We provide theoretical convergence rates for our algorithms by considering a sparsity parameter and demonstrate that under a mild condition on data's sparsity, our algorithms yield consistent latent class analysis. Furthermore, we introduce a metric to assess the strength of latent class analysis and develop procedures based on this metric to determine the optimal number of latent classes for real-world ordered categorical data. Extensive simulation experiments demonstrate the efficiency and accuracy of our algorithms, and we demonstrate their practical application to real-world ordered categorical data with promising results.
\end{abstract}
\end{frontmatter}
\section{Introduction}\label{sec1}
Ordered categorical data is prevalent in social, educational, and psychological sciences, such as political surveys, psychological tests, and educational assessments \citep{sloane1996introduction,agresti2012categorical}. This type of data generally involves subjects, items, and the observed responses of subjects to items. For instance, in psychological tests and educational assessments, the subjects are usually individuals, the items are psychological and educational questions, and the responses are the choices made by individuals to each question. For ordered categorical data, generally, the responses have implicit ordering and can be divided into binary responses (also known as dichotomous responses) and polytomous responses. Binary responses typically include yes or no responses in psychological tests and political surveys, right or wrong responses in educational assessments, and so on. Polytomous responses usually consist of never true, rarely true, sometimes true, often true, and always true, as well as strongly disagree, disagree, neutral, agree, and strongly agree in psychological tests; or high school, bachelor’s degree, master’s degree, and doctorate choices in political surveys, and so on. The observed responses are usually denoted by a finite number of non-negative integers and they are recorded by an $N\times J$ observed response matrix $R$ such that $R(i,j)$ ranges in $\{0,1,2,\ldots, M\}$ and $R(i,j)$ denotes subject $i$'s response to item $j$, where $N$ is the number of subjects, $J$ is the number of items, and $m$ denotes the different intensity choices of responses for $m=0,1,2,\ldots, M$. For binary responses, $M$ is 1 and $R\in\{0,1\}^{N\times J}$. For polytomous responses, $M$ is a positive integer which is usually no smaller than 2 and $R\in\{0,1,2,\ldots,M\}^{N\times J}$. For example, for responses like strongly disagree, disagree, neutral, agree, and strongly agree, psychometric researchers usually use 1, 2, 3, 4, and 5 to denote the five choices, respectively, i.e., $M=5$ for this case. Note that for categorical data with choices like A, B, C, or D, male or female, or apple, pear, or orange, such data is unordered categorical data as there is no ordering relationship among these choices. In this paper, we focus on ordered categorical data. For clarity, we will employ categorical data to denote ordered categorical data for the rest of this paper.

Latent class analysis (LCA) is a crucial problem in categorical data analysis across various fields, including machine learning and statistical, social, and psychological analysis \citep{mccutcheon1987latent, hagenaars2002applied, lanza2013latent, vermunt2004latent, lanza2016latent, nylund2018ten}. The objective of LCA is to identify latent classes (or subgroups) in categorical data such that individuals within the same latent class share some common characteristics. For instance, latent classes can represent different types of personalities in psychological tests, such as shy or bold, selfishness or generosity, hesitation or resolution, or different levels of abilities in educational tests, such as basic, intermediate, skill, advanced, or expert levels. These examples demonstrate that it is natural to assume that individuals within the same latent class exhibit some common response patterns.

The latent class model (LCM) \citep{goodman1974exploratory} is a widely used generative model in latent class analysis. Prior methods developed to infer latent classes under LCM include Bayesian inference methods \citep{garrett2000latent,asparouhov2011using,white2014bayeslca,li2018bayesian}, maximum likelihood estimation (MLE) approaches \citep{bakk2016robustness,chen2022beyond}, and tensor-based algorithm \citep{zeng2023tensor}. Other interesting method for categorical data with polytomous responses includes the K-modes algorithm \citep{huang1998extensions} which is model-free and applies Hamming distance to replace Euclidean distance in the K-means algorithm. On the one hand, these previous methods face several limitations. Firstly, none of them consider the influence of sparsity in categorical data, which can be characterized by a parameter in the LCM. Secondly, Bayesian inference and MLE methods are computationally expensive \citep{chen2023spectral}. Thirdly, all methods except the tensor-based algorithm in \citep{zeng2023tensor} lack theoretical convergence guarantees. Fourthly, most methods only consider binary responses and neglect polytomous responses. Finally, there is no metric to evaluate the quality of estimated classes in latent class analysis of real-world categorical data with unknown ground truth. Therefore, it is urgent to establish theoretically-guaranteed, efficient, and easy-to-implement methods for LCM-based latent class analysis, along with criteria to evaluate algorithm performance in categorical data. On the other hand, spectral clustering algorithms based on eigen-decomposition or singular value decomposition (SVD) of certain matrices \citep{ng2001spectral,von2007tutorial} are popular techniques in statistical learning and social network analysis, see \citep{KSB2011,qin2013regularized,SCORE,rohe2016co,qing2023regularized,qing2025discovering,qing2025community,qing2025mixed} and references therein for their good theoretical properties, ease of implementation, and computational efficiency. However, they are rarely used for the problem of latent class analysis. In this article, we take advantage of spectral clustering for the problem of latent class analysis to develop efficient algorithms. Our contributions are threefold:
\begin{itemize}
  \item We propose two new regularized spectral clustering methods for estimating a latent class model in categorical data. The proposed methods utilize the singular value decomposition of a newly defined regularized Laplacian matrix. To the best of our knowledge, we are the first to introduce this newly defined regularized Laplacian matrix to the problem of latent class analysis.
  \item By considering a sparsity parameter under LCM, we establish the error rates of our methods. We show that our methods yield consistent estimations for latent classes and other latent class model parameters under a mild condition on the sparsity of a response matrix.
  \item To measure the quality of latent class analysis, we propose a metric by using the well-known Newman-Girvan modularity \citep{newman2004finding,newman2006modularity} in social network analysis to a matrix computed from the observed response matrix $R$. We then propose some algorithms to estimate the number of latent classes by maximizing the modularity. The effectiveness of our metric is guaranteed by the high accuracies of our algorithms for substantial computer-generated categorical data under LCM.
\end{itemize}

We should emphasize that the proposed spectral clustering algorithms are theoretically sound and may also be applicable in real‑world settings such as recommendation systems, psychological assessment, educational testing, and market segmentation. By efficiently identifying latent classes in ordered categorical data, our methods can help uncover hidden patterns in user preferences, personality traits, student abilities, or consumer behaviors. The ability to handle polytomous responses and provide consistent estimates under mild sparsity conditions makes our algorithms particularly valuable for large-scale surveys and behavioral studies where traditional methods (e.g., MLE) are computationally prohibitive. Furthermore, the incorporation of a modularity-based criterion for selecting the number of classes enhances the practical utility of our approach, enabling researchers to validate and interpret the latent structure in their data without ground truth.

We organize the rest of this article as follows. Section~\ref{sec2} briefly introduces the latent class model for categorical data. Section~\ref{sec3} presents our proposed algorithms and then describes a modularity-based criterion for quantifying the quality of latent class analysis and for selecting the number of latent classes. Section~\ref{sec4} establishes error bounds for our methods. Section~\ref{sec5} presents simulation studies that assess the performance of our methods. Section~\ref{sec6} demonstrates real-data applications to illustrate the practical usefulness of the proposed approaches. Section~\ref{sec7} concludes the article. All proofs are collected in \ref{SecProofs}, and additional algorithmic variants are described in \ref{AlternativeAlgorithms}.
\section{The latent class model (LCM)}\label{sec2}
Here, we briefly introduce the latent class model for categorical data with polytomous responses studied in this article. Let $\mathcal{C}=\{1,2,\ldots,N\}$ represent the collection of indices of the subjects. Suppose that all subjects are divided into $K$ disjoint latent classes, and we use $\mathcal{C}_{k}$ to represent the $k$-th latent class for $k\in[K]$, where $[m]$ for any inter $m$ means the set $\{1,2,\ldots,m\}$ in this article. Thus, $\mathcal{C}=\mathcal{C}_{1}\bigcup\mathcal{C}_{2}\ldots\bigcup\mathcal{C}_{K}$.
Set $Z\in\{0,1\}^{N\times K}$ as the classification matrix such that $Z(i,k)=1$ if subject $i$ belongs to the $k$-th latent class and 0 otherwise for $i\in[N],k\in[K]$. Since LCM assumes that each subject belongs solely to a single latent class, only one entry of $Z(i,:)$ equals 1 while the other (K-1) entries are 0. For convenience, we define $\ell$ as a vector with dimension $N\times 1$ such that $\ell(i)=k$ if the $i$-th subject belongs to the $k$-th latent class for $i\in[N],k\in[K]$. Let $N_{k}$ denote the size of the $k$-th latent class for $k\in[K]$. Thus, $\sum_{k=1}^{K}N_{k}=N$. Let $N_{\mathrm{max}}=\mathrm{max}_{k\in[K]}N_{k}$ and $N_{\mathrm{min}}=\mathrm{min}_{k\in[K]}N_{k}$ denote the sizes of the largest and smallest latent classes, respectively. For convenience, define $\mathcal{I}$ as the index of subject corresponding to $K$ subjects, one from each latent class, i.e., $\mathcal{I}=\{s_{1}, s_{2}, \ldots,s_{K}\}$ with $s_{k}$ being a subject in $\mathcal{C}_{k}$ for $k\in[K]$. For the index set $\mathcal{I}$, we always have $Z(\mathcal{I},:)=I_{K\times K}$, where $Z(\mathcal{I},:)$ is a submatrix of $Z$ for rows in the index set $\mathcal{I}$ (same notation holds for other matrices) and $I_{m\times m}$ is a $m\times m$ identity matrix for any positive integer $m$ in this article.

Define $\Theta$ as a $J\times K$ matrix whose $(j,k)$-th entry $\Theta(j,k)\in[0,M]$ for $j\in[J], k\in[K]$. We call the $\Theta$ item parameter matrix in this article. $\Theta$ captures the conditional responses under the Binomial distribution. After introducing the classification matrix $Z$ and the item parameter matrix $\Theta$, now we are ready to present the latent class model. Given the latent class that subject $i$ belongs to, LCM for categorical data assumes that subject $i$'s response to item $j$ is conditionally independent of a Binomial distribution with probability $\frac{\Theta(j,\ell(i))}{M}$ and $M$ independent trials. That is, for $i\in[N], j\in[J]$, LCM assumes that
\begin{align}\label{LCMBinomial}
\mathbb{P}(R(i,j)=m)=\binom{M}{m}(\frac{\Theta(j,\ell(i))}{M})^{m}(1-\frac{\Theta(j,\ell(i))}{M})^{M-m} \qquad m=0,1,2,\ldots,M,
\end{align}
where $\mathbb{P}(\cdot)$ means the probability and $\binom{M}{m}$ denotes the binomial coefficient $\frac{M!}{m!(M-m)!}$. Thus, $\Theta(j,k)$ is the expected response of a subject in class $k$ to item $j$. The model explicitly incorporates the ordered nature of the categorical responses through the use of the Binomial distribution in Equation (\ref{LCMBinomial}). Specifically, for each subject \(i\) and item \(j\), the response \(R(i,j)\) is modeled as a Binomial random variable with \(M\) trials and success probability \(\frac{\Theta(j,\ell(i))}{M}\). A similar binomial modeling choice has been used recently for categorical data with polytomous responses in \citep{qing2024grade,lyu2025degree,qing2025mixed}.

This construction naturally respects the ordering (i.e., intensity) of the response categories \(\{0, 1, \ldots, M\}\), as the Binomial distribution is defined over ordered integers and its mean \(\Theta(j,\ell(i))\) directly reflects the expected response level. Moreover, the probability mass function \(\mathbb{P}(R(i,j) = m)\) decreases or increases in a structured manner as \(m\) varies, which inherently captures the ordinal nature of the data. Particularly, when $M=1$, we have $R$'s elements range in $\{0,1\}$. For this case, we have $\Theta\in[0,1]^{J\times K}$ and  Equation (\ref{LCMBinomial}) gives $\mathbb{P}(R(i,j)=1)=\Theta(j,\ell(i))$ and $\mathbb{P}(R(i,j)=0)=1-\Theta(j,\ell(i))$. This simple case has been considered in the literature of latent class analysis, see \citep{xu2017aos,zeng2023tensor}. 

Equation (\ref{LCMBinomial}) says that the expected response for the $i$-th subject to the $j$-th item (i.e., the expectation of $R(i,j)$) is $\Theta(j,\ell(i))$ under a Binomial distribution. Under the LCM model, define $\mathscr{R}$ as the expectation of $R$, i.e., $\mathscr{R}\equiv \mathbb{E}(R)$, where $\mathbb{E}(\cdot)$ denotes the expectation. We call $\mathscr{R}$ expectation response matrix (or call it population response matrix occasionally) in this article. Based on Equation (\ref{LCMBinomial}), the population response matrix $\mathscr{R}$  under the LCM model is
\begin{align}\label{R0ZTheta}
\mathscr{R}=Z\Theta',
\end{align}
where $\Theta'$ means the transpose of $\Theta$. Thus, each row of $\mathscr{R}$ (the expected response vector of a subject) equals the $k$-th column of $\Theta$ (as a row vector), where $k$ is the subject's latent class.  Based on Equation (\ref{R0ZTheta}), we see that Equation (\ref{LCMBinomial}) can be simplified as $R(i,j)\sim\mathrm{Binomial}(M,\frac{\mathscr{R}(i,j)}{M})$. It is easy to see that LCM for categorical data only has two model parameters $Z$ and $\Theta$. For brevity, we denote LCM by $\mathrm{LCM}(Z,\Theta)$. Meanwhile, Equation (\ref{R0ZTheta}) also provides us a direct understanding on the item parameter matrix $\Theta$: from Equation (\ref{R0ZTheta}), we have $\sum_{\ell(i)\equiv k}\mathscr{R}(i,j)=\sum_{\ell(i)\equiv k}\Theta(j,k)=N_{k}\Theta(j,k)$, i.e., $\Theta(j,k)$ denotes the averaged response for subjects in latent class $\mathcal{C}_{k}$ to the $j$-th item for $j\in[J], k\in[K]$.
\section{Methods}\label{sec3}
This section introduces the proposed algorithms and then describes a modularity‑based criterion for quantifying the quality of latent class analysis and for selecting the number of latent classes.
\subsection{Algorithms}
In this article, we assume that each latent class has at least one subject. Thus, $Z$'s rank is $K$. To simplify our analysis, we assume that the rank of the $J$-by-$K$ item parameter matrix $\Theta$ is $K$. Based on Equation (\ref{R0ZTheta}), as a result, $\mathscr{R}$'s rank is $K$. Our goal is to estimate the latent class membership vector $\ell$ and the item parameter matrix $\Theta$ from the observed data matrix $R$. The core insight is to use the low-rank structure of the expectation matrix $\mathscr{R} = \mathbb{E}[R]$ under LCM.

However, directly replacing \(\mathscr{R}\) by the observed matrix \(R\) and applying a standard singular value decomposition can be problematic for two reasons. First, individuals often differ substantially in their overall response intensity: some consistently give high ratings (large row sums), others low ones. These global differences can dominate the left singular vectors, obscuring the latent class structure. Second, some subjects may have very small row sums (e.g., all zero responses), making direct normalization unstable. The regularized Laplacian matrix introduced below addresses both issues by first scaling each row of \(R\) to remove the effect of varying row sums, and then adding a regularization term to the scaling factors to prevent instability when a row sum is too small.

To motivate the construction of the regularized Laplacian, consider first the oracle case where the population response
matrix $\mathscr{R}$ is known. Define $\mathscr{D}$ as an $N\times N$ diagonal matrix such that $\mathscr{D}(i,i)=\sum_{j=1}^{J}\mathscr{R}(i,j)$ for $i\in[N]$. Define $\mathscr{D}_{\tau}=\mathscr{D}+\tau I_{N\times N}$, where $\tau\geq0$ is a  regularization parameter (we also call it regularizer or tuning parameter occasionally). Then, let the $N$-by-$J$ population regularized Laplacian matrix $\mathcal{L}_{\tau}$ be defined in the following way:
\begin{align}\label{PopulationL}
\mathcal{L}_{\tau}=\mathscr{D}^{-1/2}_{\tau}\mathscr{R}.
\end{align}

In words, we first compute the row sums of the expected response matrix, regularize them by adding $\tau$, and then row‑standardize $\mathscr{R}$ to obtain the population regularized Laplacian matrix $\mathcal{L}_{\tau}$. Since the diagonal matrix $\mathscr{D}_{\tau}$'s rank is $N$ and the rank of the expected response matrix $\mathscr{R}$ is $K$, $\mathcal{L}_{\tau}$ has a rank $K$. In this article, we assume that $K\ll\mathrm{min}(N, J)$, indicating that $\mathcal{L}_{\tau}$ enjoys a low-dimensional structure and it only has $K$ nonzero singular values, while the other $(\mathrm{min}(N, J)-K)$ singular values are 0. The intuition behind the construction of the newly defined regularized Laplacian matrix \(\mathcal{L}_\tau = \mathscr{D}_\tau^{-1/2} \mathscr{R}\) is inspired by regularized graph Laplacians commonly used in undirected network analysis \citep{qin2013regularized,joseph2016impact,qing2023regularized,qing2025community} and directed network analysis in \citep{rohe2016co} for co-clustering asymmetric relations. In directed graphs, regularization helps handle sparsity and degree heterogeneity by smoothing row and column sums. Similarly, in the context of latent class models for ordered categorical data, we regularize the diagonal degree matrix \(\mathscr{D}\) by adding \(\tau I\) to mitigate the effects of sparse and variable row sums in the expectation matrix \(\mathscr{R}\). This ensures stability in spectral clustering, improves concentration properties, and allows consistent estimation of latent classes even under high sparsity.

The following lemma demonstrates that the SVD of the population regularized Laplacian matrix $\mathcal{L}_{\tau}$ is intimately connected to the classification matrix $Z$ and serves as the foundation for our algorithms. In essence, the left singular vectors $U$ of $\mathcal{L}_{\tau}$ have the property that $U(i,:)=U(\bar{i},:)$ if and only if $i$ and $\bar{i}$ belong to the same latent class. Hence $U$ contains exactly $K$ distinct rows, and applying $K$-means to the rows of $U$ (or to its row normalized version $U_*$) recovers the true class labels perfectly.
\begin{lem}\label{SVDPopulationLtau}
(SVD for $\mathcal{L}_{\tau}$) Under $\mathrm{LCM}(Z,\Theta)$, set $\mathcal{L}_{\tau}=U\Sigma V'$ as the compact SVD of $\mathcal{L}_{\tau}$ such that the $k$-th diagonal entry of the $K$-by-$K$ diagonal matrix $\Sigma$ is the $k$-th largest singular value of $\mathcal{L}_{\tau}$ for $k\in[K]$, $U\in\mathbb{R}^{N\times K}$ and $V\in\mathbb{R}^{J\times K}$ collects the corresponding left and right singular vectors, respectively, and they satisfy $U'U=I_{K\times K}$ and $V'V=I_{K\times K}$, where $\mathbb{R}$ denotes the real number set. Let $U_{*}$ be an $N\times K$ matrix such that its $i$-th row is $U_{*}(i,:)=\frac{U(i,:)}{\|U(i,:)\|_{F}}$ for $i\in[N]$, where $U(i,:)$ means the $i$-th row (same notation holds for other matrices) and $\|\cdot\|_{F}$ for a matrix denotes its Frobenius norm. We have:
\begin{description}
  \item[(1)] $\Theta$ can be rewritten as $\Theta=\mathscr{R}'Z(Z'Z)^{-1}$.
  \item[(2)] $U=ZX$ with $X=U(\mathcal{I},:)$, i.e., $U$ has $K$ distinct rows and $U(i,:)=U(\bar{i},:)$ if $\ell(i)=\ell(\bar{i})$ for $i,\bar{i}\in[N]$.
  \item[(3)] $U_{*}=ZY$ with $Y=U_{*}(\mathcal{I},:)$, i.e., $U_{*}$ has $K$ distinct rows and $U_{*}(i,:)=U_{*}(\bar{i},:)$ if $\ell(i)=\ell(\bar{i})$ for $i,\bar{i}\in[N]$.
  \item[(4)] For $k\in[K], l\in[K]$, and $k\neq l$, we have
  \begin{align}\label{XYDistance}
  \|X(k,:)-X(l,:)\|_{F}=\sqrt{\frac{1}{N_{k}}+\frac{1}{N_{l}}} \mathrm{~and~}\|Y(k,:)-Y(l,:)\|_{F}=\sqrt{2}.
  \end{align}
\end{description}
\end{lem}
The 2nd (and 3rd) statement of Lemma \ref{SVDPopulationLtau} provides a useful fact: for the ideal case that the population response matrix $\mathscr{R}$ is known in advance, after computing the population regularized Laplacian matrix $\mathcal{L}_{\tau}$ and its compact SVD $U\Sigma V'$, applying K-means algorithm on all rows of $U$ (and $U_{*}$) will perfectly recover the classification matrix $Z$, where the exact form of K-means clustering is displayed in Equation (2.2) of \citep{lei2015consistency} and we omit it here. After obtaining $Z$, $\Theta$ can be recovered by setting $\Theta=\mathscr{R}'Z(Z'Z)^{-1}$ according to the 1st statement of Lemma \ref{SVDPopulationLtau}. The above analysis suggests the following two ideal algorithms, Ideal LCA-RSC and Ideal LCA-RSCn, where LCA-RSC is short for latent class analysis by regularized spectral clustering and $n$ means normalized.
\begin{algorithm}
\caption{\textbf{Ideal LCA-RSC}}
\label{alg:IdealRSC}
\begin{algorithmic}[1]
\Require Population response matrix $\mathscr{R}$, number of latent classes $K$, and regularizer $\tau$.
\Ensure Classification matrix $Z$ and item parameter matrix $\Theta$.
\State Calculate $\mathcal{L}_{\tau}$ by Equation (\ref{PopulationL}).
\State Obtain $U\Sigma V'$, the top $K$ SVD of $\mathcal{L}_{\tau}$.
\State Run K-means algorithm on all rows of $U$ with $K$ clusters to obtain $Z$.
\State Recover $\Theta$ by setting $\Theta=\mathscr{R}'Z(Z'Z)^{-1}$.
\end{algorithmic}
\end{algorithm}

\begin{algorithm}
\caption{\textbf{Ideal LCA-RSCn}}
\label{alg:IdealRSCn}
\begin{algorithmic}[1]
\Require $\mathscr{R}, K$, and $\tau$.
\Ensure $Z$ and $\Theta$.
\State Calculate  $\mathcal{L}_{\tau}$ by Equation (\ref{PopulationL}).
\State Obtain $U\Sigma V'$, the top $K$ SVD of $\mathcal{L}_{\tau}$. Then calculate $U_{*}$ from $U$.
\State Run K-means algorithm on all rows of $U_{*}$ with $K$ clusters to obtain $Z$.
\State Recover $\Theta$ by setting $\Theta=\mathscr{R}'Z(Z'Z)^{-1}$.
\end{algorithmic}
\end{algorithm}

For real-world categorical data, we only know the observed response matrix $R$ instead of its expectation $\mathscr{R}$. To estimate the classification matrix $Z$ and the item parameter matrix $\Theta$ from $R$, first we define the regularized Laplacian matrix $L_{\tau}$:
\begin{align}\label{realL}
L_{\tau}=D^{-1/2}_{\tau}R,
\end{align}
where $D_{\tau}=D+\tau I_{N\times N}$, and $D$ is an $N\times N$ diagonal matrix with $D(i,i)=\sum_{j=1}^{J}R(i,j)$ for $i\in[N]$. Here, the regularization parameter $\tau$ is added to the row sums before normalization to prevent the scaling factor $1/\sqrt{D(i,i)+\tau}$ from becoming excessively large when a subject has a very small row sum $D(i,i)$. Without $\tau$ ($\tau=0$), a subject with $D(i,i)=0$ would cause division by zero, and a subject with a tiny positive sum would be scaled up dramatically, amplifying noise. Adding a positive $\tau$ ensures that even the smallest row sums yield a bounded scaling factor, thereby stabilising the singular vectors. Our default $\tau = M\max(N,J)$ (detailed analysis is given after Theorem~\ref{mainRLCM}) provides a moderate regularization that balances stability against signal strength.

We then let $\hat{U}\hat{\Sigma}\hat{V}'$ be the top $K$ SVD of the regularized Laplacian matrix $L_{\tau}$ such that $\hat{\Sigma}$ is a $K\times K$ diagonal matrix containing the top $K$ singular values of $L_{\tau}$, $\hat{U}$ and $\hat{V}$ collect the corresponding left and right singular vectors, respectively, and they satisfy $\hat{U}'\hat{U}=I_{K\times K}$ and $\hat{V}'\hat{V}=I_{K\times K}$. Because $\mathcal{L}_{\tau}$'s rank is $K$ and it is the population version of $L_{\tau}$, we see that $L_{\tau}$'s top $K$ SVD should be a good approximation of $\mathcal{L}_{\tau}$. Let $\hat{U}_{*}$ be an $N\times K$ matrix whose $i$-th row is $\hat{U}_{*}(i,:)=\frac{\hat{U}(i,:)}{\|\hat{U}(i,:)\|_{F}}$ for $i\in[N]$. Then applying K-means to all rows of $\hat{U}$ (and $\hat{U}_{*}$) should return a good estimation of the classification matrix $Z$. The above analysis can be summarized using the following two algorithms, LCA-RSC and LCA-RSCn, which extend the Ideal LCA-RSC and the Ideal LCA-RSCn naturally to the real case. Note that in both algorithms, $\tilde{\Theta}$'s elements may not range in $[0,M]$ while $\Theta$'s entries range in $[0,M]$, therefore we should update $\tilde{\Theta}$'s elements in the range $[0,M]$. What's more, we employ the standard K-means algorithm to cluster the rows of the singular vector matrix for estimating the latent class memberships. It is worth noting that alternative enhanced K-means methods, such as those incorporating meta-heuristic optimization algorithms \citep{al2023enhancement}, could potentially improve clustering performance in certain scenarios, though we leave such extensions for future work.

\begin{algorithm}
\caption{\textbf{LCA-RSC}}
\label{alg:RSC}
\begin{algorithmic}[1]
\Require Observed response matrix $R\in\{0,1,2,\ldots, M\}^{N\times J}$, number of latent classes $K$, and regularizer $\tau$ (as we analyze later, a good default for $\tau$ is $M\mathrm{max}(N,J)$.).
\Ensure Estimated classification matrix $\hat{Z}$ and estimated item parameter matrix $\hat{\Theta}$.
\State Calculate the regularized Laplacian matrix $L_{\tau}$ by Equation (\ref{realL}).
\State Obtain $\hat{U}\hat{\Sigma}\hat{V}'$, the top $K$ SVD of $L_{\tau}$.
\State Run K-means algorithm on all rows of $\hat{U}$ with $K$ clusters to obtain $\hat{Z}$.
\State Set $\tilde{\Theta}=R'\hat{Z}(\hat{Z}'\hat{Z})^{-1}$.
\State For $j\in[J], k\in[K]$, write $\hat{\Theta}$ as $\hat{\Theta}(j,k)=\begin{cases}
0& \mbox{if~}\tilde{\Theta}(j,k)<0,\\
\tilde{\Theta}(j,k), & \mbox{if~}0\leq\tilde{\Theta}(j,k)\leq M,\\
M&\mbox{if~}\tilde{\Theta}(j,k)>M.
\end{cases}.$
\end{algorithmic}
\end{algorithm}

\begin{algorithm}
\caption{\textbf{LCA-RSCn}}
\label{alg:RSCn}
\begin{algorithmic}[1]
\Require $R, K$, and $\tau$ (the default $\tau$ is $M\mathrm{max}(N,J)$.).
\Ensure $\hat{Z}$ and $\hat{\Theta}$.
\State Calculate  $L_{\tau}$ by Equation (\ref{realL}).
\State Obtain $\hat{U}\hat{\Sigma}\hat{V}'$, the top $K$ SVD of $L_{\tau}$. Then calculate $\hat{U}_{*}$ from $\hat{U}$.
\State Run K-means algorithm on all rows of $\hat{U}_{*}$ with $K$ clusters to obtain $\hat{Z}$.
\State Obtain $\hat{\Theta}$ using Steps 4-5 of Algorithm \ref{alg:RSC}.
\end{algorithmic}
\end{algorithm}

For both algorithms, the complexities of the SVD step, the K-means step, and the last two steps are $O(\mathrm{max}(N^{2}, J^{2})K)$, $O(NlK^{2})$, and $O(JNK)$, respectively, where $l$ is the number of iterations in K-means and we set it as 100 in this article. Since $K\ll\mathrm{min}(N, J)$, as a result, both algorithms have a complexity of $O(\mathrm{max}(N^{2}, J^{2})K)$. 

The proposed spectral clustering algorithms (LCA-RSC and LCA-RSCn) exhibit a computational complexity of \(O(\max(N^2, J^2)K)\), which is dominated by the truncated SVD step. This is notably efficient compared to the  expectation-maximization (EM) algorithm, which typically requires \(O(NJK)\) per iteration and may need many iterations to converge, often with no global convergence guarantee and a tendency to settle into local optima \citep{chen2023spectral,lyu2025degree}. Bayesian methods, such as those using Markov Chain Monte Carlo, generally involve even higher per-iteration costs (often \(O(NJK)\) as well) and require a large number of samples for posterior inference, leading to an overall complexity that is impractical for large-scale data \citep{chen2023spectral,lyu2025degree}. Moreover, unlike the EM and Bayesian methods, our algorithms enjoy deterministic convergence with theoretical guarantees under mild sparsity conditions, as established in Theorem \ref{mainRLCM} and Corollary \ref{Corollary} given in the next section. Thus, the proposed methods offer a computationally efficient and provably consistent alternative to traditional iterative methods for latent class analysis.

We should also emphasize that our proposed methods, LCA-RSC and LCA-RSCn, do not impose any explicit assumptions on the spacing or thresholds between the ordered categories. The latent class model (LCM) defined in Equation (\ref{LCMBinomial}) assumes that the responses follow a Binomial distribution conditional on the latent class, where the parameter \(\Theta(j,k)\) represents the expected response for item \(j\) in class \(k\). This formulation inherently accommodates any ordinal structure without requiring equal intervals or specific thresholds between categories. The algorithms rely solely on the spectral properties of the regularized Laplacian matrix derived from the response matrix \(R\), and do not depend on the specific numerical encoding of the categories beyond their ordinal nature. Consequently, the algorithms are robust to violations of equal spacing or threshold assumptions, as they do not utilize such information. For completeness, \ref{AlternativeAlgorithms} describes several additional spectral variants. Their performance is similar to that of LCA-RSC and LCA-RSCn, so the main discussion concentrates on these two primary algorithms.

\subsection{Quantifying the quality of latent class analysis}
Our algorithms will normally be applied to study the latent classes for real-world categorical data for which the ground-truth latent classes are usually unknown in advance. Recall that our algorithms can always return some partitions of all subjects into different classes, this naturally leads to the following problem: is there a way to tell us whether the estimated latent classes returned by our algorithms are good or not? Meanwhile, when there are many different algorithms in latent class analysis, we also want to compare their performances and choose the one that returns the best partition. Furthermore, the number of latent classes, $K$, is usually unknown for real data, which raises the following problem: can we design an efficient approach to estimate $K$?

To answer these problems, we now focus on the binary data case to simplify our analysis. For this case, $R(i,j)\in\{0,1\}$ for $i\in[N], j\in[J]$. Define $A$ as $A=RR'$ and call it the adjacency matrix in this article. From the definition of the adjacency matrix $A$, we have the following two facts:
\begin{itemize}
  \item $A(i,i)=\sum_{j=1}^{J}R(i,j)$, i.e., the summation of responses of subject $i$ to all items for $i\in[N]$.
  \item For $i\neq\bar{i}$, $A(i,\bar{i})=\sum_{j=1}^{J}R(i,j)R(\bar{i},j)$, i.e., the number of common responses between the two subjects $i, \bar{i}$.
\end{itemize}

Suppose that there are three subjects $i, \bar{i},$ and $\tilde{i}$ such that subject $i$ and subject $\bar{i}$ are in the same latent class while subject $\tilde{i}$ is from a different latent class. Because subjects in the same latent class usually have similar response patterns, generally, subjects $i$ and $\bar{i}$ have more common responses than that of subjects $i$ and $\tilde{i}$. For the categorical data case, we can also define the adjacency matrix $A$ as $RR'$, and a further understanding of $A$ is similar to that of the binary data case. Therefore, such adjacency matrix $A$ naturally forms an assortative weighted network \citep{newman2003mixing} in which nodes within the same class connect more than nodes across classes. Based on this observation, it is natural to use the Newman-Girvan modularity \citep{newman2004finding,newman2006modularity} to measure the quality of class partition for latent class analysis using the matrix $A$. Let $d_{i}=\sum_{j=1}^{N}A(i,j)$ for $i\in[N]$. Let $\omega=\frac{1}{2}\sum_{i=1}^{N}d_{i}$. Suppose that $\hat{Z}$ is an $N$-by-$k$ matrix obtained by applying a latent class analysis algorithm $\mathcal{M}$ on the observed response matrix $R$ with $k$ latent classes. The Newman-Girvan modularity is calculated as
\begin{align}\label{Modularity}
Q_{\mathcal{M}}(k)=\frac{1}{2\omega}\sum_{i=1}^{N}\sum_{j=1}^{N}(A(i,j)-\frac{d_{i}d_{j}}{2\omega})\hat{Z}(i,:)\hat{Z}'(j,:).
\end{align}

Intuitively, $Q_{\mathcal{M}}(k)$ measures how well the partition $\hat{Z}$ (obtained with $k$ classes) matches the weighted network $A=RR'$, where a higher value indicates that subjects within the same class have more common responses than would be expected by chance. Thus, we always prefer a larger value of modularity since it usually indicates a better quality of latent class partitions \citep{newman2006modularity}. After defining the modularity, to infer the number of latent classes $K$, the idea introduced in \cite{nepusz2008fuzzy} is adopted. In detail, we choose the $k$ that maximizes the modularity defined in Equation (\ref{Modularity}), i.e.,
\begin{align}\label{EstimateK}
\hat{K}_{\mathcal{M}}=\mathrm{arg~max}_{k\in[\mathrm{min}(N, J)]}Q_{\mathcal{M}}(k).
\end{align}

It should be noted that the estimated classification matrix $\hat{Z}$ in Equation (\ref{Modularity}) is returned by an arbitrary algorithm $\mathcal{M}$, thus the estimated number of latent classes in Equation (\ref{EstimateK}) can be different for different algorithms. For convenience, when algorithm $\mathcal{M}$ is applied to Equation (\ref{EstimateK}), we call our approach for estimating $K$ as K$\mathcal{M}$.

We emphasize that the modularity-based criterion for selecting $K$ is a heuristic device. As shown in Experiment~4 of Section~\ref{sec5}, it works well when $K\ge 2$ but fails for $K=1$, which is expected because the criterion presupposes a nontrivial partition. It also lacks theoretical guarantees for estimation consistency of $K$. We therefore present this approach as a practically useful but preliminary tool for model selection. Developing rigorous methods remains an important direction for future work as discussed in the concluding section.
\section{Main results}\label{sec4}
In this section, we show the estimation consistencies of LCA-RSC and LCA-RSCn under the latent class model. Specifically, we derive asymptotic error bounds for both the clustering accuracy and the estimation of the item parameter matrix  under a mild condition as $N$ and/or $J$ go to infinity.

Before presenting our theoretical results, we introduce the sparsity scaling. By Equation (\ref{LCMBinomial}), we see that $\mathbb{P}(R(i,j)=0)=(1-\frac{\Theta(j,\ell(i))}{M})^{M}$ for $i\in[N], j\in[J]$, which implies that the probability of $R(i,j)$ equals to 0 increases if we decrease $\Theta(j,\ell(i))$. For convenience, we set $\rho=\mathrm{max}_{j\in[J],k\in[K]}\Theta(j,k)$ and call $\rho$ the sparsity parameter in this article since it controls the number of 0s in $R$. Note that even for the case that 0s in the observed response matrix $R$ denotes the lowest intensity responses, $\rho$ still controls $R$'s sparsity (i.e., number of 0s).

Since $\Theta$'s elements range in $[0, M]$ under $\mathrm{LCM}(Z,\Theta)$, we see that the sparsity parameter $\rho$ ranges in $(0,M]$. In this article, we allow $\rho$ to go to zero and we will let it enter the theoretical bounds to study its influence on the performances of LCA-RSC and LCA-RSCn. For convenience, let $B=\frac{\Theta}{\rho}$. We see that $\mathrm{max}_{j\in[J],k\in[K]}B(j,k)=1$.

The following assumption controls the sparsity of a response matrix for our theoretical analysis.
\begin{assum}\label{Assum1}
 $\rho\mathrm{max}(N,J)\geq M^{2}\mathrm{log}(N+J)$.
\end{assum}

This assumption serves as a sufficient condition to control the tail behavior of the error term \(\|L_\tau - \mathcal{L}_\tau\|\). This condition ensures that the response matrix is not too sparse, thereby enabling concentration of the regularized Laplacian around its expectation. While this condition is technically necessary for the proof, it is relatively mild in practice. For instance, in our simulations (Experiments 2 and 3), the algorithms perform well whenever \(\rho\) exceeds the threshold implied by Assumption \ref{Assum1}, and deteriorate otherwise. This indicates that the assumption is not overly restrictive and aligns well with practical scenarios where the number of observations is large and the responses are not excessively sparse. Further understanding of this assumption can be found in Remarks \ref{remAssum0} and \ref{remarkAssum1}.

The following lemma bounds $\|L_{\tau}-\mathcal{L}_{\tau}\|$ under LCM.
\begin{lem}\label{boundLLCM}
Under $\mathrm{LCM}(Z,\Theta)$, suppose Assumption \ref{Assum1} is satisfied, with probability at least $1-o((N+J)^{-3})$,
\begin{align*}
\|L_{\tau}-\mathcal{L}_{\tau}\|=O((1+\sqrt{\frac{MN}{\tau+\delta_{\mathrm{min}}}})\sqrt{\frac{\rho\mathrm{max}(N, J)\mathrm{log}(N+J)}{\tau+\delta_{\mathrm{min}}}}),
\end{align*}
where $\delta_{\mathrm{min}}=\mathrm{min}_{i\in[N]}\mathscr{D}(i,i)$ and $\|\cdot\|$ for a matrix denotes its spectral norm.
\end{lem}

\begin{figure}
\centering
\includegraphics[width=0.75\textwidth]{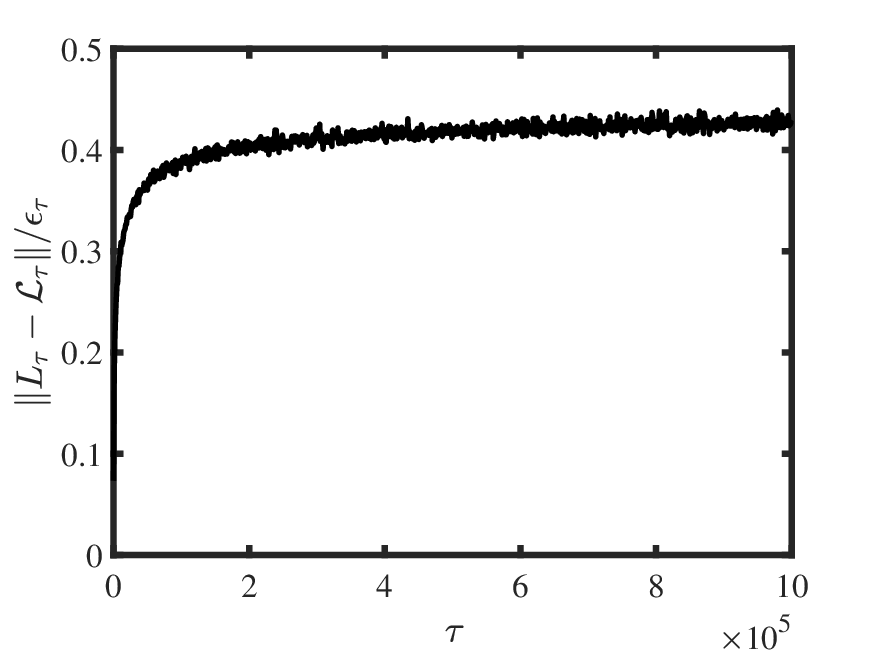}
\caption{$\frac{\|L_{\tau}-\mathcal{L}_{\tau}\|}{\epsilon_{\tau}}$ against the regularizer $\tau$.}
\label{Lepsilon} %% label for entire figure
\end{figure}
Set $\epsilon_{\tau}=(1+\sqrt{\frac{MN}{\tau+\delta_{\mathrm{min}}}})\sqrt{\frac{\rho\mathrm{max}(N, J)\mathrm{log}(N+J)}{\tau+\delta_{\mathrm{min}}}}$, which can be seen as a function of the regularization parameter $\tau$. It is easy to see that when all model parameters are fixed, $\epsilon_{\tau}$ decreases to zero as $\tau$ increases to infinity. In fact, $\|L_{\tau}-\mathcal{L}_{\tau}\|$ under LCM is usually smaller than $\epsilon_{\tau}$. To explain this statement, we provide a toy example in Figure \ref{Lepsilon} by plotting $\frac{\|L_{\tau}-\mathcal{L}_{\tau}\|}{\epsilon_{\tau}}$ against the regularizer $\tau$. For this figure, we let $N=500, J=200, K=3, M=5, \rho=1$, all entries of $\Theta$ be random values from a Uniform distribution on $(0,1)$, and each subject belong to one of the three latent classes with equality probability. After setting $Z$ and $\Theta$, we can generate $R$ from LCM. Then we can plot $\frac{\|L_{\tau}-\mathcal{L}_{\tau}\|}{\epsilon_{\tau}}$ against $\tau$. From Figure \ref{Lepsilon}, we observe that $\frac{\|L_{\tau}-\mathcal{L}_{\tau}\|}{\epsilon_{\tau}}$ is always smaller than 0.5, which confirms our statement. We also observe that $\frac{\|L_{\tau}-\mathcal{L}_{\tau}\|}{\epsilon_{\tau}}$ increases as $\tau$ increases, which means that $\epsilon_{\tau}$ represents a good order of $\|L_{\tau}-\mathcal{L}_{\tau}\|$.

Define $\hat{\mathcal{C}}=\hat{\mathcal{C}}_{1}\bigcup\hat{\mathcal{C}}_{2}\ldots\bigcup\hat{\mathcal{C}}_{K}$ as the set collecting all estimated partitions, where $\hat{\mathcal{C}}_{k}=\{i: \hat{Z}(i,k)=1 \mathrm{~for~}i\in[N]\}$ for $k\in[K]$. To evaluate the difference between the ground-truth classification set $\mathcal{C}$ and the estimated classification set $\hat{\mathcal{C}}$, we consider the \emph{Clustering error} used in \citep{joseph2016impact}. This measure is defined as
\begin{align}\label{clusteringerror}
\hat{f}=\mathrm{min}_{\pi\in S_{K}}\mathrm{max}_{k\in[K]}\frac{|\mathcal{C}_{k}\cap \mathcal{\hat{\mathcal{C}}}^{c}_{\pi(k)}|+|\mathcal{C}^{c}_{k}\cap \mathcal{\hat{\mathcal{C}}}_{\pi(k)}|}{N_{K}},
\end{align}
where $S_{K}$ means all permutations of $\{1,2,\ldots, K\}$ and the superscript $c$ means the complementary sets. According to the reference \citep{joseph2016impact}, we know that the Clustering error $\hat{f}$ quantifies the maximum proportion of subjects in the symmetric difference of $\mathcal{C}_{k}$ and $\hat{\mathcal{C}}_{\pi(k)}$.

Our main result, which bounds the Clustering errors of LCA-RSC and LCA-RSCn, is presented below. 
\begin{thm}\label{mainRLCM}
	Under $\mathrm{LCM}(Z,\Theta)$, when Assumption \ref{Assum1} is satisfied, with probability at least $1-o((N+J)^{-3})$,
\begin{align*}
&\hat{f}=O((1+\frac{MN}{\tau+\delta_{\mathrm{min}}}+2\sqrt{\frac{MN}{\tau+\delta_{\mathrm{min}}}})\frac{K^{2}(\tau+\delta_{\mathrm{max}})N_{\mathrm{max}}\mathrm{max}(N,J)\mathrm{log}(N+J)}{(\tau+\delta_{\mathrm{min}})\rho\sigma^{2}_{K}(B)N^{2}_{\mathrm{min}}}),\\
&\frac{\|\hat{\Theta}-\Theta\|_{F}}{\|\Theta\|_{F}}=O(\sqrt{\frac{K~\mathrm{max}(N,J)\mathrm{log}(N+J)}{\rho N_{\mathrm{min}}\sigma^{2}_{K}(B)}}),
\end{align*}
where $\delta_{\mathrm{max}}=\mathrm{max}_{i\in[N]}\mathscr{D}(i,i)$ and $\sigma_{K}(B)$ denotes $B$'s $K$-th largest singular value.
\end{thm}
Note that LCA-RSC and LCA-RSCn have the same theoretical error rates. Theorem \ref{mainRLCM} guarantees that, under the sparsity condition of Assumption~\ref{Assum1}, the clustering error $\hat f$ and the relative estimation error $\|\hat{\Theta}-\Theta\|_F/\|\Theta\|_F$ become small when the numbers of subjects $N$ and items $J$ are sufficiently large, provided that the sparsity parameter $\rho$ is not too small and that two mild conditions hold: the latent classes are well separated (the smallest singular value $\sigma_K(B)$ of the normalized item parameter matrix $B$ is bounded away from zero), and each class contains enough subjects (the smallest class size $N_{\min}$ is not too small). This theorem also says that the error rates are smaller when we increase the sparsity parameter $\rho$. We also find that if the minimum size of a latent class (and $\sigma_{K}(B)$) is too small, then the error rates are large, i.e., a harder scenario for latent class analysis. Sure, Theorem \ref{mainRLCM} also works for the binary response data case by simply setting $M=1$.

\textbf{Choice of $\tau$:} Theoretical bounds in Theorem \ref{mainRLCM} depends on the regularization parameter $\tau$ through the term $\varrho_{\tau}=(1+\frac{MN}{\tau+\delta_{\mathrm{min}}}+2\sqrt{\frac{MN}{\tau+\delta_{\mathrm{min}}}})\frac{\tau+\delta_{\mathrm{max}}}{\tau+\delta_{\mathrm{min}}}$. Because $0<\delta_{\mathrm{min}}\leq\delta_{\mathrm{max}}=\mathrm{max}_{i\in[N]}\sum_{j=1}^{J}\mathscr{R}(i,j)\leq\rho J\leq M\mathrm{max}(N,J)$, Theorem \ref{mainRLCM} implies that increasing $\tau$ decreases the upper bounds of error rates. Suppose that $\tau\geq M\mathrm{max}(N,J)$, we have $\frac{MN}{\tau+\delta_{\mathrm{min}}}\leq 1,\frac{\tau+\delta_{\mathrm{max}}}{\tau+\delta_{\mathrm{min}}}=O(1)$, and $\varrho_{\tau}=O(1)$, which suggests that the results in Theorem \ref{mainRLCM} can be simplified as $\hat{f}=O(\frac{K^{2}N_{\mathrm{max}}\mathrm{max}(N,J)\mathrm{log}(N+J)}{\rho\sigma^{2}_{K}(B)N^{2}_{\mathrm{min}}})$. On the one hand, we cannot let $\tau$ too large because $\|L_{\tau}\|=\|D^{-1/2}_{\tau}R\|\leq\|D^{-1/2}_{\tau}\|\|R\|=\frac{\|R\|}{\sqrt{\tau+\mathrm{min}_{i\in[N]}D(i,i)}}$, which suggests that $L_{\tau}$'s top-$K$ singular values are close to zero for a too large $\tau$ and this causes that $L_{\tau}$'s top-$K$ singular vectors may lose valuable information about latent classes. On the other hand, by analyzing the form of $\varrho_{\tau}$, we find that though a larger $\tau$ is always preferred either for the case $\delta_{\mathrm{min}}\ll\delta_{\mathrm{max}}$ or the case $\delta_{\mathrm{min}}=O(\delta_{\mathrm{max}})$, continuing to increase the value of $\tau$ will not significantly reduce $\varrho_{\tau}$ since $\varrho_{\tau}$ is always at the order $O(1)$ when $\tau\geq M\mathrm{max}(N, J)$, i.e., $\varrho_{\tau}$ is insensitive for large $\tau$. Therefore, a moderate value of the regularizer $\tau$ is preferred. In numerical studies, we find that setting $\tau=M\mathrm{max}(N, J)$ provides satisfactory results.

From now on, unless specified, we set $\tau=M\mathrm{max}(N, J)$. If we consider more conditions on model parameters, Theorem \ref{mainRLCM} can be further simplified using the following corollary.
\begin{cor}\label{Corollary}
Under $\mathrm{LCM}(Z,\Theta)$, suppose that Assumption \ref{Assum1} holds. If $K=O(1), N_{\mathrm{min}}=O(\frac{N}{K}), N_{\mathrm{max}}=O(\frac{N}{K}), J=O(N)$, and $\sigma^{2}_{K}(B)=O(J)$, then with probability at least $1-o((N+J)^{-3})$,
\begin{align*}
&\hat{f}=O(\frac{\mathrm{log}(N)}{\rho N}) \mathrm{~and~}\frac{\|\hat{\Theta}-\Theta\|_{F}}{\|\Theta\|_{F}}=O(\sqrt{\frac{\mathrm{log}(N)}{\rho N}}).
\end{align*}
\end{cor}
In Corollary \ref{Corollary}, all conditions are mild: $K=O(1)$ means that the number of latent classes is a small positive integer, $N_{\mathrm{min}}=O(\frac{N}{K})$ and $N_{\mathrm{max}}=O(\frac{N}{K})$ means that the sizes of each latent class are in the same order, $J=O(N)$ means that $J$ is in the same order as $N$, and $\sigma^{2}_{K}(B)=O(J)$ is also mild since $B$ is a $J$-by-$K$ matrix and $B$'s maximum entry is 1. By Corollary \ref{Corollary}, it is clear to see that the error rates tend to zero as the number of subjects $N$ goes to infinity. Thus, LCA-RSC and LCA-RSCn are consistent in latent class analysis. Furthermore, Corollary \ref{Corollary} further suggests that to achieve significantly smaller error rates, $\rho$ should decrease at a rate slower than $\frac{\mathrm{log}(N)}{N}$, aligning with the sparsity requirement in Assumption \ref{Assum1}. 

\section{Simulation studies}\label{sec5}
In this section, we conduct simulation studies to evaluate the proposed methods under various settings. Meanwhile, all experimental studies in this article are reported using MATLAB R2024b. Before starting our experimental studies, we provide several baseline methods and some criteria for measuring the performance of different approaches.
\subsection{Baseline methods}
In our experimental studies, we compare our methods with the K-modes (KM for short) algorithm \citep{huang1998extensions} which does not rely on any statistical model for categorical data. For comparison, we also provide four alternative spectral clustering algorithms LCA-RSCORS, LCA-PCA, LCA-RMK, and LCA-RLMK that can also fit the latent class model in \ref{AlternativeAlgorithms}. Other widely used approaches for latent class analysis, such as EM-based methods, Bayesian inference, and tensor-based algorithms, are not included in our simulation study for comparison. These methods were primarily developed for binary responses, whereas our work focuses on ordered polytomous responses with $M\ge 2$. Directly applying them to our setting would require non‑trivial extensions that are beyond the scope of this paper. We therefore restrict our numerical comparisons to the K‑modes algorithm and to several spectral variants derived from our framework.
\subsection{Evaluation metrics}
\textbf{For the task of classification when the true classification matrix $\mathbf{Z}$ is known}, four metrics are considered: Clustering error $\hat{f}$ computed by Equation (\ref{clusteringerror}), Hamming error \citep{SCORE}, normalized mutual
information (NMI) \citep{strehl2002cluster,danon2005comparing,bagrow2008evaluating,luo2017community}, and adjusted rand index (ARI) \citep{luo2017community,hubert1985comparing,vinh2009information}. Hamming error is defined as $\frac{\mathrm{min}_{\mathcal{P}\in\mathcal{P}_{K}}|i: \hat{Z}(i,:)\neq (Z\mathcal{P})(i,:)\mathrm{~for~}i\in[N]|}{N}$,
where $\mathcal{P}_{K}$ is the set containing all $K\times K$ permutation matrices. Hamming error ranges in $[0,1]$, and a smaller Hamming error means a better estimation. The exact forms of NMI and ARI can be found in Equations (5) and (6) in \citep{qing2023community}. NMI ranges in [0,1], ARI ranges in [-1,1], and both metrics are the larger the better.

\textbf{For the task of classification when the true classification matrix $\mathbf{Z}$ is unknown}, we use the Newman-Girvan modularity calculated by Equation (\ref{Modularity}) using the adjacency matrix $A=RR'$ to measure the quality of latent class analysis.

\textbf{For the task of estimating $\mathbf{\Theta}$}, we use the Relative $l_{1}$ error defined as $\mathrm{min}_{\mathcal{P}\in\mathcal{P}_{K}}\frac{\|\hat{\Theta}-\Theta\mathcal{P}\|_{1}}{\|\Theta\|_{1}}$ and the Relative $l_{2}$ error defined as $\mathrm{min}_{\mathcal{P}\in\mathcal{P}_{K}}\frac{\|\hat{\Theta}-\Theta\mathcal{P}\|_{F}}{\|\Theta\|_{F}}$ to measure the performance, where $\|\cdot\|_{1}$ for a matrix means its $l_{1}$ norm. Both measures are the smaller the better.

\textbf{For the task of estimating the number of latent classes } $\mathbf{K}$, we use Accuracy rate to measure the performances of all K$\mathcal{M}$ methods. This metric is defined as the fraction of times a method correctly determines $K$. Recall that this article develops six algorithms LCA-RSC, LCA-RSCn, LCA-RSCORS, LCA-PCA, LCA-RMK, and LCA-RLMK for latent class analysis. As a result, there are six algorithms KLCA-RSC, KLCA-RSCn, KLCA-RSCORS, KLCA-PCA, KLCA-RMK, and KLCA-RLMK for determining $K$.
\subsection{Synthetic categorical data}\label{sec5Synthetic}
\begin{figure}
\centering
%\resizebox{\columnwidth}{!}{
\subfigure[]{\includegraphics[width=0.4\textwidth]{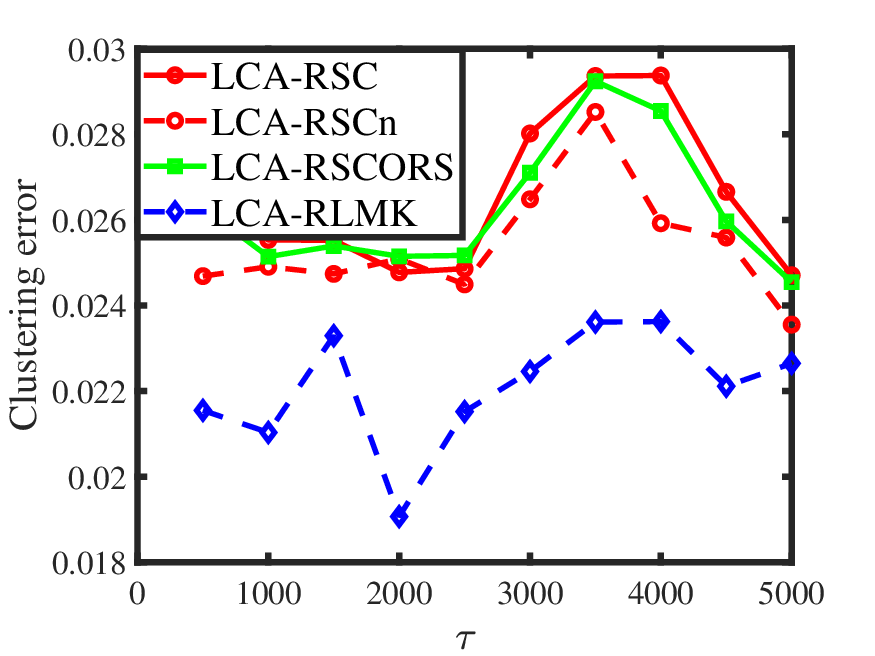}}
\subfigure[]{\includegraphics[width=0.4\textwidth]{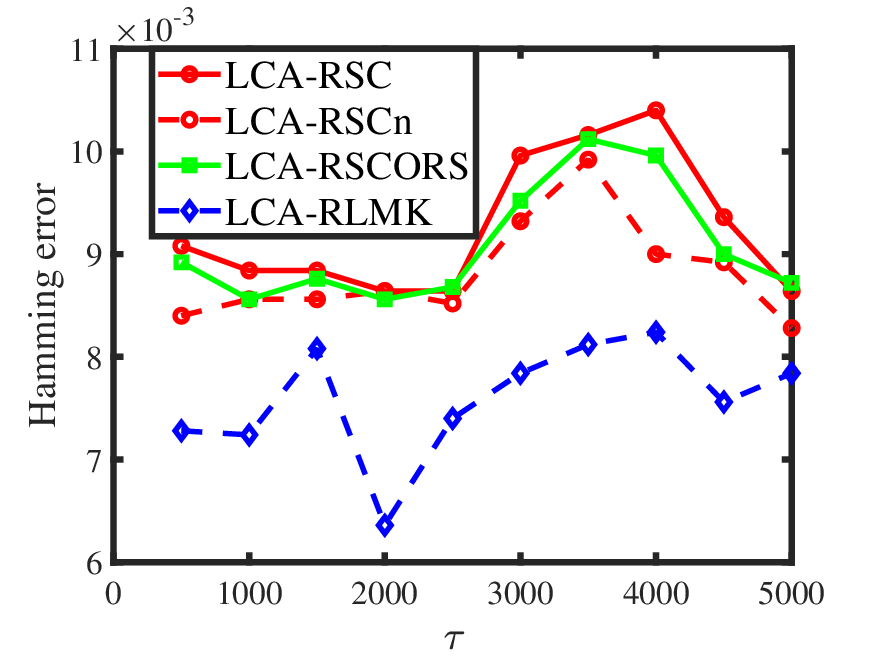}}
\subfigure[]{\includegraphics[width=0.4\textwidth]{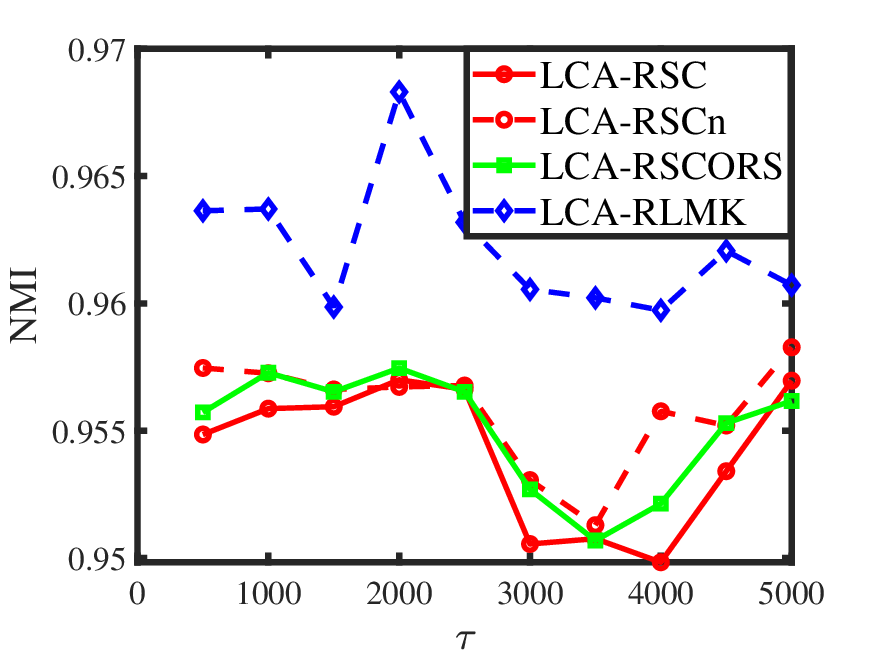}}
\subfigure[]{\includegraphics[width=0.4\textwidth]{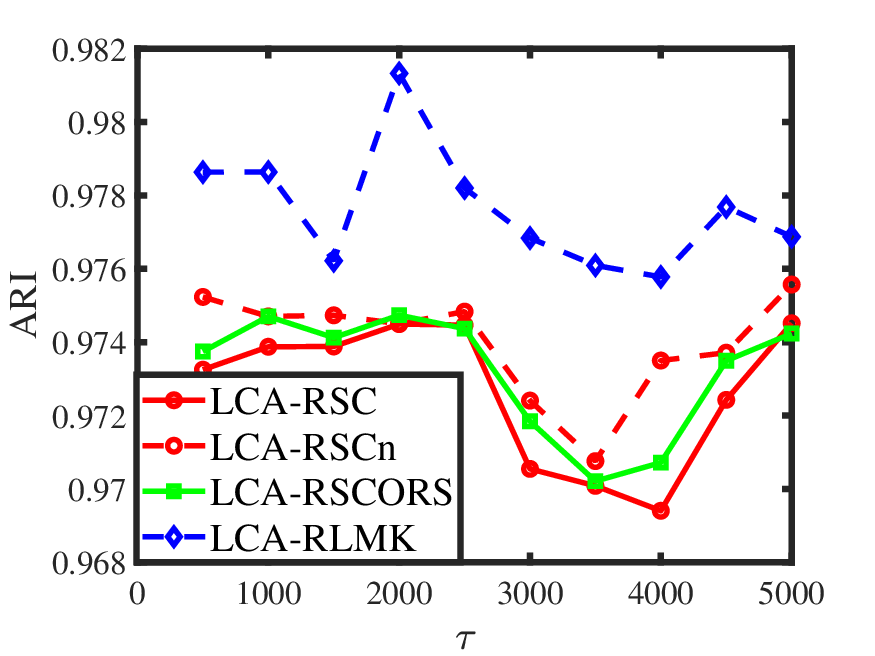}}
\subfigure[]{\includegraphics[width=0.4\textwidth]{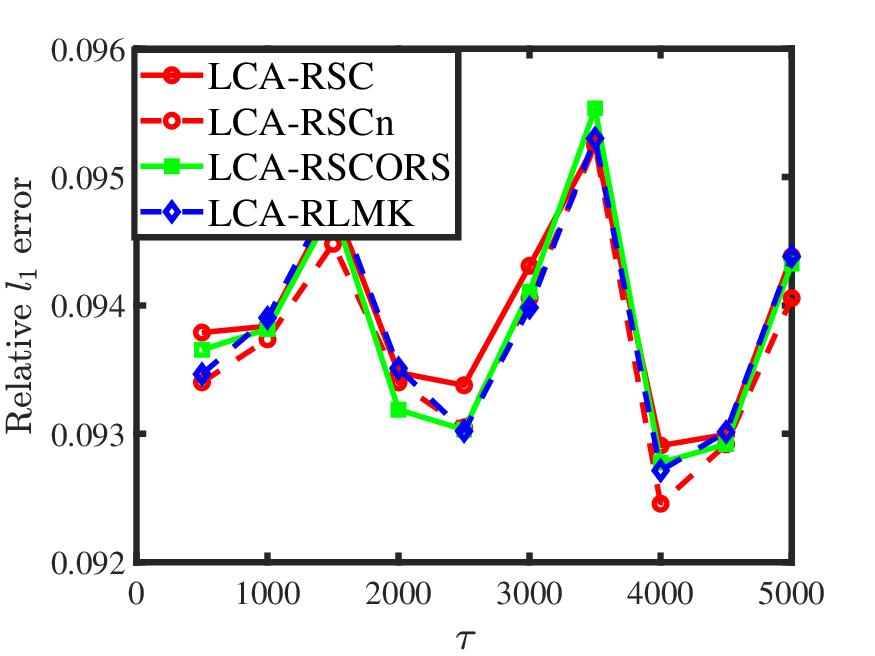}}
\subfigure[]{\includegraphics[width=0.4\textwidth]{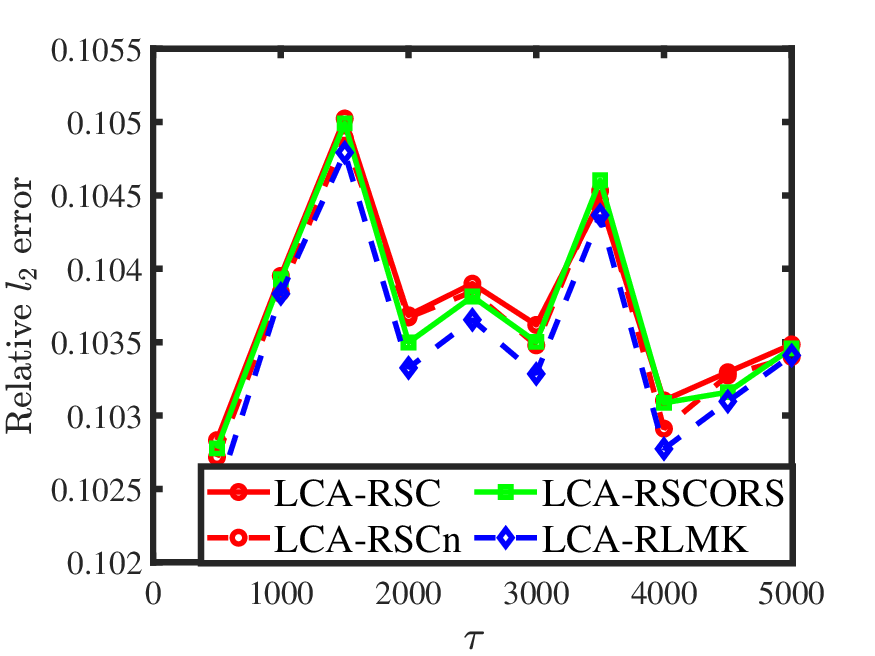}}
\subfigure[]{\includegraphics[width=0.4\textwidth]{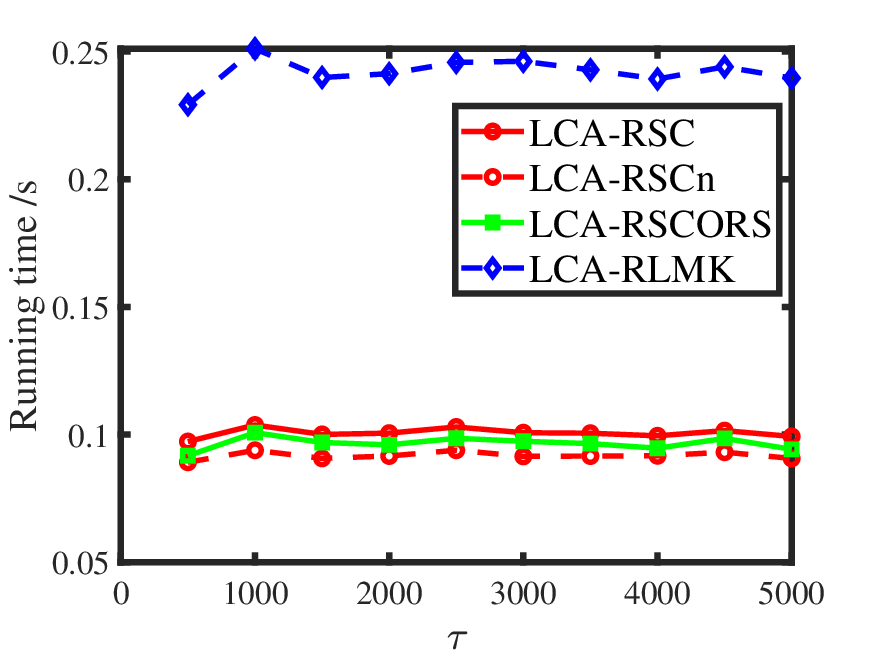}}
\subfigure[]{\includegraphics[width=0.4\textwidth]{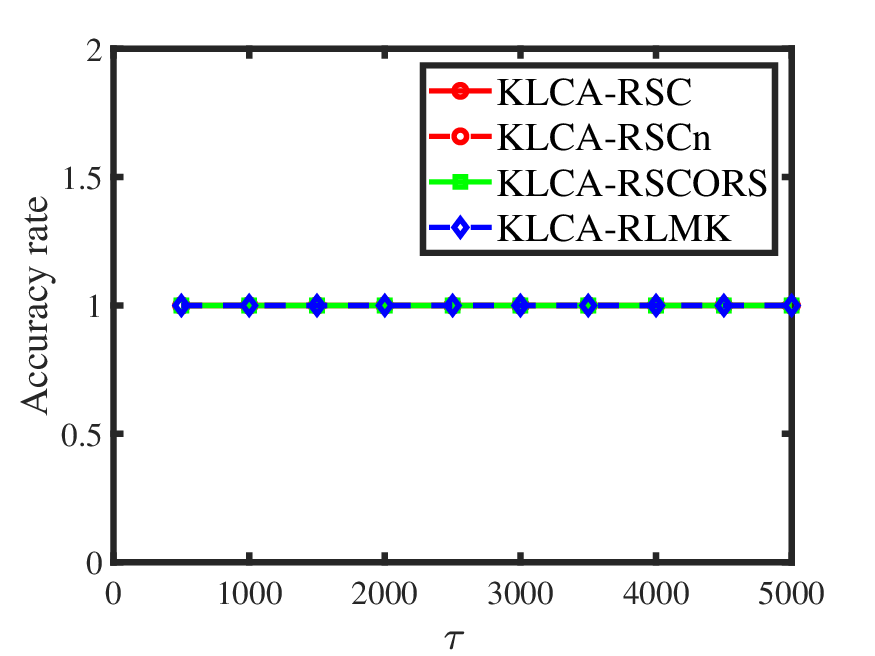}}
%}
\caption{Numerical results of Experiment 1.}
\label{Ex1} %% label for entire figure
\end{figure}
Our numerical studies focus on three issues: the choice of the regularizer $\tau$, the effectiveness of all aforementioned methods in latent class analysis, and the accuracy of our methods for estimating $K$. Unless specified, we set $K=3, J=\frac{N}{5}$, and $M=5$. Note that since we let $M$ be $5$ in our numerical studies, $R(i,j)$ takes value from the set $\{0,1,2,3,4,5\}$ for $i\in[N], j\in[J]$, which suggests a categorical data case. $Z$ is generated by letting each subject belong to one of the $K$ latent classes with equal probability. Let $\tilde{B}$ be a $J$-by-$K$ matrix such that $\tilde{B}(j,k)=\mathrm{rand}(1)$ for $i\in[J],k\in[K]$, where $\mathrm{rand}(1)$ is a random value generated from a Uniform distribution on $[0,1]$. Let $B=\frac{\tilde{B}}{\mathrm{max}_{j\in[J],k\in[K]}\tilde{B}(j,k)}$ (thus, $\mathrm{max}_{j\in[J],k\in[K]}B(j,k)=1$ is guaranteed). We set $\rho$ and $N$ independently for each experiment. Except for the experiment that studies the choice of $\tau$, we set $\tau$ as the default value $M\mathrm{max}(N, J)$. After setting $Z$ and $\Theta$, each numerical experiment has below steps:
\begin{description}
  \item[(a)] For $i\in[N],j\in[J]$,  generate $R(i,j)$ from a Binomial distribution with $M$ independent trials and probability $\frac{\mathscr{R}(i,j)}{M}$, where $\mathscr{R}=Z\Theta'=\rho ZB'$.
  \item[(b)] Apply our proposed methods to $R$. Record each metric and running time of all methods.
  \item[(c)] Repeat (a)-(b) 100 times and report the mean of each metric.
\end{description}
\begin{figure}
\centering
%\resizebox{\columnwidth}{!}{
\subfigure[]{\includegraphics[width=0.4\textwidth]{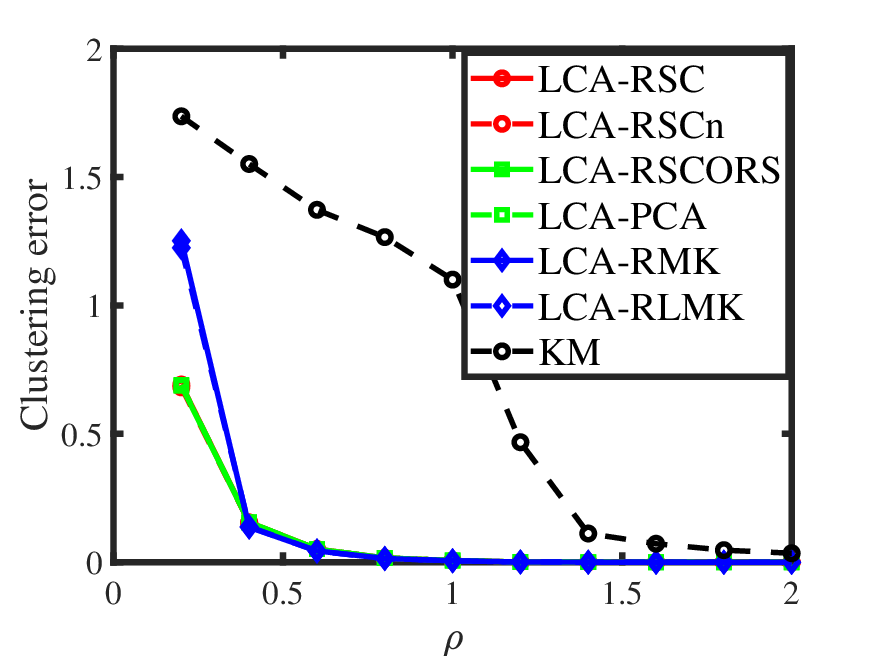}}
\subfigure[]{\includegraphics[width=0.4\textwidth]{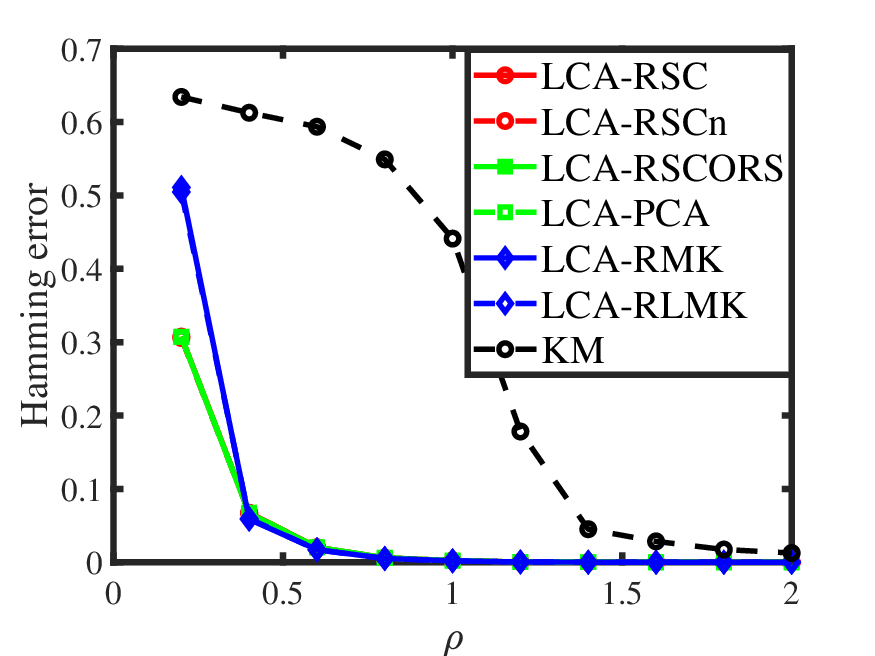}}
\subfigure[]{\includegraphics[width=0.4\textwidth]{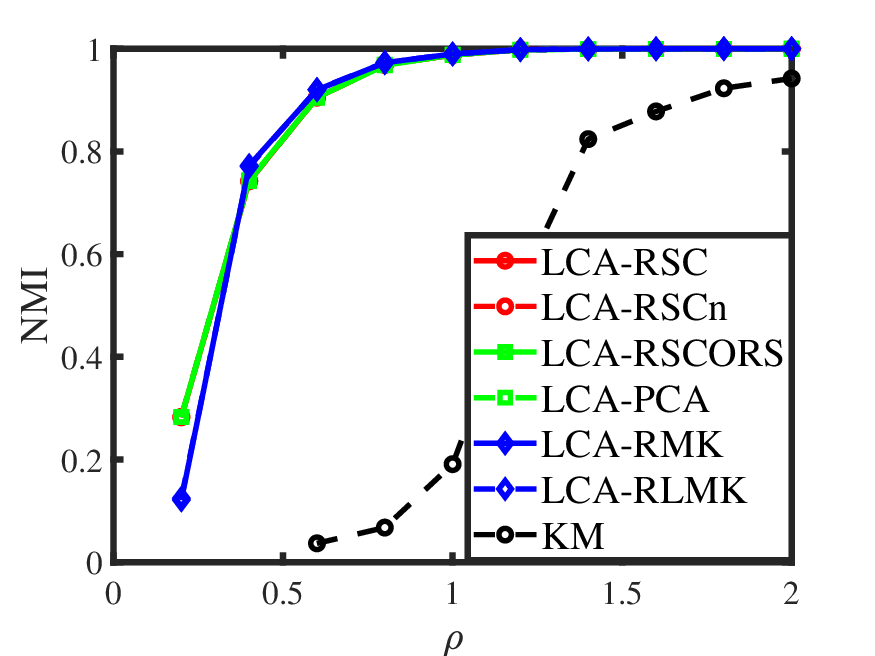}}
\subfigure[]{\includegraphics[width=0.4\textwidth]{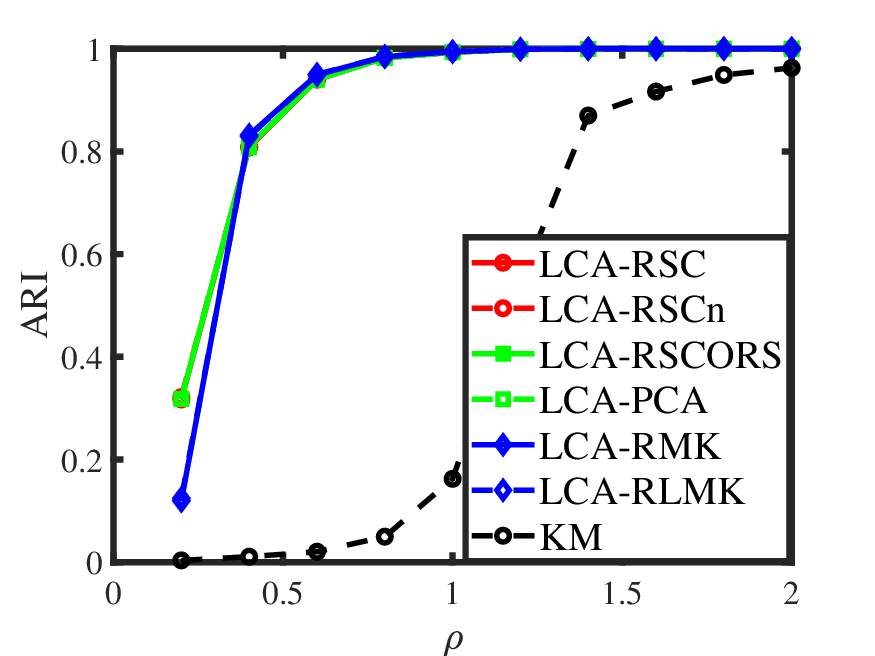}}
\subfigure[]{\includegraphics[width=0.4\textwidth]{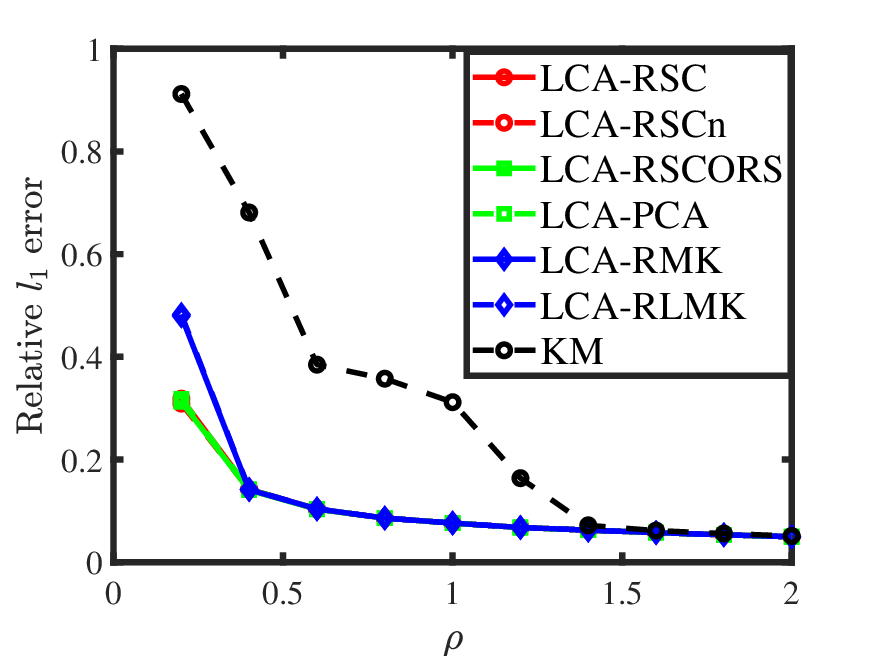}}
\subfigure[]{\includegraphics[width=0.4\textwidth]{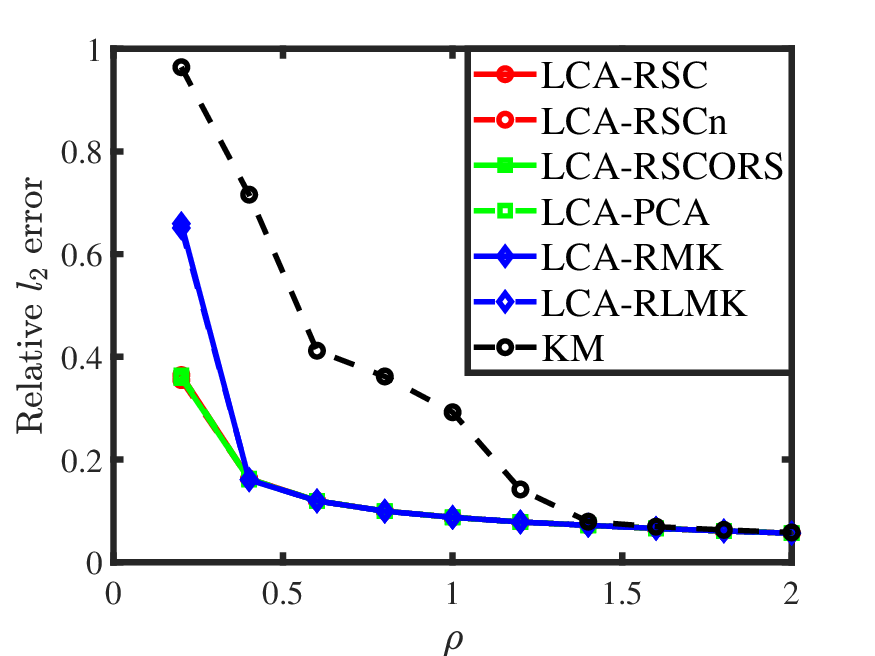}}
\subfigure[]{\includegraphics[width=0.4\textwidth]{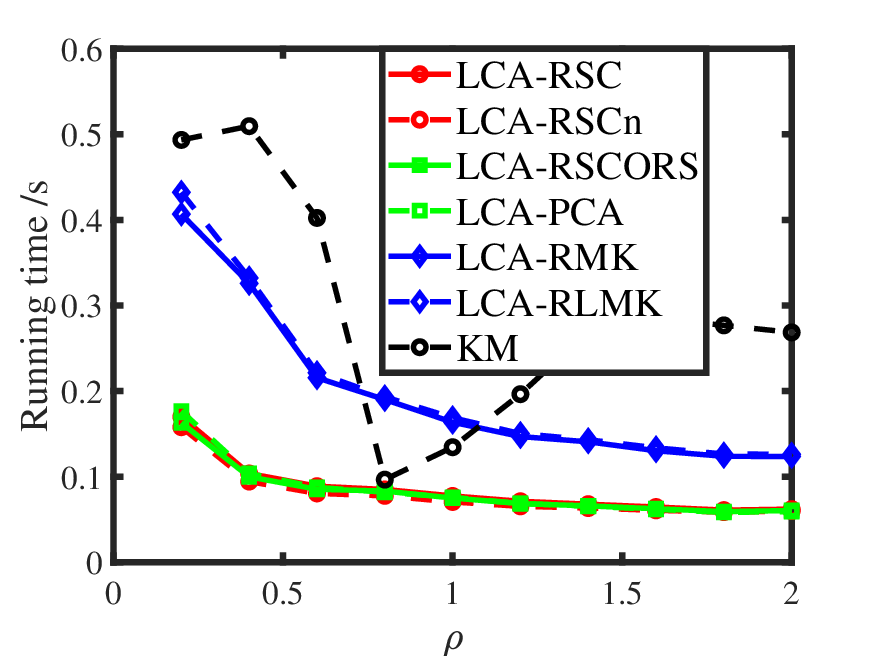}}
\subfigure[]{\includegraphics[width=0.4\textwidth]{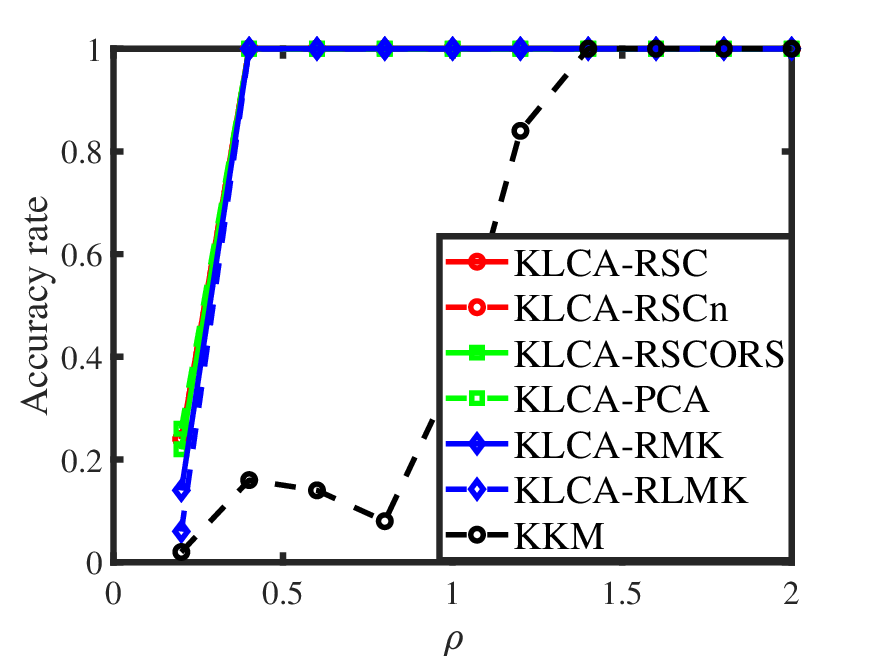}}
%}
\caption{Numerical results of Experiment 2.}
\label{Ex2} %% label for entire figure
\end{figure}

In this article, we conduct the following four Experiments.

\texttt{Experiment 1: Choice of the regularizer $\tau$.} Recall that the regularization parameter $\tau$ appears in our LCA-RSC, LCA-RSCn, LCA-RSCORS, and LCA-RLMK methods. Here, we study the choice of $\tau$ numerically. For this experiment, we set $\rho=0.8, N=500$, and $\tau=c_{0}M\mathrm{max}(N, J)$. We vary  $c_{0}$ in the range $\{0.2, 0.4, \ldots,2\}$. The results can be found in Figure \ref{Ex1}. From panels (a)-(f) of Figure \ref{Ex1}, we observe that the four methods have the smallest Clustering error, Hamming error, Relative $l_{1}$ error, and Relative $l_{2}$ error when $\tau$ is set around $M\mathrm{max}(N, J)$; they also have the largest NMI and ARI when $\tau$ is around $M\mathrm{max}(N, J)$. This suggests that setting $\tau$ as $M\mathrm{max}(N, J)$ is a good choice for these methods. Meanwhile, we also find that all metrics in panels (a)-(f) do not change significantly when $\tau$ increases and this verifies our analysis after Theorem \ref{mainRLCM} that the error rates are insensitive for large $\tau$.  From Panel (g) of Figure \ref{Ex1}, we see that LCA-RSC, LCA-RSCn, and LCA-RSCORS run faster than LCA-RLMK, which verifies our complexity analysis provided after Algorithm \ref{alg:RMK} and Algorithm \ref{alg:RLMK}. The last panel of Figure \ref{Ex1} says that KLCA-RSC, KLCA-RSCn, KLCA-RSCORS, and K-LCA-RLMK determines the number of classes $K$ correctly in this experiment, which also supports the effectiveness of our idea that using Newman-Girvan modularity to measure the quality of latent class when the adjacency matrix $A$ is set as $RR'$.

\texttt{Experiment 2: Changing the sparsity parameter $\rho$.} For this experiment, we investigate the performances of our methods by changing $\rho$ when $N=500$. Because $\rho$ should be no larger than the number of trials $M$ under Binomial distribution, here we let $\rho$ take value in the set $\{0.2, 0.4, \ldots, 2\}$. Figure \ref{Ex2} presents the results. From the first six panels, we observe that our methods (i.e., LCA-RSC, LCA-RSCn, LCA-RSCORS, LCA-PCA, LCA-RMK, LCA-RLMK) perform similarly, they perform better when the sparsity parameter increases, which confirms our theoretical analysis provided after Theorem \ref{mainRLCM}. Meanwhile, our methods significantly outperform the KM algorithm. From panel (g) of Figure \ref{Ex2}, we observe that LCA-RMK and LCA-RLMK run slower than LCA-RSC, LCA-RSCn, LCA-RSCORS, and LCA-PCA, which supports our analysis after Algorithm \ref{alg:RMK} and Algorithm \ref{alg:RLMK}.  We also see that all methods except for KM run faster as the sparsity parameter becomes bigger, where KM runs slowest among all methods. Finally, the last panel tells us that our methods for estimating $K$ perform better as $\rho$ increases, they significantly outperform KKM, and they exactly infer $K$ when $\rho$ is no smaller than 0.5 for this experiment. Again, the high accuracy of our methods in estimating $K$ suggests the success of using the Newman-Girvan modularity to evaluate the quality of latent class analysis.
\begin{rem}\label{remAssum0}
In the context of Experiment 2, the minimum $\rho$ that satisfies Assumption \ref{Assum1} is $\rho = \frac{M^{2}\mathrm{log}(N+J)}{\mathrm{max}(N,J)} = 0.3198$. According to the numerical results of Experiment 2, our methods demonstrate satisfactory performance when $\rho \geq 0.4 > 0.3198$ and poor performance when $\rho = 0.2 < 0.3198$, indicating that Assumption \ref{Assum1} is essential for satisfactory method performance.
\end{rem}

\texttt{Experiment 3: Changing the number of subjects $N$.} For this experiment, we study the performances of all approaches by increasing $N$ when $\rho=0.15$. We let $N$ range in $\{1000, 2000, \ldots, 8000\}$. Figure \ref{Ex3} shows the results. The first six panels say that all methods have similar performances in estimating $Z$ and $\Theta$, they outperform KM significantly, and they perform better as $N$ increases, which supports our analysis after Theorem \ref{mainRLCM}. Panel (g) of Figure \ref{Ex3} says that LCA-RSC, LCA-RSCn, LCA-RSCORS, and LCA-PCA run faster than LCA-RMK and LCA-RLMK, which also verifies our analysis after Algorithm \ref{alg:RMK}. Meanwhile, from panel (g), we also observe that LCA-RSC, LCA-RSCn, LCA-RSCORS, and LCA-PCA process a response matrix of up to 8000 subjects and 1600 items within five seconds, which suggests the effectiveness of our methods in latent class analysis. Furthermore, by panel (g), we find that our methods run much faster than KM. The last panel shows that our methods are more accurate in deciding the number of latent classes than KKM, which implies the effectiveness of the Newman-Girvan modularity in turn.
\begin{rem}\label{remarkAssum1}
In the context of Experiment 3, the minimum $\rho$ that satisfies Assumption \ref{Assum1} is $\rho = \frac{M^{2}\mathrm{log}(N+J)}{\mathrm{max}(N,J)}$ (i.e., the values are $0.1773, 0.0973, 0.0682, 0.0530, 0.0435, 0.0370, 0.0323$, and $0.0287$ for $N=1000, 2000, 3000, 4000, 5000, 6000, 7000$, and $8000$, respectively). Recall that in Experiment 3, $\rho=0.15$ is a value that is smaller than the minimum requirement for $N=1000$ and larger than the minimum requirement for $N=2000, 3000, \ldots, 8000$. The numerical results of Experiment 3 demonstrate that our methods achieve satisfactory performance when $N\geq2000$ and poor performance when $N=1000$. This suggests that Assumption \ref{Assum1} is crucial for the success of our methods.
\end{rem}
\begin{figure}
\centering
%\resizebox{\columnwidth}{!}{
\subfigure[]{\includegraphics[width=0.4\textwidth]{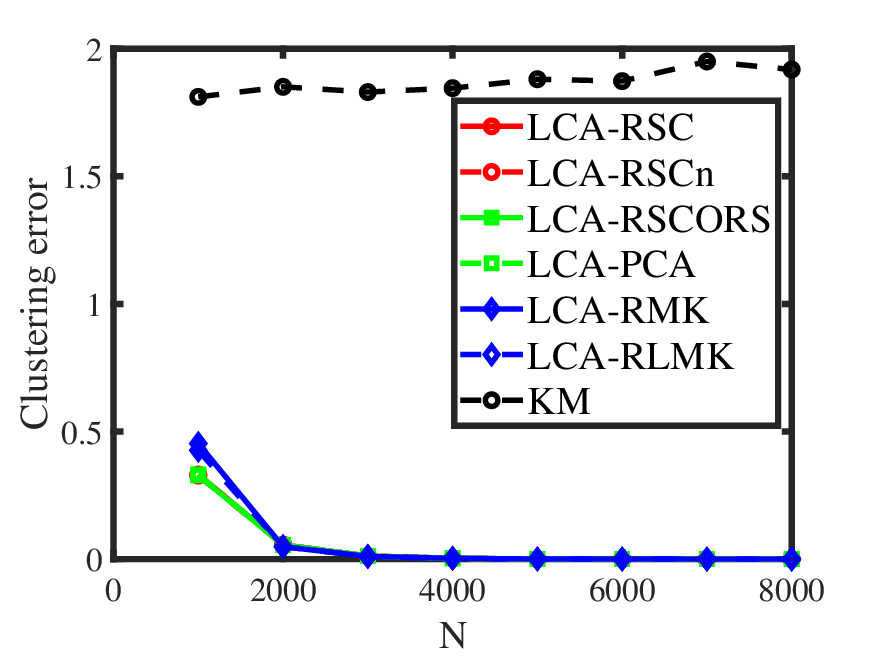}}
\subfigure[]{\includegraphics[width=0.4\textwidth]{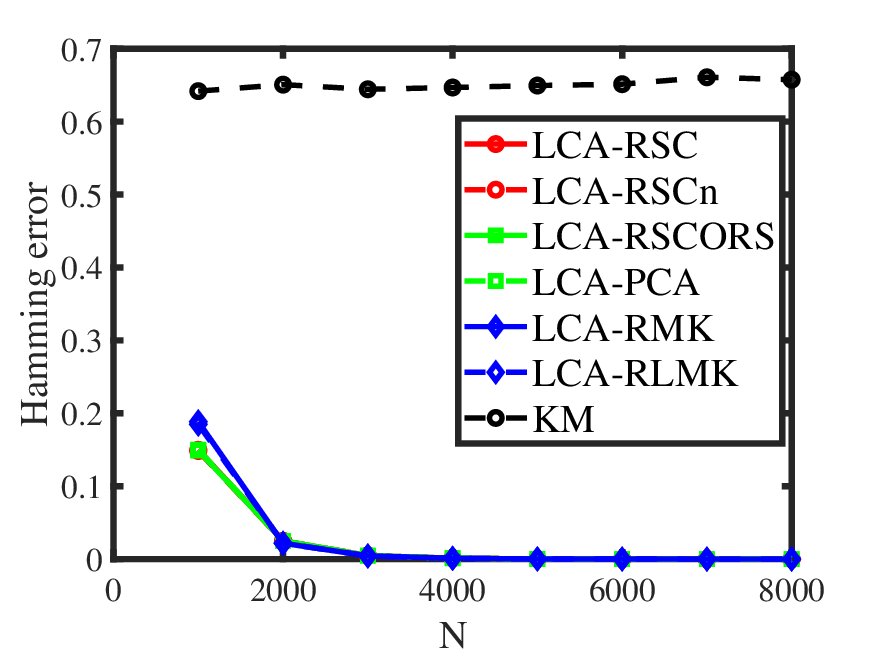}}
\subfigure[]{\includegraphics[width=0.4\textwidth]{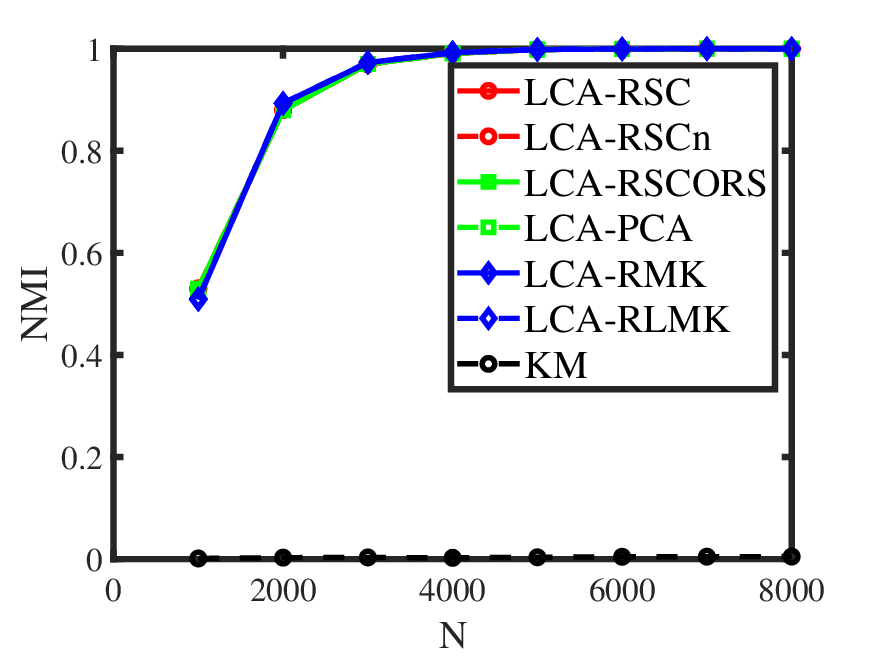}}
\subfigure[]{\includegraphics[width=0.4\textwidth]{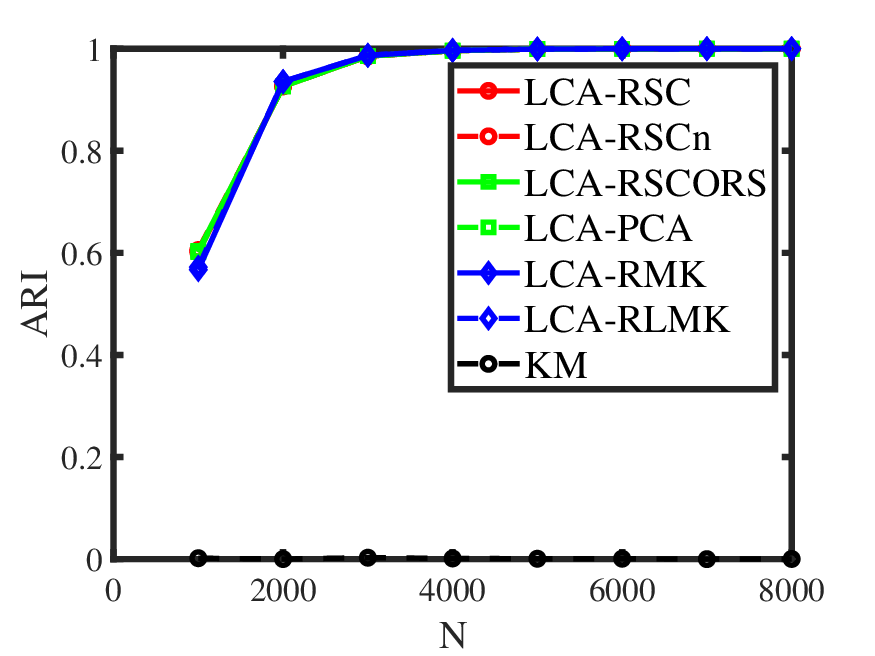}}
\subfigure[]{\includegraphics[width=0.4\textwidth]{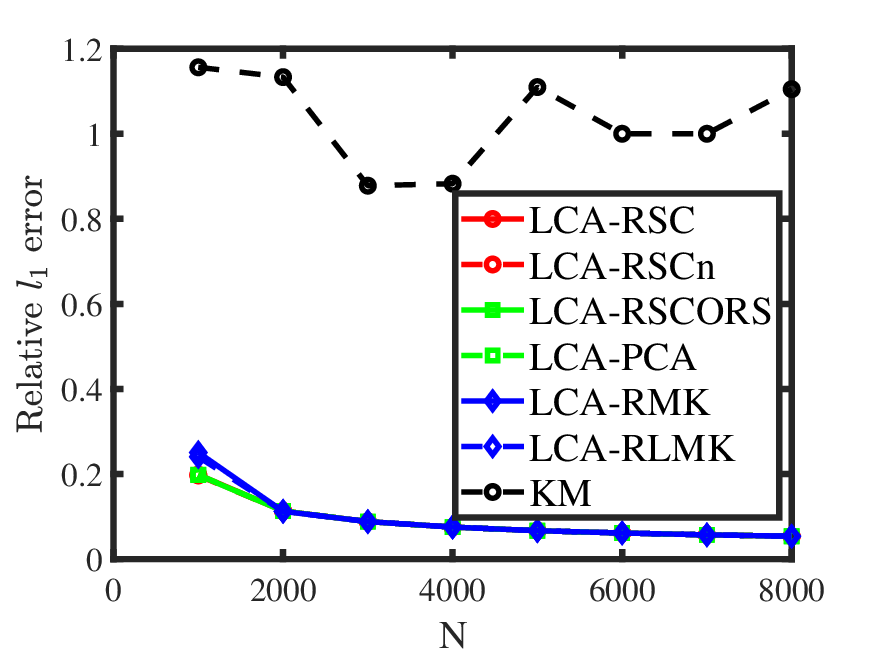}}
\subfigure[]{\includegraphics[width=0.4\textwidth]{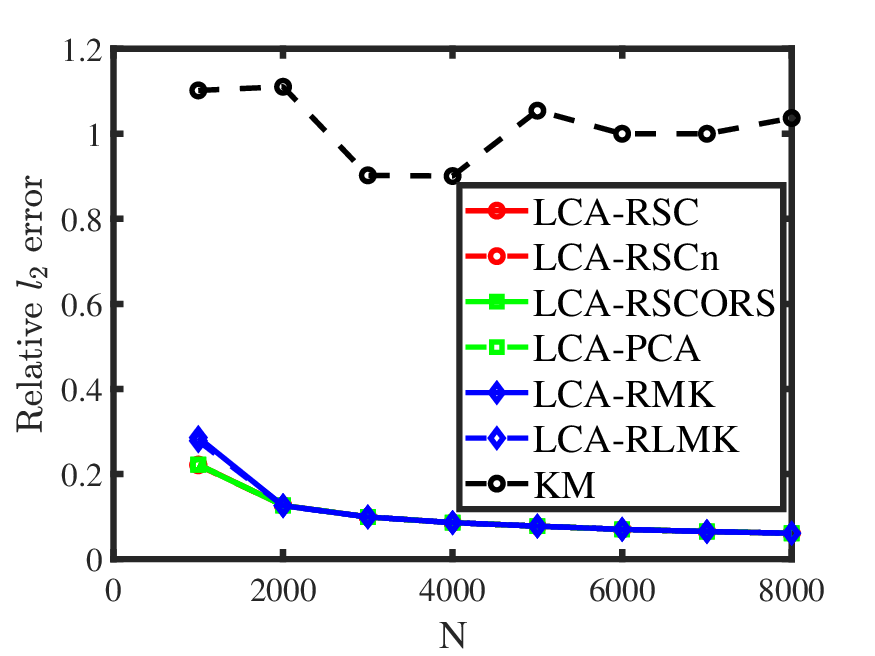}}
\subfigure[]{\includegraphics[width=0.4\textwidth]{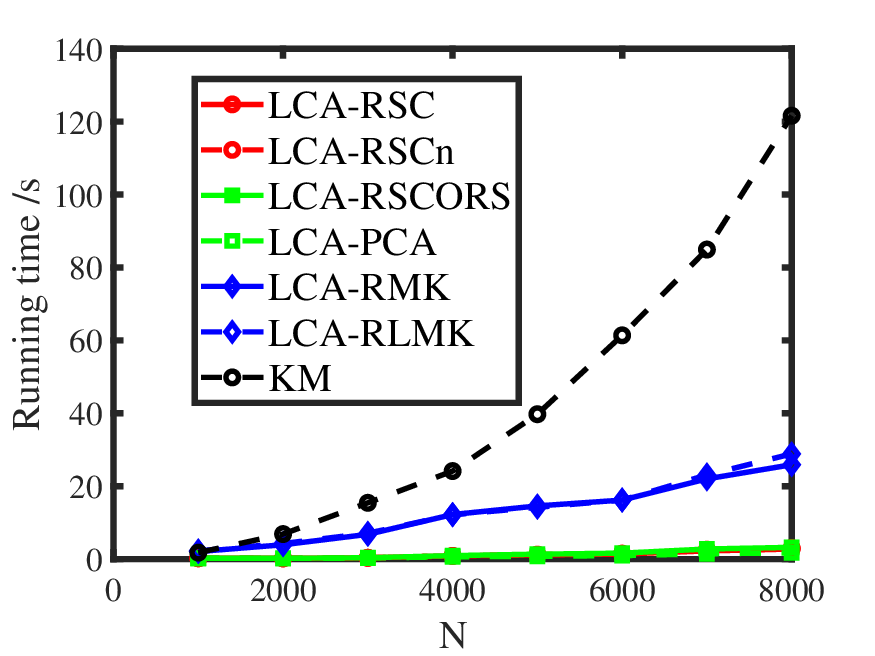}}
\subfigure[]{\includegraphics[width=0.4\textwidth]{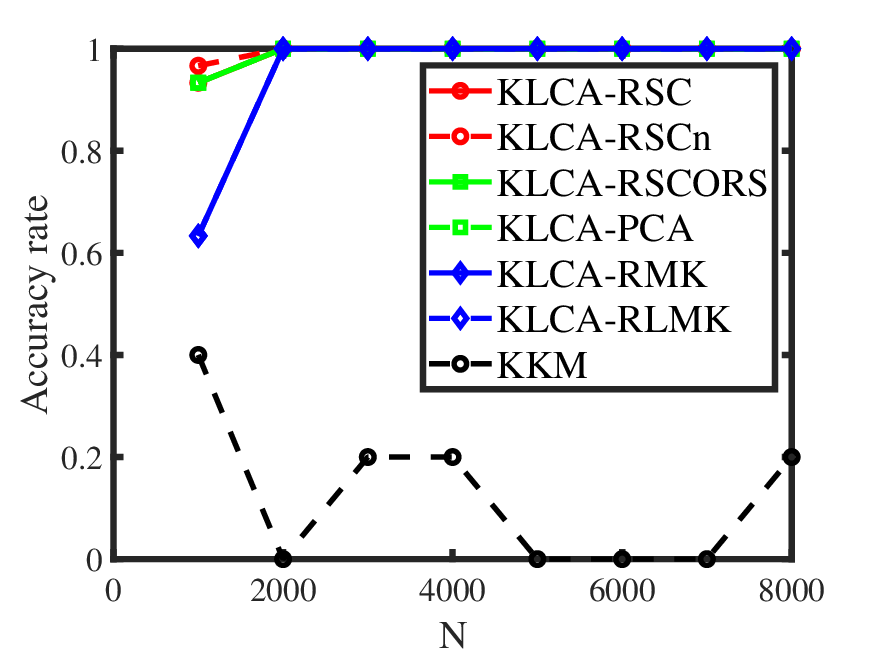}}
%}
\caption{Numerical results of Experiment 3.}
\label{Ex3} %% label for entire figure
\end{figure}

\texttt{Experiment 4: Changing the number of latent classes $K$.} Here, we investigate the performance of all methods in estimating $K$ as it increases. To ensure accurate results, we set
$N=1200, J=240, M=5$, and $\rho=4$. We vary $K$ from 1 to 8. Figure \ref{Ex4} displays the outcomes. Panels (a)-(d) show that all methods accurately recover $Z$. Panels (e)-(f) reveal that all methods produce similar estimations of $\Theta$ and their error rates in estimating $\Theta$ slightly increase as $K$ rises. This excellent performance in estimating $Z$ and $\Theta$ is due to our choice of large $N$ and $\rho$. Panel (g) shows that our methods run faster than KM. Panel (h) reveals that all methods accurately estimate $K$ when $K\geq2$ but fail to estimate $K=1$. This suggests that more advanced methods are needed to estimate $K=1$ accurately. Furthermore, in the context of latent class analysis for real-world categorical data with polytomous responses, researchers often aim to divide subjects into $K\geq2$ subgroups for further analysis. If there were only one group, there would be no need for further exploration. Therefore, while our methods may not estimate $K$ accurately when the true $K$ is $1$, this is not a significant concern as they can accurately estimate $K$ when the true $K$ is at least $2$, which is the scenario that researchers typically find more relevant for real-world categorical data analysis. 
\begin{figure}
\centering
%\resizebox{\columnwidth}{!}{
\subfigure[]{\includegraphics[width=0.4\textwidth]{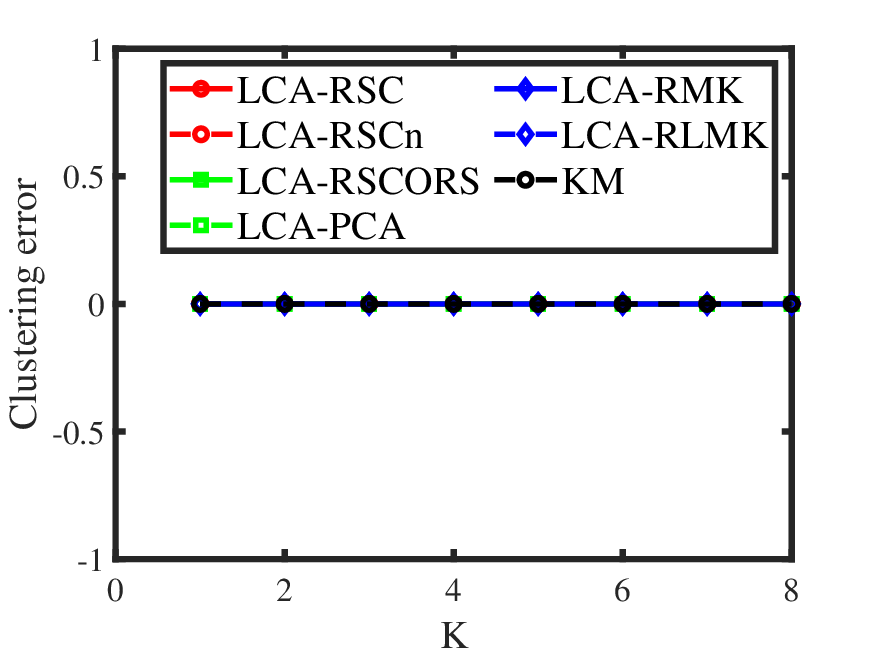}}
\subfigure[]{\includegraphics[width=0.4\textwidth]{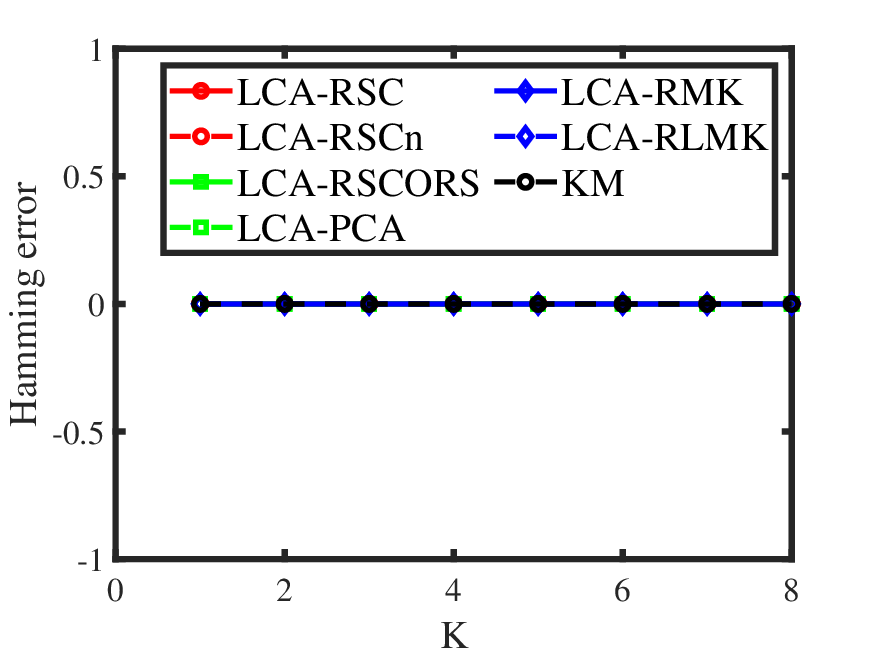}}
\subfigure[]{\includegraphics[width=0.4\textwidth]{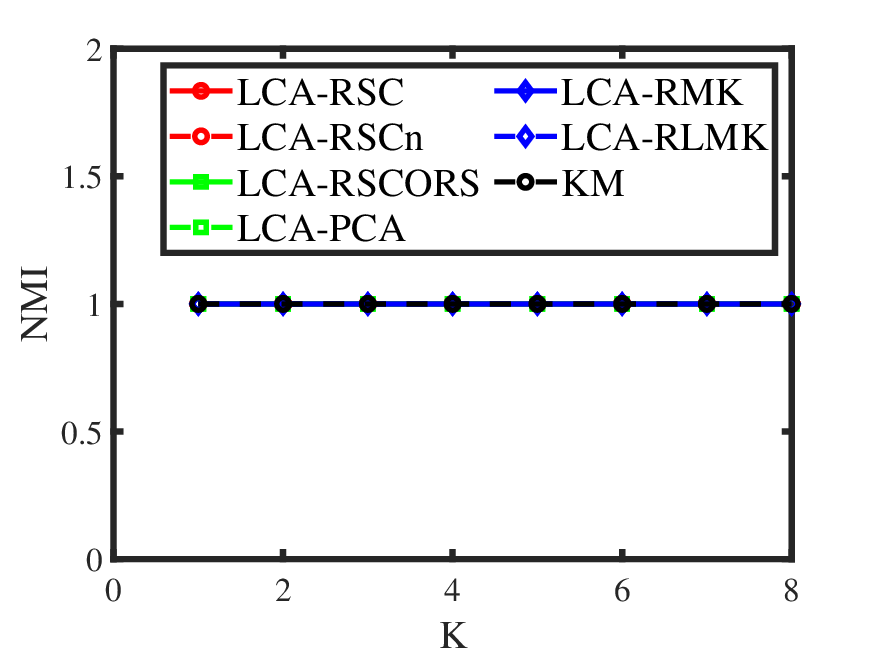}}
\subfigure[]{\includegraphics[width=0.4\textwidth]{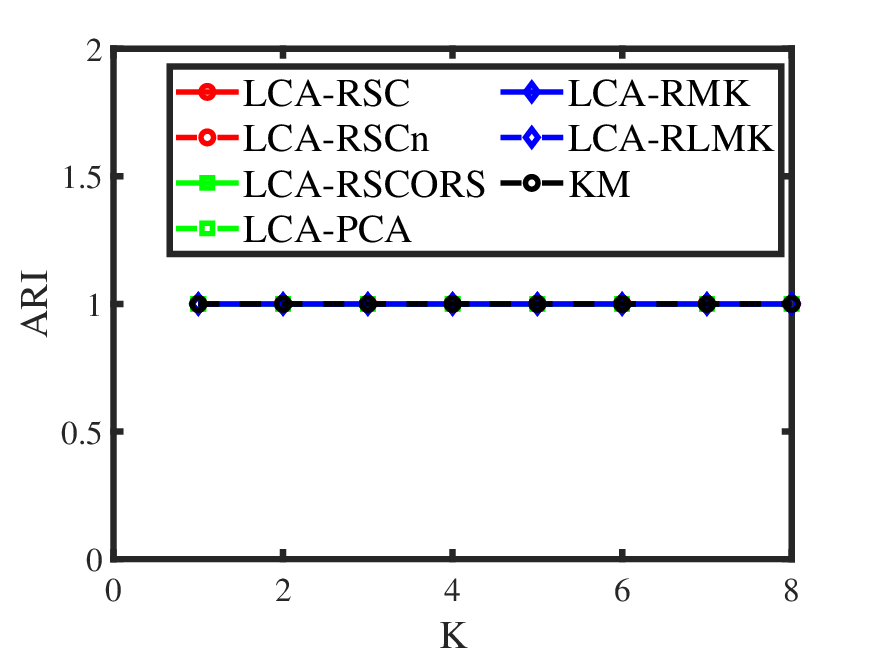}}
\subfigure[]{\includegraphics[width=0.4\textwidth]{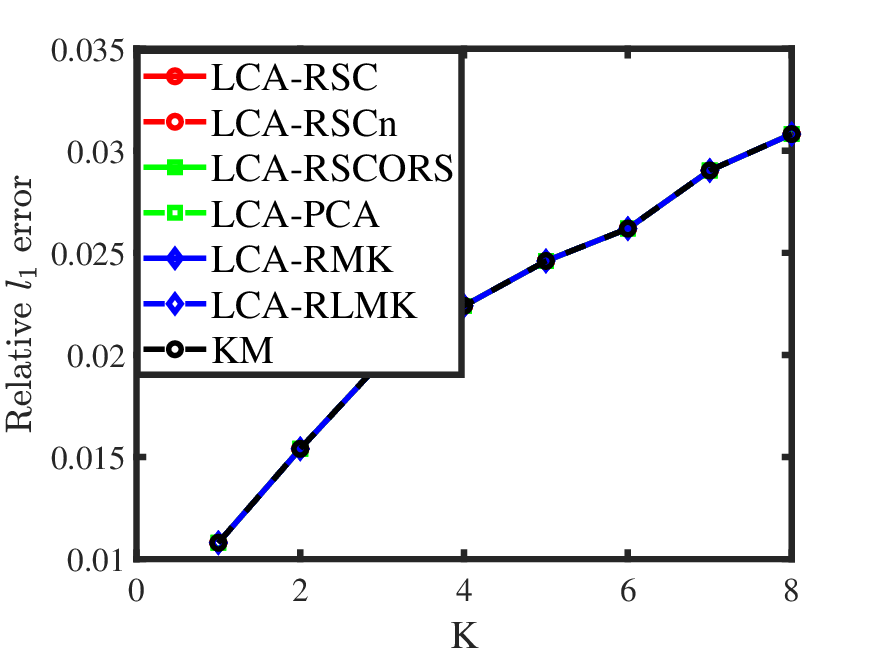}}
\subfigure[]{\includegraphics[width=0.4\textwidth]{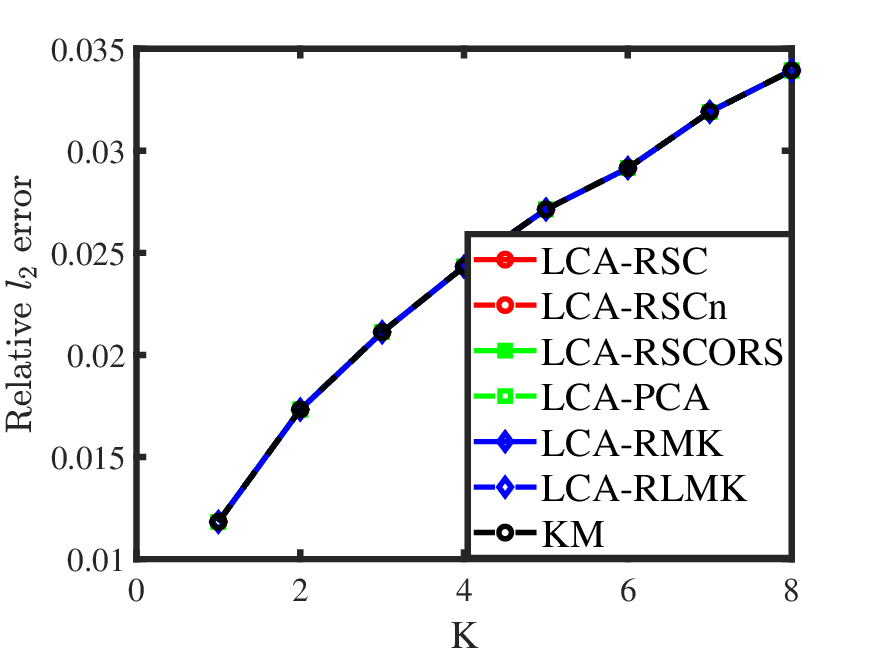}}
\subfigure[]{\includegraphics[width=0.4\textwidth]{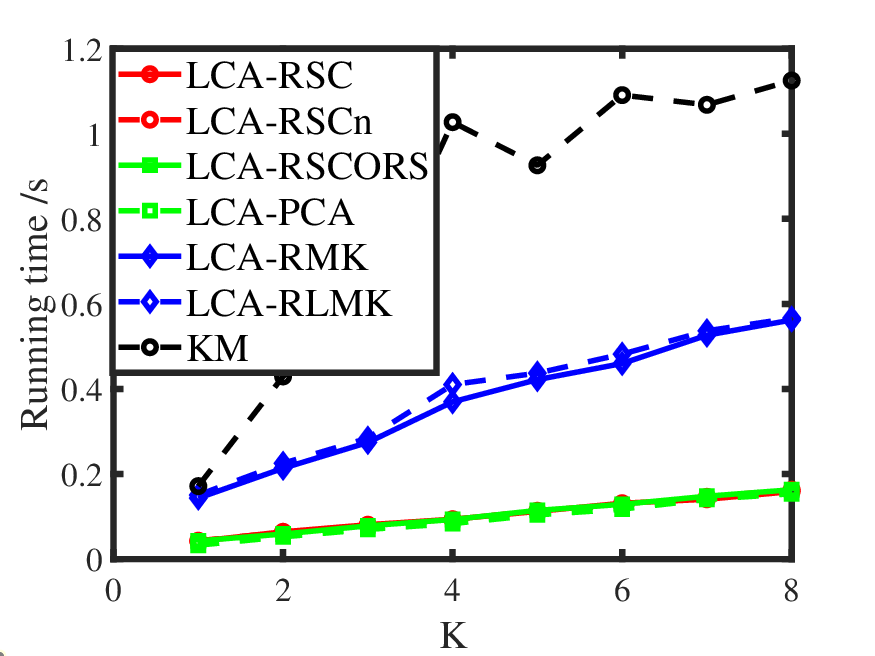}}
\subfigure[]{\includegraphics[width=0.4\textwidth]{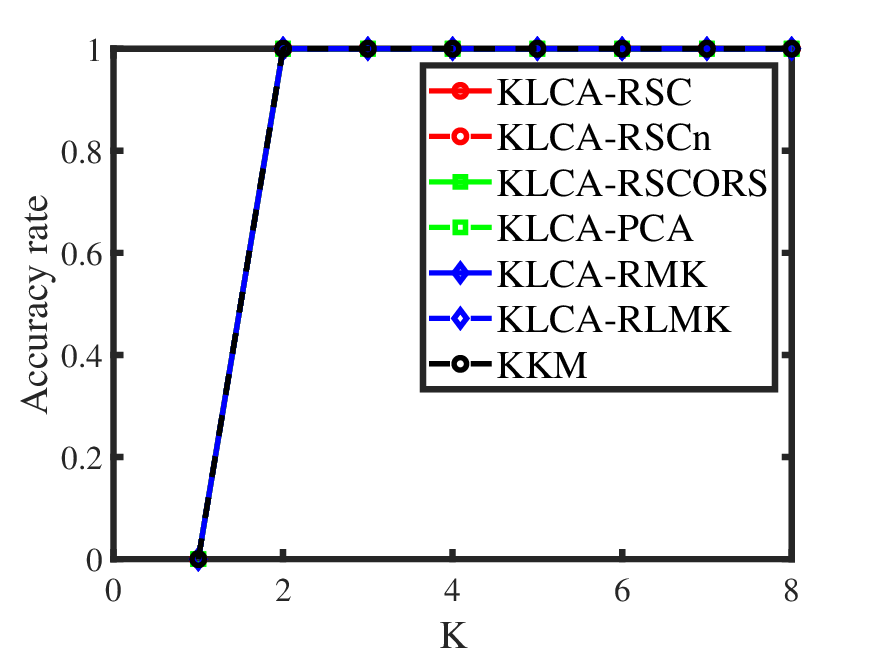}}
%}
\caption{Numerical results of Experiment 4.}
\label{Ex4} %% label for entire figure
\end{figure}

\begin{figure}[htbp]
\centering
%\resizebox{\columnwidth}{!}{
\subfigure[]{\includegraphics[width=0.4\textwidth]{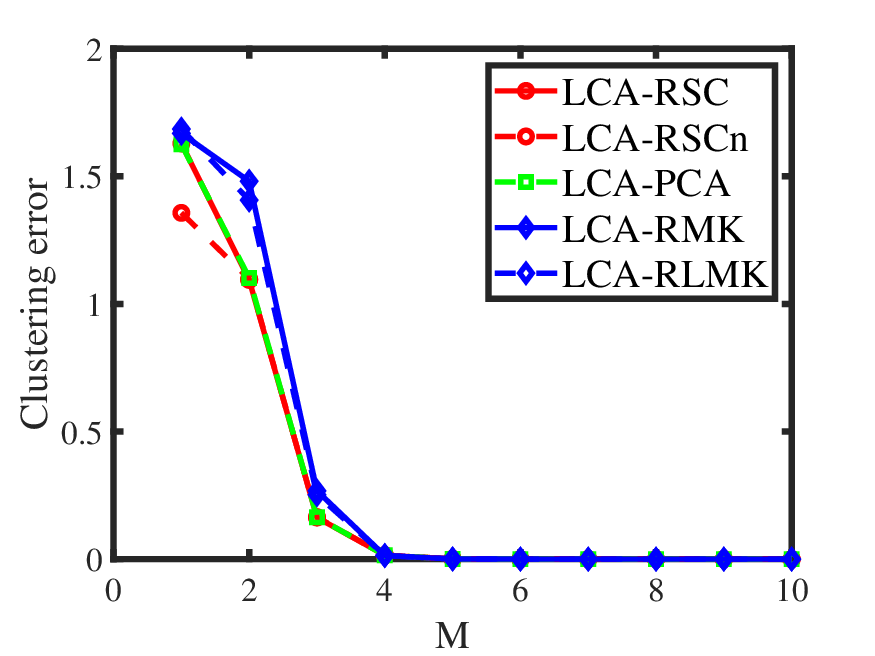}}
\subfigure[]{\includegraphics[width=0.4\textwidth]{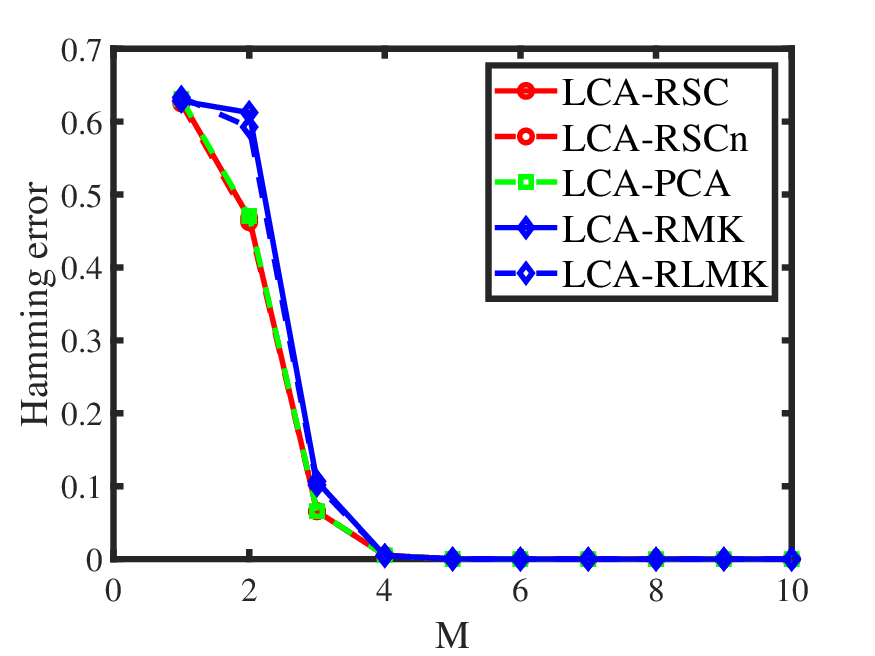}}
\subfigure[]{\includegraphics[width=0.4\textwidth]{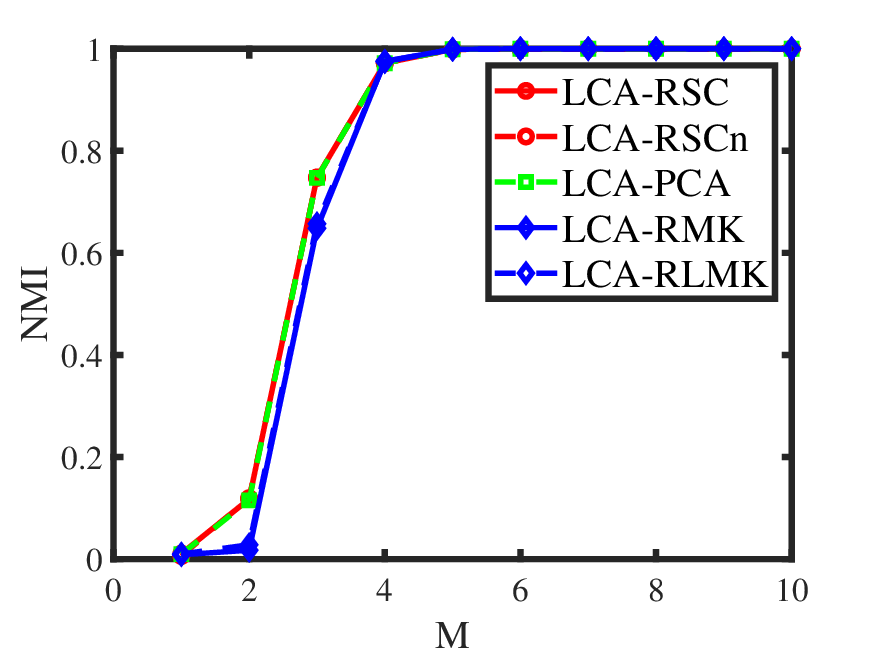}}
\subfigure[]{\includegraphics[width=0.4\textwidth]{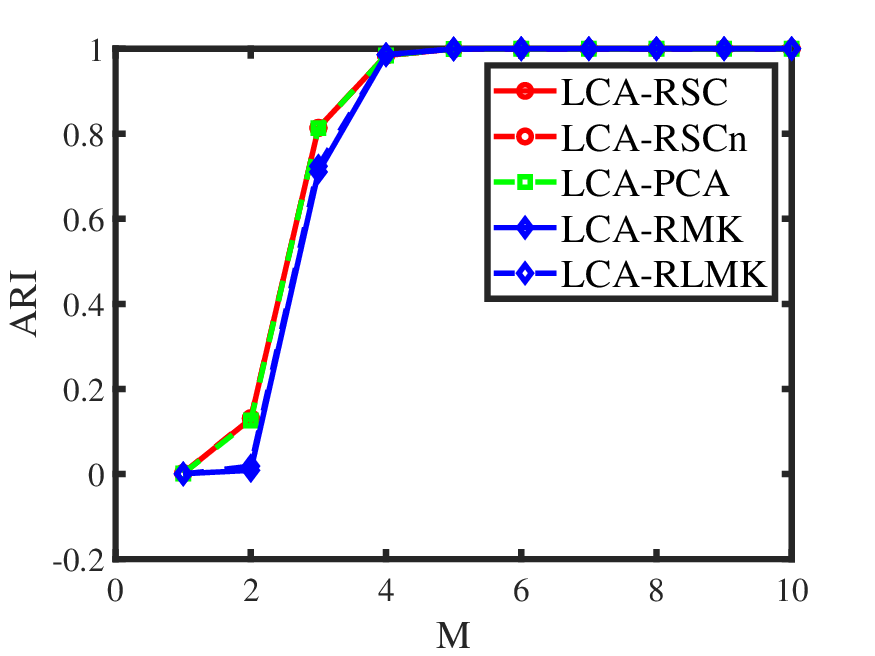}}
\subfigure[]{\includegraphics[width=0.4\textwidth]{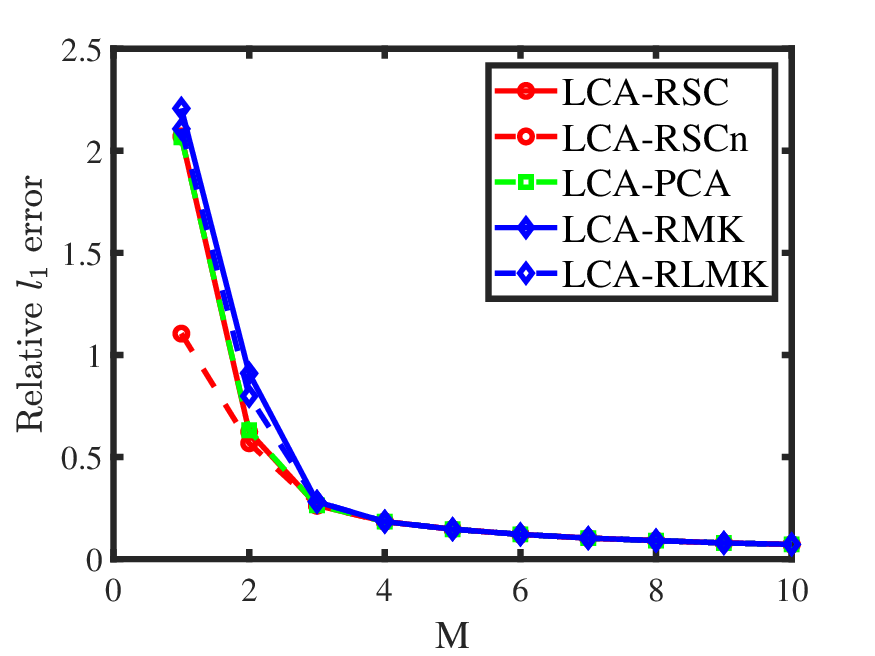}}
\subfigure[]{\includegraphics[width=0.4\textwidth]{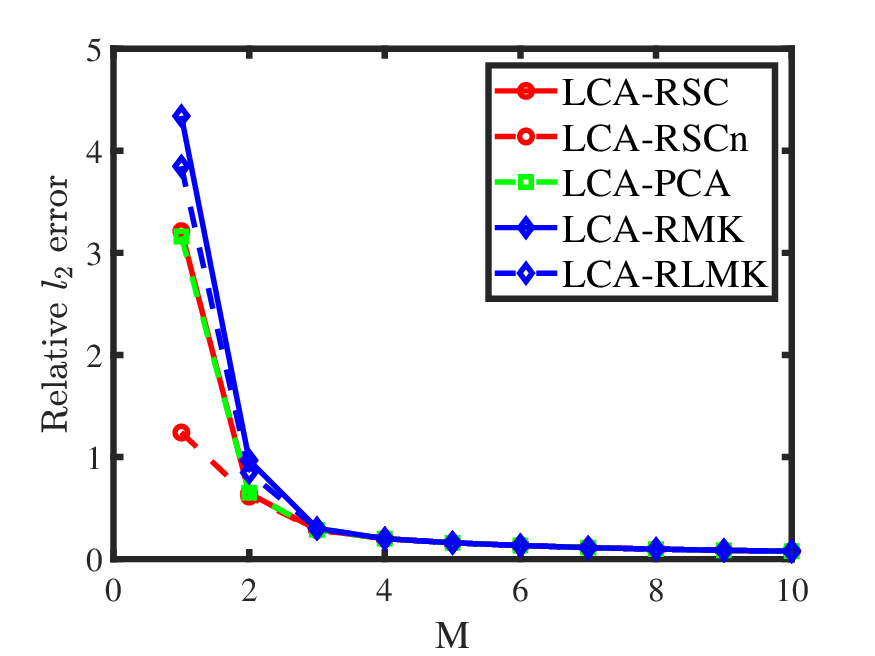}}
\subfigure[]{\includegraphics[width=0.4\textwidth]{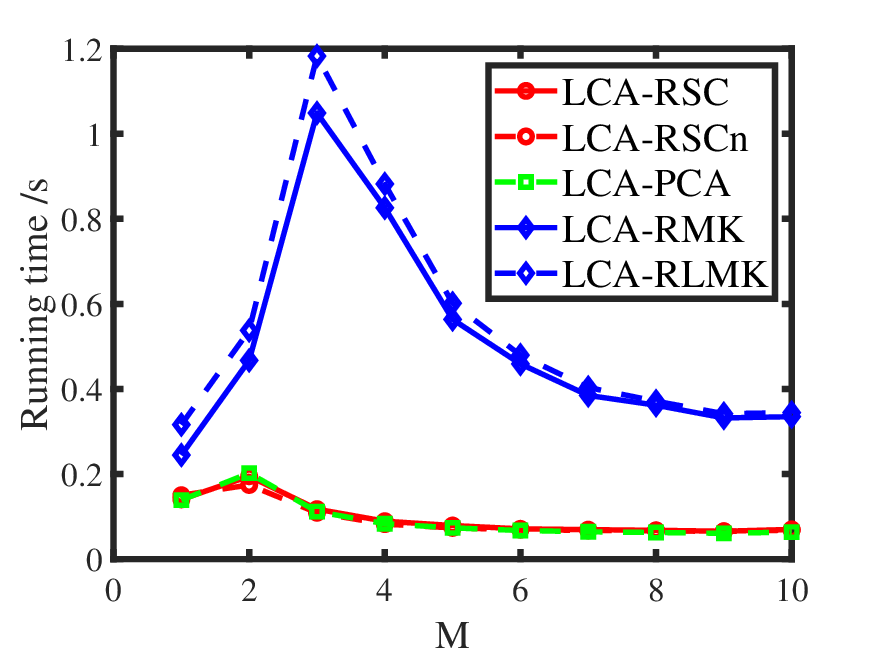}}
\subfigure[]{\includegraphics[width=0.4\textwidth]{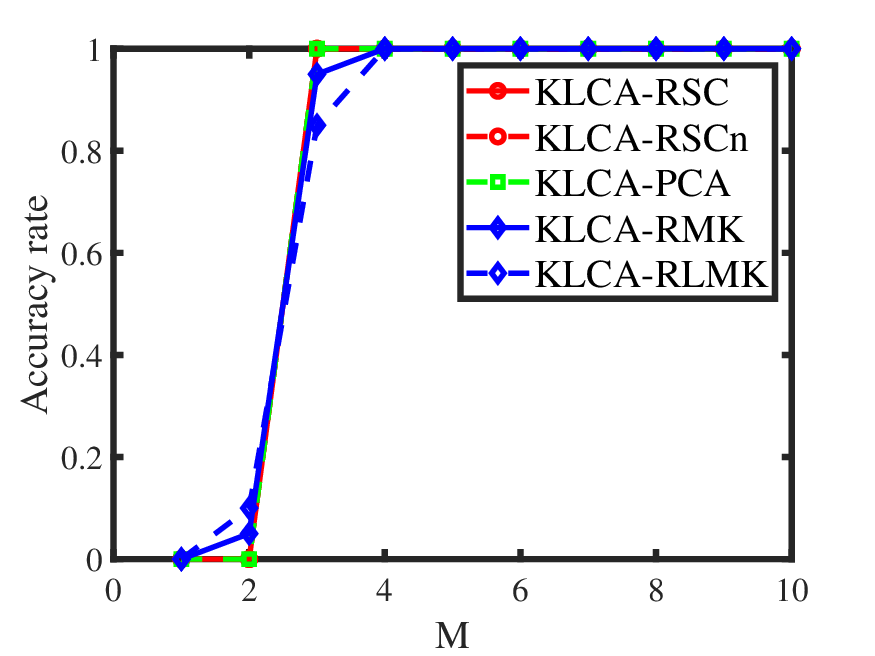}}
%}
\caption{Numerical results of Experiment 5.}
\label{Ex5} %% label for entire figure
\end{figure}

\texttt{Experiment 5: Changing $M$.} In this experiment, we evaluate the performance of our methods when the sparsity parameter $\rho$ is set to the minimum requirement stated in Assumption \ref{Assum1}. This value is calculated as $\rho = \frac{M^{2}\mathrm{log}(N+J)}{\mathrm{max}(N,J)}$ with $M$ increasing from 1 to 10, $N=500$, $J=500$, and $K=3$. Figure \ref{Ex5} displays the results, excluding LCA-RSCORS and KM due to their failure to produce output when $R$ is excessively sparse in this experiment. Our methods demonstrate satisfactory performance when $\rho$ is set according to the minimum requirement, with improved performance observed as $M$ increases.

\section{Real-data applications}\label{sec6}
For empirical studies, we apply algorithms developed in this article to real-world ordered categorical data. To address potential issues of missing data in real-world ordered categorical datasets, we implement a straightforward yet effective preprocessing procedure before applying the proposed algorithms. Specifically, any subject with one or more missing responses (i.e., ``NaN'' values)  is entirely removed from the observed response matrix \(R\). This ensures the completeness of the data matrix required for constructing the regularized Laplacian matrix and subsequent spectral clustering. Some real datasets originally code non‑response entries as 0. Examples are ``no rating'' in the MovieLens 100k data and ``no response'' in the International Personality Item Pool personality test data. Our model only needs each entry to be an integer in $\{0,1,\dots,M\}$, so we keep these codes as they are. The value 0 is treated as a valid response category. We do not claim that a coded 0 means the same as a true lowest rating. We should emphasize that our LCM model is built upon the Binomial distribution, which naturally accommodates response patterns at the boundaries (e.g., all ``0''s or all ``M''s). Such patterns are not treated as outliers to be removed. Instead, they are valid responses that may be generated from LCM. For instance, a subject consistently selecting the lowest response category (i.e., 0) across all items likely belongs to a latent class characterized by extremely low endorsement or ability. Our algorithms are designed to detect these patterns and will typically assign such subjects to a distinct cluster. The item parameter matrix $\hat{\Theta}$ will then reflect this, with very low (for all ``0''s) or very high (for all ``M''s) estimated values for that class. Therefore, no special preprocessing for these extreme but valid patterns is necessary, i.e., our methods are inherently robust to them.
\subsection{MovieLens 100k data}
\begin{table}[h!]
\footnotesize
	\centering
	\caption{Estimated number of latent classes and the respective modularity of all methods for real data considered in this paper. The 1st value in the bracket represents the estimated number of latent classes $\hat{K}_{\mathcal{M}}$ and the 2nd value in the bracket represents the respective modularity $Q_{\hat{K}_{\mathcal{M}}}$ for method $\mathcal{M}$.}
	\label{RealDataModularityK}
%	\resizebox{\columnwidth}{!}{
	\begin{tabular}{cccccccccccc}
\hline\hline
Dataset&LCA-RSC&LCA-RSCn&LCA-RSCORS&LCA-PCA&LCA-RMK&LCA-RLMK&KM\\
\hline
MovieLens 100k&(3, 0.0941)&(3, 0.0990)&(3, 0.0729)&(3, 0.0933)&(4, 0.0667)&(4, 0.0673)&(6,0.0292)\\
IPIP&(2, 0.0077)&(2, 0.0077)&(2, 0.0073)&(2, 0.0077)&(2, 0.0071)&(2, 0.0071)&(2,0.0042)\\
\hline\hline
\end{tabular}%}
\end{table}

The MovieLens 100k \citep{kunegis2013konect} is a user-movie ratings data set that is available at the link \url{http://konect.cc/networks/movielens-100k_rating/}. This data consists of 943 users and 1682 movies. For this data, the rating scores range in $\{0,1,2,3,4,5\}$, where 0 means no rating, and the higher the user's rating for a movie, the more he/she likes it. Therefore we have $R\in\{0,1,2,3,4,5\}^{943\times1682}$. To choose how many latent classes should be used for the MovieLens 100k data,
we apply our methods to its observed response matrix $R$ using Equation (\ref{Modularity}). The results reported in Table \ref{RealDataModularityK} suggest that $K=3$ is most appropriate for the MovieLens 100k data because LCA-RSCn returns the largest modularity value 0.0990 when $K=3$. For this reason, the LCA-RSCn method should be applied to the MovieLens 100k data with 3 latent classes. Applying LCA-RSCn to this data when $K=3$, we get the $943\times 3$ matrix $\hat{Z}$ and the $1682\times 3$ matrix $\hat{\Theta}$. For convenience, we denote the estimated latent classes as Class 1, Class 2, and Class 3. Based on $\hat{Z}$, we find that Class 1, 2, and 3 have 237, 253, and 453 users, respectively. For visualization, we plot the estimated latent classes for the top 30 users and the estimated item parameter matrix for the top 30 movies in Figure \ref{100KZTheta}. We can find the estimated latent class for each class clearly from the 1st matrix in Figure \ref{100KZTheta}. The strength of preference on each movie for users in the same latent class can be found from the 2nd matrix in Figure \ref{100KZTheta}, where Remark \ref{WHYS} explains why $\hat{\Theta}(j,k)$ reflects the strength of preference for users in the $k$-th latent class to the $j$-th movie for $j\in[1682], k\in[3]$. For example, users in Class 1 tend to give higher rating scores than that of users in Classes 2 and 3 for Movies 1-5, 7, 11, 15, 17, 21, 22, 24, 25, 27, and 29; users in Class 2 give higher rating scores than that of users in Class 3 for Movies 1, 2, 4-23, 25-26, and 27-30. In fact, by analyzing the $1682\times 3$ matrix $\hat{\Theta}$, we find that $\sum_{j=1}^{1682}\hat{\Theta}(j,1)=604.9283, \sum_{j=1}^{1682}\hat{\Theta}(j,2)=502.6364$, and $\sum_{j=1}^{1682}\hat{\Theta}(j,3)=182.0110$, which suggests that users in Class 1 tend to give higher rating scores than users in Class 2 and Class 3 while users in Class 2 tend to give higher values than users in Class 3. Based on this finding, we interpret Class 1 as users who hold an optimistic view of movies, Class 2 as users who hold a neutral view of movies, and Class 3 as users who hold a passive view of movies.

The application of our methods to the MovieLens 100k dataset demonstrates their practical utility in uncovering user preference profiles based on movie ratings. The three classes differ mainly in their overall rating level, but users in the same class still show similar rating patterns. This similarity alone supports user‑based collaborative filtering: a movie liked by one user in a class can be recommended to other users in the same class and vice versa. If more detailed preference data are available, such recommendations can be made even more precise. Such insights can directly inform recommendation systems by identifying distinct user groups with shared preferences, enabling more targeted and effective recommendations. Furthermore, the interpretation of latent classes as optimistic, neutral, and passive viewers provides a meaningful segmentation that can guide marketing strategies.
\begin{rem}\label{WHYS}
Let $\hat{N}_{k}=\sum_{i=1}^{N}\hat{Z}(i,k)$ for $k\in[K]$, i.e., $\hat{N}_{k}$ denotes the size of the $k$-th estimated latent class. By the 4th and 5th steps of Algorithm \ref{alg:RSC}, we have $\hat{\Theta}(j,k)$ roughly equals to $\frac{\sum_{i: \hat{Z}(i,k)==1\mathrm{~for~}i\in[N]}R(i,j)}{\hat{N}_{k}}$ for $j\in[J], k\in[K]$. Therefore, similar to $\Theta(j,k)$, $\hat{\Theta}(j,k)$ also denotes the averaged response for subjects in the $k$-th estimated latent class to the $j$-th item.
\end{rem}
\begin{figure}
\centering
\includegraphics[width=1\textwidth]{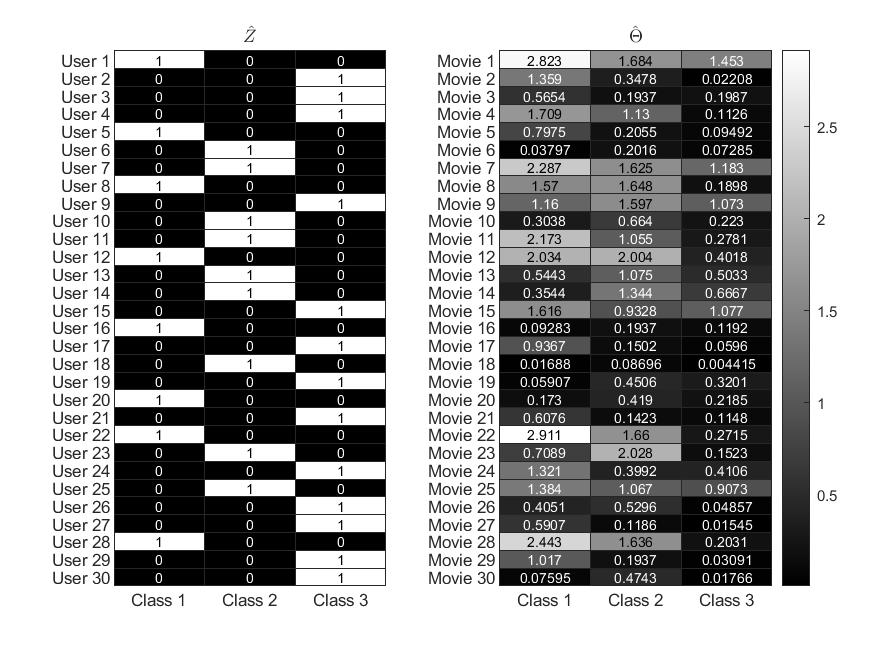}
\caption{Left matrix: Heatmap of the estimated classification matrix $\hat{Z}$ of a subset of 30 users for the MovieLens 100k data (white $=1$, black $=0$). Right matrix: Heatmap of the  estimated item parameter matrix $\hat{\Theta}$ of a subset of 30 movies for the MovieLens 100k data (continuous color scale, lighter means higher expected rating).}
\label{100KZTheta} %% label for entire figure
\end{figure}
\subsection{International Personality Item Pool (IPIP) personality test data}
Our methods are also applied to deal with a psychological test dataset, the International Personality Item Pool (IPIP) personality test data. It can be downloaded from \url{https://openpsychometrics.org/_rawdata/} and it contains responses of 1005 individuals to 40 personality test questions. After removing one individual that does not respond to any test question, we get $N=1004$. For this data, 0, 1, 2, 3, 4, and 5 denote no response, strongly disagree, disagree, neither agree not disagree, agree, and strongly agree, respectively. Therefore, $R\in\{0,1,2,3,4,5\}^{1004\times40}$. Questions 1-10, 11-20, 21-30, and 31-40 measure the personality factors Assertiveness (``AS”), Social confidence ( “SC”),  Adventurousness (``AD”), and Dominance (``DO”), respectively. These questions can be found in the 2nd matrix of Figure \ref{IPIPZTheta}. Table \ref{RealDataModularityK} reports the estimated number of latent classes and the respective modularity for each method. We see that all methods suggest $K=2$ for this data. Here, we also consider LCA-RSCn for the IPIP data because LCA-RSCn returns the largest value of modularity. Applying LCA-RSCn to IPIP with $K=2$, we get the $1004\times2$ matrix $\hat{Z}$ and the $40\times 2$ matrix $\hat{\Theta}$. We also use Class 1 and Class 2 to denote the two estimated latent classes for the IPIP data. By analyzing $\hat{Z}$, we find that there are 475 individuals in Class 1 and 529 individuals in Class 2. For visualization, we plot the estimated latent classes for the top 40 individuals and $\hat{\Theta}$ in Figure \ref{IPIPZTheta}. By carefully analyzing the heatmap of $\hat{\Theta}$, Class 1 can be interpreted as individuals who are socially passive or negative and Class 2 can be interpreted as individuals who are socially active or positive. Elements of $\hat{\Theta}$ indicate that individuals in Class 2 score consistently higher than those in Class 1 on a range of positively oriented personality traits, including expressiveness, activity, ambition, confidence, creativity, dominance, social engagement, and openness. Overall, Class 2 can be interpreted as reflecting a more socially active, assertive, and outgoing personality profile, whereas Class 1 represents a more reserved and less expressive disposition.

The application of our methods to the IPIP personality test data underscores their relevance in psychological research and assessment. By identifying latent classes of individuals with distinct personality traits, our methods facilitate a deeper understanding of population heterogeneity in psychological constructs. This can aid in developing personalized interventions, improving team composition in organizational settings, and enhancing diagnostic accuracy in clinical psychology. The ability to efficiently analyze large datasets with minimal preprocessing makes our methods valuable tools for researchers and practitioners dealing with complex behavioral data.
\begin{figure}
\centering
\includegraphics[width=1\textwidth]{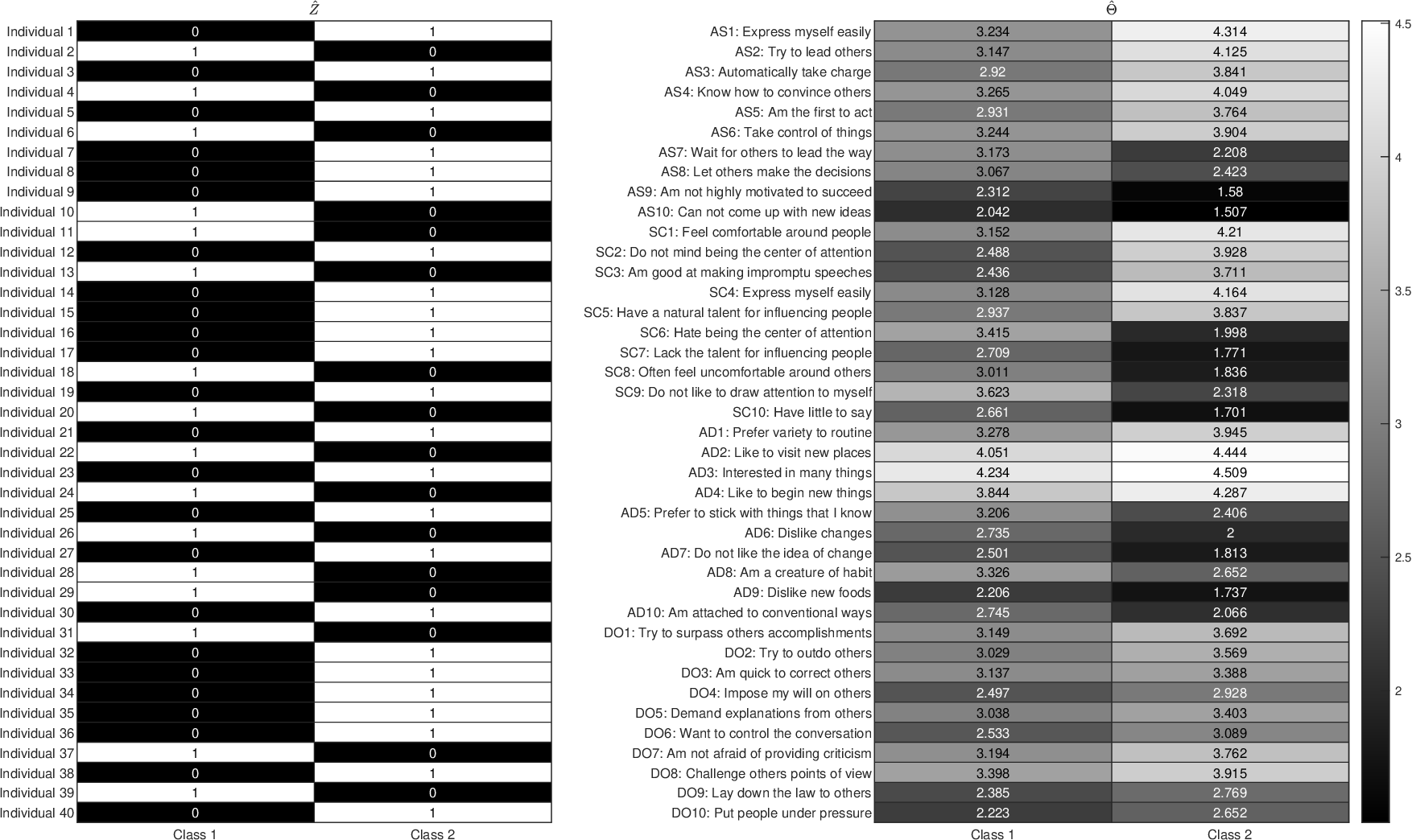}
\caption{Left matrix: Heatmap of the estimated classification matrix $\hat{Z}$ of a subset of 40 individuals for the IPIP data (white $=1$, black $=0$). Right matrix: Heatmap of the  estimated item parameter matrix $\hat{\Theta}$ for the IPIP data (continuous color scale, lighter means higher expected rating).}
\label{IPIPZTheta} %% label for entire figure
\end{figure}
\section{Conclusion}\label{sec7}
In this article, two efficient regularized spectral clustering algorithms are proposed for the latent class model in ordered categorical data with polytomous responses. This is the first time to design algorithms based on the newly defined regularized Laplacian matrix in latent class analysis for ordered categorical data. We establish the theoretical guarantees for our algorithms by considering the sparsity parameter. We show that our algorithms can consistently recover the hidden classes and the item parameter matrix under a mild condition on the sparsity parameter. We also use a well-known modularity to evaluate the quality of latent class analysis. A strategy combining modularity with our algorithms is provided to determine how many latent classes are there for real-world ordered categorical data. We conduct substantial experiments to demonstrate that our algorithms for latent class analysis are efficient and accurate, and our methods for estimating $K$ have satisfactory performances, which also supports the effectiveness of the modularity in quantifying the quality of latent class analysis in turn. Although the binomial model is a specific parametric choice, it has been used in other recent work on categorical data \citep{qing2024grade,lyu2025degree} because it can directly model responses in $\{0,1,\dots,M\}$, is convenient for theoretical analysis, and degenerates to the Bernoulli distribution when $M=1$. We hope that our regularized spectral clustering algorithms, our idea of using modularity to measure the quality of latent class analysis, and our methods for estimating $K$ could be useful for ordered categorical data in diverse fields like social, psychological, and behavioral sciences.

Recall that this paper presents six spectral algorithms in total, but the four variants in \ref{AlternativeAlgorithms} are included only for comparison. The two main methods, LCA-RSC and LCA-RSCn, have similar numerical performance and share the same theoretical guarantees under mild conditions. For practitioners facing real data with unknown true latent classes, one empirical way to choose among different algorithms for a fixed $K$ is to compare their modularity values (Equation~(\ref{Modularity})): a larger modularity indicates a partition that better captures the assortative structure of $A=RR'$, as illustrated in our real-data analysis in Table~\ref{RealDataModularityK} where we selected the method with the highest modularity. For most applications, we recommend LCA-RSCn as the default, because its normalized singular vectors yield constant inter-class distances $\sqrt{2}$ in the fourth point of Lemma~\ref{SVDPopulationLtau}, which can be advantageous when class sizes are unknown or imbalanced.

For future research, first, theoretical guarantees of LCA-RSCORS, LCA-RMK, and LCA-RLMK should be developed. Second, developing rigorous goodness-of-fit tests such as formal hypothesis testing for \(K = 1\) versus \(K > 1\) would help determine whether any latent structure actually exists in the data, thus addressing cases where no true partition is present.  Furthermore, rigorous approaches should be proposed to determine $K$ under the latent class model for ordered categorical data. In \citep{lei2016goodness,wu2024spectral}, the authors proposed rigorous methods for estimating the number of communities \(K\) in undirected networks under the stochastic block model \citep{holland1983stochastic}. However, these methods are designed for symmetric adjacency matrices and are not directly applicable to the asymmetric observed response matrix \(R\) considered in the latent class model. To address this gap, a new methodology for estimating \(K\) under LCM is needed. Recently, \citep{qing2025joint} introduced efficient approaches for estimating the numbers of asymmetric communities in directed networks, which rely on the limiting behavior of a test statistic rather than a specific distribution. It is promising to extend this idea to develop a formal goodness-of-fit test for determining \(K\) under LCM. Third, the grade of membership (GoM) model \citep{woodbury1978mathematical,erosheva2005comparing} and its extension weighted GoM \citep{qing2025mixed} allow a subject to have partial memberships in different classes and it is more general than LCM. Thus, following ideas proposed in this article, designing regularized spectral clustering algorithms to estimate a GoM or weighted GoM model is meaningful. Fourth, developing more efficient algorithms for latent class analysis to handle large-scale ordered categorical data is an appealing topic.

\section*{CRediT authorship contribution statement}
\textbf{Huan Qing} is the sole author of this article.
\section*{Declaration of competing interest}
The author declares no competing interests.
\section*{Data availability}
Data and code will be made available on request.
\section*{Acknowledgements}
Qing's work was sponsored by the Scientific Research Foundation of Chongqing University of Technology (Grant No. 2024ZDR003) and the Science and Technology Research Program of Chongqing Municipal Education Commission (Grant No. KJQN202401168).
\appendix
\section{Proofs under LCM}\label{SecProofs}
\subsection{Proof of Lemma \ref{SVDPopulationLtau}}
\begin{proof}
For the 1st statement, $\mathscr{R}=Z\Theta'$ gives $\Theta Z'=\mathscr{R}'\Rightarrow \Theta Z'Z=\mathscr{R}'Z\Rightarrow \Theta=\mathscr{R}'Z(Z'Z)^{-1}$, where $Z'Z$ is a $K\times K$ nonsingular matrix since $Z$'s rank is $K$.

For the 2nd statement, since $\mathscr{D}_{\tau}(i,i)=\tau+\sum_{j=1}^{J}\mathscr{R}(i,j)$ and $\mathscr{R}(i,j)=Z(i,:)\Theta'(j,:)=\Theta(j,\ell(i))$, we get $\mathscr{D}_{\tau}(i,i)=\mathscr{D}_{\tau}(\bar{i},\bar{i})$ if $\ell(i)=\ell(\bar{i})$. Thus, the diagonal entries of $\mathscr{D}_{\tau}$ only have $K$ distinct values since all subjects belong to the $K$ latent classes.

Since $\mathcal{L}_{\tau}=\mathscr{D}^{-1/2}_{\tau}\mathscr{R}=\mathscr{D}^{-1/2}_{\tau}Z\Theta'=U\Sigma V'$, we have $U=\mathscr{D}^{-1/2}_{\tau}Z\Theta'V\Sigma^{-1}$, which gives that $U(i,:)=\mathscr{D}^{-1/2}_{\tau}(i,i)Z(i,:)\Theta'V\Sigma^{-1}$. Thus, we have $U(i,:)=U(\bar{i},:)$ if $\ell(i)=\ell(\bar{i})$ and $U$ has $K$ distinct rows since all subjects belong to the $K$ latent classes. Meanwhile, since $\mathscr{D}_{\tau}(i,i)=\mathscr{D}_{\tau}(\bar{i},\bar{i})$ when $\ell(i)=\ell(\bar{i})$, we have $U=\mathscr{D}^{-1/2}_{\tau}Z\Theta'V\Sigma^{-1}=Z\mathscr{D}^{-1/2}_{\tau}(\mathcal{I},\mathcal{I})\Theta V'\Sigma^{-1}$, i.e., $U=ZX$ with $X$ being $\mathscr{D}^{-1/2}_{\tau}(\mathcal{I},\mathcal{I})\Theta V'\Sigma^{-1}\equiv U(\mathcal{I},:)$.

For the 3rd statement, we have $U_{*}(i,:)=\frac{U(i,:)}{\|U(i,:)\|_{F}}=\frac{Z(i,:)X}{\|Z(i,:)X\|_{F}}=\frac{X(\ell(i),:)}{\|X(\ell(i),:)\|_{F}}=U_{*}(\bar{i},:)$ if $\ell(i)=\ell(\bar{i})$, which gives that $U_{*}$ has $K$ distinct rows and $U_{*}=ZY$ with $Y$ being $U_{*}(\mathcal{I},:)$.

For the 4th statement, since $U'U=I_{K\times K}$ and $U=ZX$, we have $X'Z'ZX=I_{K\times K}$, which gives that $XX'=(Z'Z)^{-1}=\mathrm{diag}(\frac{1}{N_{1}}, \frac{1}{N_{2}}, \ldots, \frac{1}{N_{K}})$. Then we get $X(k,:)X'(k,:)=\frac{1}{N_{k}}$ and $X(k,:)X'(l,:)=0$ for $k\neq l\in[K]$, which gives that $\|X(k,:)-X(l,:)\|_{F}=\sqrt{(X(k,:)-X(l,:))(X(k,:)-X(l,:))'}=\sqrt{\frac{1}{N_{k}}+\frac{1}{N_{l}}}$ when $k\neq l$.

Since $Y=U_{*}(\mathcal{I},:)$ and $X=U(\mathcal{I},:)$, we have $Y(k,:)=\frac{X(k,:)}{\|X(k,:)\|_{F}}=\frac{X(k,:)}{\sqrt{X(k,:)X'(k,:)}}=\sqrt{N_{k}}X(k,:)$ for $k\in[K]$. Thus, we have $Y(k,:)Y'(l,:)=\sqrt{N_{k}N_{l}}X(k,:)X'(l,:)=0$ when $k\neq l$ and $Y(k,:)Y'(k,:)=N_{k}X(k,:)X'(k,:)=1$ for $k\in[K]$. Finally, we have $\|Y(k,:)-Y(l,:)\|_{F}=\sqrt{(Y(k,:)-Y(l,:))(Y(k,:)-Y(l,:))'}=\sqrt{2}$ when $k\neq l$. This completes the proof.
\end{proof}
\subsection{Proof of Lemma \ref{boundLLCM}}
\begin{proof}
First, the following result holds.
\begin{align*}
\|L_{\tau}-\mathcal{L}_{\tau}\|&=\|D_{\tau}^{-1/2}R-\mathscr{D}^{-1/2}_{\tau}\mathscr{R}\|=\|D_{\tau}^{-1/2}R-\mathscr{D}^{-1/2}_{\tau}R+\mathscr{D}^{-1/2}_{\tau}R-\mathscr{D}^{-1/2}_{\tau}\mathscr{R}\|\\
&\leq\|\mathscr{D}^{-1/2}_{\tau}R-\mathscr{D}^{-1/2}_{\tau}\mathscr{R}\|+\|D_{\tau}^{-1/2}R-\mathscr{D}^{-1/2}_{\tau}R\|.
\end{align*}
Next, we bound the two terms $\|\mathscr{D}^{-1/2}_{\tau}R-\mathscr{D}^{-1/2}_{\tau}\mathscr{R}\|$ and $\|D_{\tau}^{-1/2}R-\mathscr{D}^{-1/2}_{\tau}R\|$ separately.

We bound the first term using Theorem 1.6 in \cite{tropp2012user}. We describe this theorem below
\begin{thm}\label{Bern}
(Theorem 1.6 of \citep{tropp2012user}) Let $\{X_{n}\}$ be a finite sequence that are independent, random matrices with dimensions $N\times J$. Assumes that
\begin{align*}
\mathbb{E}(X_{n})=0, \mathrm{and~}\|X_{n}\|\leq r~\mathrm{almost~surely}.
\end{align*}
Then, for any $t\in[0,+\infty)$,
\begin{align*}
\mathbb{P}(\|\sum_{n}X_{n}\|\geq t)\leq (N+J)\cdot \mathrm{exp}(\frac{-t^{2}/2}{\sigma^{2}+rt/3}),
\end{align*}
where $\sigma^{2}:=\mathrm{max}\{\|\sum_{n}\mathbb{E}(X_{n}X'_{n})\|,\|\sum_{n}\mathbb{E}(X'_{n}X_{n})\|\}$.
\end{thm}
Set $e_{i}$ as an $N$-by-$1$ vector with $e_{i}(i)=1$ and $0$ elsewhere for $i\in[N]$, and $\tilde{e}_{j}$ as a $J$-by-$1$ vector with $\tilde{e}_{j}(j)=1$ and $0$ elsewhere for $j\in[J]$.
Let $W=\mathscr{D}^{-1/2}_{\tau}R-\mathscr{D}^{-1/2}_{\tau}\mathscr{R}$. Then we can rewrite $W$ as $W=\sum_{i\in[N]}\sum_{j\in[J]}W(i,j)e_{i}\tilde{e}'_{j}=\sum_{i\in[N]}\sum_{j\in[J]}W^{(i,j)}$, where $W^{(i,j)}=W(i,j)e_{i}\tilde{e}'_{j}$. Under $\mathrm{LCM}(Z,\Theta)$, for $i\in[N],j\in[J]$, $\mathbb{E}(W(i,j))=0$ holds and
\begin{align*}
\|W^{(i,j)}\|&=\frac{1}{\sqrt{\tau+\mathscr{D}(i,i)}}\|(R(i,j)-\mathscr{R}(i,j))e_{i}\tilde{e}'_{j}\|\leq\frac{1}{\sqrt{\tau+\delta_{\mathrm{min}}}}|R(i,j)-\mathscr{R}(i,j)|\|e_{i}\tilde{e}'_{j}\|\leq\frac{M}{\sqrt{\tau+\delta_{\mathrm{min}}}},
\end{align*}
where the last inequality holds since $R(i,j)\in\{0,1,2,\ldots, M\}, 0\leq \mathscr{R}(i,j)\leq \rho\leq M$, and $\|e_{i}\tilde{e}'_{j}\|=1$. Thus we have $r=\frac{M}{\sqrt{\tau+\delta_{\mathrm{min}}}}$.

Then we aim at bounding the variance term $\sigma^{2}=\mathrm{max}\{\|\sum_{i\in[N]}\sum_{j\in[J]}\mathbb{E}(W^{(i,j)}(W^{(i,j)})')\|,\|\sum_{i\in[N]}\sum_{j\in[J]}\mathbb{E}((W^{(i,j)})'W^{(i,j)})\|\}$. Because $R(i,j)\sim\mathrm{Binomial}(M,\frac{\mathscr{R}(i,j)}{M})$, we have $\mathbb{E}(W^{2}(i,j))=\frac{1}{\tau+\mathscr{D}(i,i)}\mathbb{E}((R(i,j)-\mathscr{R}(i,j))^{2})\leq\frac{1}{\tau+\delta_{\mathrm{min}}}\mathrm{Var}(R(i,j))=\frac{1}{\tau+\delta_{\mathrm{min}}}M\frac{\mathscr{R}(i,j)}{M}(1-\frac{\mathscr{R}(i,j)}{M})=\frac{1}{\tau+\delta_{\mathrm{min}}}\mathscr{R}(i,j)(1-\frac{\mathscr{R}(i,j)}{M})\leq \frac{1}{\tau+\delta_{\mathrm{min}}}\mathscr{R}(i,j)\leq \frac{\rho}{\tau+\delta_{\mathrm{min}}}$, which gives
\begin{align*}
\|\sum_{i=1}^{N}\sum_{j=1}^{J}\mathbb{E}(W^{(i,j)}(W^{(i,j)})')\|=\|\sum_{i=1}^{N}\sum_{j=1}^{J}\mathbb{E}(W^{2}(i,j))e_{i}\tilde{e}'_{j}\tilde{e}_{j}e'_{i}\|=\|\sum_{i=1}^{N}\sum_{j=1}^{J}\mathbb{E}(W^{2}(i,j))e_{i}e'_{i}\|\leq \frac{\rho J}{\tau+\delta_{\mathrm{min}}}.
\end{align*}
Similarly, we have $\|\sum_{i=1}^{N}\sum_{j=1}^{J}\mathbb{E}((W^{(i,j)})'W^{(i,j})\|\leq \frac{\rho N}{\tau+\delta_{\mathrm{min}}}$. Hence, we have
\begin{align*}
\sigma^{2}\leq \frac{\rho\mathrm{~max}(N,J)}{\tau+\delta_{\mathrm{min}}}.
\end{align*}
Let $t=\frac{\alpha+1+\sqrt{\alpha^{2}+20\alpha+19}}{3\sqrt{\tau+\delta_{\mathrm{min}}}}\sqrt{\rho \mathrm{~max}(N,J)\mathrm{log}(N+J)}$, where $\alpha$ is any positive value. Based on Theorem \ref{Bern} and Assumption \ref{Assum1}, we obtain the following result
\begin{align*}
&\mathbb{P}(\|W\|\geq t)\leq (N+J)\mathrm{exp}(-\frac{t^{2}/2}{\sigma^{2}+\frac{rt}{3}})\leq (N+J)\mathrm{exp}(-\frac{t^{2}/2}{\frac{\rho\mathrm{~max}(N,J)}{\tau+\delta_{\mathrm{min}}}+\frac{tM}{3\sqrt{\tau+\delta_{\mathrm{min}}}}})\\
&=(N+J)\mathrm{exp}(-(\alpha+1)\mathrm{log}(N+J)\cdot \frac{1}{\frac{2(\alpha+1)\rho\mathrm{~max}(N,J)\mathrm{log}(N+J)}{t^{2}(\tau+\delta_{\mathrm{min}})}+\frac{2(\alpha+1)M}{3\sqrt{\tau+\delta_{\mathrm{min}}}}\frac{\mathrm{log}(N+J)}{t}})\\
&=(N+J)\mathrm{exp}(-(\alpha+1)\mathrm{log}(N+J)\cdot \frac{1}{\frac{18}{(\sqrt{\alpha+19}+\sqrt{\alpha+1})^{2}}+\frac{2\sqrt{\alpha+1}}{\sqrt{\alpha+19}+\sqrt{\alpha+1}}\sqrt{\frac{M^{2}\mathrm{log}(N+J)}{\rho\mathrm{~max}(N,J)}}})\\
&\leq (N+J)\mathrm{exp}(-(\alpha+1)\mathrm{log}(N+J))=\frac{1}{(N+J)^{\alpha}}.
\end{align*}
Therefore, with probability at least $1-o(\frac{1}{(N+J)^{\alpha}})$, we have
\begin{align*}
\|\mathscr{D}^{-1/2}_{\tau}R-\mathscr{D}^{-1/2}_{\tau}\mathscr{R}\|\leq \frac{\alpha+1+\sqrt{\alpha^{2}+20\alpha+19}}{3\sqrt{\tau+\delta_{\mathrm{min}}}}\sqrt{\rho \mathrm{~max}(N,J)\mathrm{log}(N+J)}.
\end{align*}

For the second term $\|D_{\tau}^{-1/2}R-\mathscr{D}^{-1/2}_{\tau}R\|$, let $\tilde{D}$ be a $J\times J$ diagonal matrix with $\tilde{D}(j,j)=\sum_{i=1}^{N}R(i,j)$ for $j\in[J]$. Since $R(i,j)\in\{0,1,2,\ldots,M\}$, we see that $0\leq\tilde{D}(j,j)\leq NM$, i.e., $\|\tilde{D}^{1/2}\|\leq\sqrt{MN}$. Let $I$ be the $N\times N$ identity matrix. Based on the fact that $\|D^{-1/2}R\tilde{D}^{-1/2}\|=1$ and $\|D^{-1/2}_{\tau}D^{1/2}\|\leq1$ provided that $\tau\geq0$, we have
\begin{align*}
\|D_{\tau}^{-1/2}R-\mathscr{D}^{-1/2}_{\tau}R\|&=\|(D^{-1/2}_{\tau}-\mathscr{D}^{-1/2}_{\tau})D^{1/2}_{\tau}D^{-1/2}_{\tau}D^{1/2}D^{-1/2}R\tilde{D}^{-1/2}\tilde{D}^{1/2}\|\\
&\leq\|I-D^{1/2}_{\tau}\mathscr{D}^{-1/2}_{\tau}\|\|D^{-1/2}_{\tau}D^{1/2}\|\|D^{-1/2}R\tilde{D}^{-1/2}\|\|\tilde{D}^{1/2}\|\leq\sqrt{MN}\|I-D^{1/2}_{\tau}\mathscr{D}^{-1/2}_{\tau}\|\\
&=\sqrt{MN}\mathrm{~max}_{i\in[N]}|\sqrt{\frac{D(i,i)+\tau}{\mathscr{D}(i,i)+\tau}}-1|\leq\sqrt{MN}\mathrm{~max}_{i\in[N]}|\frac{D(i,i)+\tau}{\mathscr{D}(i,i)+\tau}-1|\\
&=\sqrt{MN}\mathrm{~max}_{i\in[N]}|\frac{D(i,i)-\mathscr{D}(i,i)}{\tau+\mathscr{D}(i,i)}|\leq\frac{\sqrt{MN}}{\tau+\delta_{\mathrm{min}}}\mathrm{max}_{i\in[N]}|D(i,i)-\mathscr{D}(i,i)|\\
&=\frac{\sqrt{MN}}{\tau+\delta_{\mathrm{min}}}\mathrm{max}_{i\in[N]}|\sum_{j=1}^{J}(R(i,j)-\mathscr{R}(i,j))|.
\end{align*}

To bound $|\sum_{j=1}^{J}(R(i,j)-\mathscr{R}(i,j))|$ for $i\in[N]$, we use Theorem 1.4 in \cite{tropp2012user}. This theorem is stated below
\begin{thm}\label{BernSquare}
(Theorem 1.4 of \citep{tropp2012user}) Let $\{\tilde{X}_{\tilde{n}}\}$ be a finite sequence that are independent, random matrices with dimension $d$. Assume that
\begin{align*}
\mathbb{E}(\tilde{X}_{\tilde{n}})=0, \mathrm{and~}\|\tilde{X}_{\tilde{n}}\|\leq \tilde{r}~\mathrm{almost~surely}.
\end{align*}
Then, for any $\tilde{t}\in[0,+\infty)$,
\begin{align*}
\mathbb{P}(\|\sum_{\tilde{n}}\tilde{X}_{\tilde{n}}\|\geq \tilde{t})\leq d\cdot \mathrm{exp}(\frac{-\tilde{t}^{2}/2}{\tilde{\sigma}^{2}+\tilde{r}\tilde{t}/3}) \mathrm{~where~}\tilde{\sigma}^{2}:=\|\sum_{\tilde{n}}\mathbb{E}(\tilde{X}^{2}_{\tilde{n}})\|.
\end{align*}
\end{thm}
To use Theorem \ref{BernSquare}, we treat $R(i,j)-\mathscr{R}(i,j)$ as a random matrix with dimension 1. Set $\tilde{W}=\sum_{j=1}^{J}(R(i,j)-\mathscr{R}(i,j))=\sum_{j=1}^{J}\tilde{W}_{j}$, where $\tilde{W}_{j}=R(i,j)-\mathscr{R}(i,j)$. Under LCM, we have $\mathbb{E}(\tilde{W}_{j})=0$, $\|\tilde{W}_{j}\|=|\tilde{W}_{j}|=|R(i,j)-\mathscr{R}(i,j)|\leq M$ (i.e., $\tilde{r}=M$), and $\mathbb{E}(\tilde{W}^{2}_{j})=\mathbb{E}((R(i,j)-\mathscr{R}(i,j))^{2})=\mathrm{Var}(R(i,j))=M\frac{\mathscr{R}(i,j)}{M}(1-\frac{\mathscr{R}(i,j)}{M})\leq\rho$ which gives that $\tilde{\sigma}^{2}=\|\sum_{j=1}^{J}\mathbb{E}(\tilde{W}^{2}_{j})\|\leq\rho J$.

Set $\tilde{t}=\frac{\alpha+1+\sqrt{\alpha^{2}+20\alpha+19}}{3}\sqrt{\rho \mathrm{max}(N,J)\mathrm{log}(N+J)}$. According to Theorem \ref{BernSquare} and Assumption \ref{Assum1}, we get
\begin{align*}
\mathbb{P}(\|\tilde{W}\|\geq\tilde{t})&\leq \mathrm{exp}(-\frac{\tilde{t}^{2}/2}{\tilde{\sigma}^{2}+\frac{\tilde{r}\tilde{t}}{3}})\leq \mathrm{exp}(-\frac{\tilde{t}^{2}/2}{\rho J+\frac{\tilde{t}M}{3}})\\
&=\mathrm{exp}(-(\alpha+1)\mathrm{log}(N+J)\cdot \frac{1}{\frac{2(\alpha+1)\rho J\mathrm{log}(N+J)}{\tilde{t}^{2}}+\frac{2(\alpha+1)M}{3}\frac{\mathrm{log}(N+J)}{\tilde{t}}})\\
&\leq\mathrm{exp}(-(\alpha+1)\mathrm{log}(N+J)\cdot \frac{1}{\frac{2(\alpha+1)\rho \mathrm{max}(N,J)\mathrm{log}(N+J)}{\tilde{t}^{2}}+\frac{2(\alpha+1)M}{3}\frac{\mathrm{log}(N+J)}{\tilde{t}}})\\
&=\mathrm{exp}(-(\alpha+1)\mathrm{log}(N+J)\cdot \frac{1}{\frac{18}{(\sqrt{\alpha+19}+\sqrt{\alpha+1})^{2}}+\frac{2\sqrt{\alpha+1}}{\sqrt{\alpha+19}+\sqrt{\alpha+1}}\sqrt{\frac{M^{2}\mathrm{log}(N+J)}{\rho\mathrm{max}(N,J)}}})\\
&\leq\mathrm{exp}(-(\alpha+1)\mathrm{log}(N+J))=\frac{1}{(N+J)^{\alpha+1}}.
\end{align*}

Therefore, with probability as least $1-o(\frac{1}{(N+J)^{\alpha}})$, we have
\begin{align*}
\|D^{-1/2}_{\tau}R-\mathscr{D}^{-1/2}_{\tau}R\|\leq\frac{\alpha+1+\sqrt{\alpha^{2}+20\alpha+19}}{3}\frac{\sqrt{MN}}{\tau+\delta_{\mathrm{min}}}\sqrt{\rho\mathrm{max}(N,J)\mathrm{log}(N+J)}.
\end{align*}

Combining the two parts gives
\begin{align*}
\|L_{\tau}-\mathcal{L}_{\tau}\|=O((1+\sqrt{\frac{MN}{\tau+\delta_{\mathrm{min}}}})\sqrt{\frac{\rho\mathrm{max}(N,J)\mathrm{log}(N+J)}{\tau+\delta_{\mathrm{min}}}}),
\end{align*}
with probability at least $1-o(\frac{1}{(N+J)^{3}})$ by setting $\alpha=3$.
\end{proof}
\subsection{Proof of Theorem \ref{mainRLCM}}
\begin{proof}
The following lemma is useful for our analysis.
\begin{lem}\label{boundUVWLCM}
Under $\mathrm{LCM}(Z,\Theta)$, we have
\begin{align*}	\|\hat{U}\hat{O}-U\|_{F}\leq\frac{2\sqrt{2K(\tau+\delta_{\mathrm{max}})}\|L_{\tau}-\mathcal{L}_{\tau}\|}{\rho\sigma_{K}(B)\sqrt{N_{\mathrm{min}}}} \mathrm{~and~}\|\hat{U}_{*}\hat{O}-U_{*}\|_{F}\leq\frac{4\sqrt{2K(\tau+\delta_{\mathrm{max}})N_{\mathrm{max}}}\|L_{\tau}-\mathcal{L}_{\tau}\|}{\rho\sigma_{K}(B)\sqrt{N_{\mathrm{min}}}},
\end{align*}
where $\hat{O}$ is an orthogonal matrix.
\end{lem}
\begin{proof}
Let $\hat{L}_{\tau}=\hat{U}\hat{\Sigma}\hat{V}'$, an $N\times J$ matrix with rank $K$. The following inequality holds by the proof of Lemma 3 in \citep{zhou2019analysis},
\begin{align*}	\|\hat{U}\hat{O}-U\|_{F}\leq\frac{\sqrt{2K}\|\hat{L}_{\tau}-\mathcal{L}_{\tau}\|}{\sqrt{\lambda_{K}(\mathcal{L}_{\tau}\mathcal{L}'_{\tau})}},
\end{align*}
where $\lambda_{k}(\cdot)$ denotes the $k$-th largest eigenvalue in magnitude for a matrix.

Due to the fact that $\hat{L}_{\tau}$ is the top $K$ SVD of $L_{\tau}$ and $\mathcal{L}_{\tau}$ is also a rank $K$ matrix, we obtain $\|L_{\tau}-\hat{L}_{\tau}\|\leq\|L_{\tau}-\mathcal{L}_{\tau}\|$, which gives that $\|\hat{L}_{\tau}-\mathcal{L}_{\tau}\|=\|\hat{L}_{\tau}-L_{\tau}+L_{\tau}-\mathcal{L}_{\tau}\|\leq 2\|L_{\tau}-\mathcal{L}_{\tau}\|$. Thus we have
\begin{align*} \|\hat{U}\hat{O}-U\|_{F}\leq\frac{2\sqrt{2K}\|L_{\tau}-\mathcal{L}_{\tau}\|}{\sqrt{\lambda_{K}(\mathcal{L}_{\tau}\mathcal{L}'_{\tau})}}.
\end{align*}
Next we obtain a lower bound of $\lambda_{K}(\mathcal{L}_{\tau}\mathcal{L}'_{\tau})$. Recall that $\mathscr{R}=Z\Theta'=\rho ZB'$ and $\mathcal{L}_{\tau}=\mathscr{D}^{-1/2}_{\tau}\mathscr{R}$, we have
\begin{align*}
\lambda_{K}(\mathcal{L}_{\tau}\mathcal{L}'_{\tau})&=\lambda_{K}(\mathscr{D}^{-1/2}_{\tau}\rho ZB'BZ'\rho\mathscr{D}^{-1/2}_{\tau})=\rho^{2}\lambda_{K}(ZB'BZ'\mathscr{D}^{-1}_{\tau})\geq\rho^{2}\lambda_{K}(ZB'BZ')\lambda_{K}(\mathscr{D}^{-1}_{\tau})\\
&=\rho^{2}\lambda_{K}(B'BZ'Z)\lambda_{K}(\mathscr{D}^{-1}_{\tau})\geq\rho^{2}\lambda_{K}(B'B)\lambda_{K}(Z'Z)\lambda_{K}(\mathscr{D}^{-1}_{\tau})=\frac{\rho^{2}N_{\mathrm{min}}\sigma^{2}_{K}(B)}{\tau+\delta_{\mathrm{max}}}.
\end{align*}
Therefore, we have
\begin{align*}	\|\hat{U}\hat{O}-U\|_{F}\leq\frac{2\sqrt{2K}\|L_{\tau}-\mathcal{L}_{\tau}\|\sqrt{\tau+\delta_{\mathrm{max}}}}{\rho\sigma_{K}(B)\sqrt{N_{\mathrm{min}}}}.
\end{align*}
For $i\in[N]$, Lemma F.2 \citep{mao2018overlapping} gives
\begin{align*}
\|\hat{U}_{*}(i,:)\hat{O}-U_{*}(i,:)\|_{F}\leq\frac{2\|\hat{U}(i,:)\hat{O}-U(i,:)\|_{F}}{\|U(i,:)\|_{F}}.
\end{align*}
Let $\mu=\mathrm{min}_{i\in[N]}\|U(i,:)\|_{F}$. Then the above inequality gives
\begin{align*}
\|\hat{U}_{*}\hat{O}-U_{*}\|_{F}=\sqrt{\sum_{i=1}^{N}\|\hat{U}_{*}(i,:)\hat{O}-U_{*}(i,:)\|^{2}_{F}}\leq\frac{2\|\hat{U}\hat{O}-U\|_{F}}{\mu}.
\end{align*}
By the proof of Lemma \ref{SVDPopulationLtau}, we know that  $\|U(i,:)\|_{F}=\frac{1}{\sqrt{N_{\ell(i)}}}\geq\frac{1}{\sqrt{N_{\mathrm{max}}}}$, i.e., $\mu\geq\frac{1}{\sqrt{N_{\mathrm{max}}}}$. Therefore, we have
\begin{align*}
\|\hat{U}_{*}\hat{O}-U_{*}\|_{F}\leq2\sqrt{N_{\mathrm{max}}}\|\hat{U}\hat{O}-U\|_{F}\leq\frac{4\sqrt{2K(\tau+\delta_{\mathrm{max}})N_{\mathrm{max}}}\|L_{\tau}-\mathcal{L}_{\tau}\|}{\rho\sigma_{K}(B)\sqrt{N_{\mathrm{min}}}}.
\end{align*}
\end{proof}
Next, we obtain the theoretical error bounds of LCA-RSC and LCA-RSCn separately.
\begin{itemize}
  \item (error rates of LCA-RSC) When $\hat{Z}$ and $\hat{\Theta}$ are obtained from Algorithm \ref{alg:RSC}: Let $\varsigma_{LCA-RSC}>0$ be a small value. Based on Lemma 2 of \citep{joseph2016impact} and the 1st equality in the 4th statement of Lemma \ref{SVDPopulationLtau}, we have the following result: if
      \begin{align}\label{RSCerror}
      \frac{\sqrt{K}}{\varsigma_{LCA-RSC}}\|U-\hat{U}\hat{O}\|_{F}(\frac{1}{\sqrt{N_{k}}}+\frac{1}{\sqrt{N_{l}}})\leq\sqrt{\frac{1}{N_{k}}+\frac{1}{N_{l}}}, \mathrm{~for~}1\leq k\neq l\leq K,
      \end{align}
      then the Clustering error $\hat{f}_{LCA-RSC}=O(\varsigma^{2}_{LCA-RSC})$. We find that setting $\varsigma_{LCA-RSC}=\sqrt{\frac{2KN_{\mathrm{max}}}{N_{\mathrm{min}}}}\|U-\hat{U}\hat{O}\|_{F}$ makes Equation (\ref{RSCerror}) hold. Thus we have $\hat{f}_{LCA-RSC}=O(\varsigma^{2}_{LCA-RSC})=O(\frac{KN_{\mathrm{max}}\|U-\hat{U}\hat{O}\|^{2}_{F}}{N_{\mathrm{min}}})$. Then by the 1st inequality in Lemma \ref{boundUVWLCM}, we have
      \begin{align*}
      \hat{f}_{LCA-RSC}=O(\frac{K^{2}(\tau+\delta_{\mathrm{max}})N_{\mathrm{max}}\|L_{\tau}-\mathcal{L}_{\tau}\|^{2}}{\rho^{2}\sigma^{2}_{K}(B)N^{2}_{\mathrm{min}}}).
      \end{align*}
      By Lemma \ref{boundLLCM}, we have
      \begin{align*}
      \hat{f}_{LCA-RSC}=O((1+\frac{MN}{\tau+\delta_{\mathrm{min}}}+2\sqrt{\frac{MN}{\tau+\delta_{\mathrm{min}}}})\frac{K^{2}(\tau+\delta_{\mathrm{max}})N_{\mathrm{max}}\mathrm{max}(N,J)\mathrm{log}(N+J)}{(\tau+\delta_{\mathrm{min}})\rho\sigma^{2}_{K}(B)N^{2}_{\mathrm{min}}}).
      \end{align*}

      Next, we bound $\frac{\|\hat{\Theta}-\Theta\|_{F}}{\|\Theta\|_{F}}$. Because $\tilde{\Theta}$ is almost always the same as $\hat{\Theta}$ in practice for both algorithms, to simplify our theoretical analysis, we set $\hat{\Theta}$ as $\tilde{\Theta}$ directly. Following a similar proof of Lemma \ref{boundLLCM} gives that $\|R-\mathscr{R}\|=O(\sqrt{\rho~\mathrm{max}(N,J)\mathrm{log}(N+J)})$ holds with probability at least $1-o(\frac{1}{(N+J)^{3}})$ when Assumption \ref{Assum1} is satisfied. Thus, we have
\begin{align*}
\|\hat{\Theta}-\Theta\|&=\|R'\hat{Z}(\hat{Z}'\hat{Z})^{-1}-\mathscr{R}'Z(Z'Z)^{-1}\|=\|(R'-\mathscr{R}')\hat{Z}(\hat{Z}'\hat{Z})^{-1}+\mathscr{R}'(\hat{Z}(\hat{Z}'\hat{Z})^{-1}-Z(Z'Z)^{-1})\|\\
&\leq\|(R'-\mathscr{R}')\hat{Z}(\hat{Z}'\hat{Z})^{-1}\|+\|\mathscr{R}'(\hat{Z}(\hat{Z}'\hat{Z})^{-1}-Z(Z'Z)^{-1})\|\\
&\leq\|R-\mathscr{R}\|\|\hat{Z}(\hat{Z}'\hat{Z})^{-1}\|+\|\rho ZB'\|\|\hat{Z}(\hat{Z}'\hat{Z})^{-1}-Z(Z'Z)^{-1}\|\\
&\leq\|R-\mathscr{R}\|\|\hat{Z}(\hat{Z}'\hat{Z})^{-1}\|+\rho \|Z\|\|B\|\|\hat{Z}(\hat{Z}'\hat{Z})^{-1}-Z(Z'Z)^{-1}\|\\
&=O(\sqrt{\frac{\rho~\mathrm{max}(N,J)\mathrm{log}(N+J)}{N_{\mathrm{min}}}})+O(\frac{\rho\sqrt{N_{\mathrm{max}}}\sigma_{1}(B)}{\sqrt{N_{\mathrm{min}}}}),
\end{align*}
where we use $\frac{1}{\sqrt{N_{\mathrm{min}}}}$ to roughly approximate $\|\hat{Z}(\hat{Z}'\hat{Z})^{-1}\|$ and $\|\hat{Z}(\hat{Z}'\hat{Z}^{-1})-Z(Z'Z)^{-1}\|$ since $\|Z(Z'Z)^{-1}\|=\frac{1}{\sqrt{N_{\mathrm{min}}}}$ and $\hat{Z}$ is close to $Z$. Since $\|\hat{\Theta}-\Theta\|_{F}\leq\sqrt{K}\|\hat{\Theta}-\Theta\|$, we have
\begin{align*}
\|\hat{\Theta}-\Theta\|_{F}=O(\sqrt{\frac{\rho K~\mathrm{max}(N,J)\mathrm{log}(N+J)}{N_{\mathrm{min}}}})+O(\frac{\rho\sqrt{KN_{\mathrm{max}}}\sigma_{1}(B)}{\sqrt{N_{\mathrm{min}}}}).
\end{align*}
Since $\|\Theta\|_{F}\geq\|\Theta\|=\rho\sigma_{1}(B)$, we have
\begin{align*}
\frac{\|\hat{\Theta}-\Theta\|_{F}}{\|\Theta\|_{F}}\leq\frac{\|\hat{\Theta}-\Theta\|_{F}}{\rho\sigma_{1}(B)}=O(\sqrt{\frac{K~\mathrm{max}(N,J)\mathrm{log}(N+J)}{\rho N_{\mathrm{min}}\sigma^{2}_{1}(B)}})+O(\sqrt{\frac{KN_{\mathrm{max}}}{N_{\mathrm{min}}}}).
\end{align*}
Because $\rho$ does not appear in the term $O(\sqrt{KN_{\mathrm{max}}/N_{\mathrm{min}}})$ and this term is $O(1)$ if we further let $K=O(1), N_{\mathrm{max}}=O(N/K)$, and $N_{\mathrm{min}}=O(N/K)$. Thus this term can be removed from the error bound and we have
\begin{align*}
\frac{\|\hat{\Theta}-\Theta\|_{F}}{\|\Theta\|_{F}}\leq O(\sqrt{\frac{K~\mathrm{max}(N,J)\mathrm{log}(N+J)}{\rho N_{\mathrm{min}}\sigma^{2}_{1}(B)}})\leq O(\sqrt{\frac{K~\mathrm{max}(N,J)\mathrm{log}(N+J)}{\rho N_{\mathrm{min}}\sigma^{2}_{K}(B)}}).
\end{align*}
  \item (error rates of LCA-RSCn) When $\hat{Z}$ and $\hat{\Theta}$ are obtained from LCA-RSCn: Let $\varsigma_{LCA-RSCn}>0$ be a small value. By Lemma 2 of \citep{joseph2016impact} and the 2nd equality in the 4th statement of Lemma \ref{SVDPopulationLtau}, if
      \begin{align}\label{RSCnerror}
      \frac{\sqrt{K}}{\varsigma_{LCA-RSCn}}\|U_{*}-\hat{U}_{*}\hat{O}\|_{F}(\frac{1}{\sqrt{N_{k}}}+\frac{1}{\sqrt{N_{l}}})\leq\sqrt{2}, \mathrm{~for~}1\leq k\neq l\leq K,
      \end{align}
      then $\hat{f}_{LCA-RSCn}=O(\varsigma^{2}_{LCA-RSCn})$. Letting $\varsigma_{LCA-RSCn}=\sqrt{\frac{2K}{N_{\mathrm{min}}}}\|U_{*}-\hat{U}_{*}\hat{O}\|_{F}$ makes Equation (\ref{RSCerror}) hold. Thus $\hat{f}_{LCA-RSCn}=O(\varsigma^{2}_{LCA-RSCn})=O(\frac{K\|U_{*}-\hat{U}_{*}\hat{O}\|^{2}_{F}}{N_{\mathrm{min}}})$. Then the 2nd inequality in Lemma \ref{boundUVWLCM} gives
      \begin{align*}
      \hat{f}_{LCA-RSCn}=O(\frac{K^{2}(\tau+\delta_{\mathrm{max}})N_{\mathrm{max}}\|L_{\tau}-\mathcal{L}_{\tau}\|^{2}}{\rho^{2}\sigma^{2}_{K}(B)N^{2}_{\mathrm{min}}}).
      \end{align*}
We see that $\hat{f}_{LCA-RSCn}$ equals to $\hat{f}_{LCA-RSC}$. Meanwhile, the bound for $\frac{\|\hat{\Theta}-\Theta\|_{F}}{\|\Theta\|_{F}}$ is the same as that of LCA-RSC.
\end{itemize}
\end{proof}
\section{Alternative algorithms}\label{AlternativeAlgorithms}
Here, we provide four alternative spectral clustering algorithms that can also fit LCM. These variants are used only for comparison: their performance is similar to that of the main algorithms LCA-RSC and LCA-RSCn, so readers primarily interested in the main contribution may skip this appendix.
\subsubsection{LCA-RSCORS}
We call the first alternative method LCA-RSCORS, where RSCORS is short for regularized spectral clustering on ratios-of-singular-vectors. Let $U=[\eta_{1}, \eta_{2}, \ldots, \eta_{K}]$, where $\eta_{k}$ is the left singular vector of the $k$-th largest singular value of $\mathcal{L}_{\tau}$. Define the $N\times (K-1)$ matrix $\Xi$ such that $\Xi(i,k)=\frac{\eta_{k+1}(i)}{\eta_{1}(i)}$ for $i\in[N], k\in[K-1]$. According to the proof of Lemma \ref{SVDPopulationLtau}, we get $U=\mathscr{D}^{-1/2}_{\tau}Z\Theta'V\Sigma^{-1}$, which gives that $\Xi=Z\Xi(\mathcal{I},:)$, i..e, $\Xi$ has $K$ distinct rows and $\Xi(i,:)=\Xi(\bar{i},:)$ if $\ell(i)=\ell(\bar{i})$. Thus running the K-means approach to all rows of $\Xi$ exactly recovers $Z$. The above analysis suggests the following algorithm called Ideal LCA-RSCORS. Input: $\mathscr{R}, K$, and $\tau$. Output: $Z$ and $\Theta$.
\begin{itemize}
  \item Calculate $\mathcal{L}_{\tau}$ by Equation (\ref{PopulationL}).
  \item Obtain $U\Sigma V'$, the top $K$ SVD of $\mathcal{L}_{\tau}$. Then calculate the $N\times (K-1)$ matrix $\Xi$ from $U.$
 \item Run K-means algorithm on all rows of $\Xi$ with $K$ clusters to obtain $Z$.
 \item Recover $\Theta$ by setting $\Theta=\mathscr{R}'Z(Z'Z)^{-1}$.
\end{itemize}

For the real case with only known the observed response matrix $R$, let $\hat{U}=[\hat{\eta}_{1}, \hat{\eta}_{2}, \ldots,\hat{\eta}_{K}]$, where $\hat{\eta}_{k}$ is the left singular vector of the $k$-th largest singular value of the regularized Laplacian matrix $L_{\tau}$ for $k\in[K]$. Let $\hat{\Xi}$ be an $N\times (K-1)$ matrix such that $\hat{\Xi}(i,k)=\frac{\hat{\eta}_{k+1}(i)}{\hat{\eta}_{1}(i)}$ for $i\in[N], k\in[K]$. The following algorithm extends the Ideal LCA-RSCORS to the real case.
\begin{algorithm}
\caption{\textbf{LCA-RSCORS}}
\label{alg:RSCORS}
\begin{algorithmic}[1]
\Require $R, K$, and $\tau$ (the default $\tau$ is $M\mathrm{max}(N,J)$.).
\Ensure $\hat{Z}$ and $\hat{\Theta}$.
\State Calculate the regularize Laplacian matrix $L_{\tau}$ by Equation (\ref{realL}).
\State Obtain $\hat{U}\hat{\Sigma}\hat{V}'$, the top $K$ SVD of $L_{\tau}$. Then calculate the $N\times (K-1)$ matrix $\hat{\Xi}$ from $\hat{U}$.
\State Run K-means algorithm on all rows of $\hat{\Xi}$ with $K$ clusters to obtain $\hat{Z}$.
\State Obtain $\hat{\Theta}$ using Steps 4-5 of Algorithm \ref{alg:RSC}.
\end{algorithmic}
\end{algorithm}

LCA-RSCORS has the same computational complexity as LCA-RSC and LCA-RSCn. Due to the row-wise ratios matrices $\Xi$ and $\hat{\Xi}$, it is complex to obtain the theoretical bounds of LCA-RSCORS. One possible way to obtain LCA-RSCORS's theoretical guarantees is to combine the theoretical analysis in this article with the proofs of the community detection algorithms provided in \citep{SCORE,wang2020spectral}. The theoretical framework of LCA-RSCORS is of independent interest.
\subsubsection{LCA-PCA}
We name the second alternative method as LCA-PCA, where PCA is short for principle component analysis. Recall that the rank of the $N\times J$ population response matrix $\mathscr{R}$ is also $K$. Without confusion, we let $\mathscr{R}=U\Sigma V'$ be the top-K SVD of $\mathscr{R}$ such that $U'U=I_{K\times K}$ and $V'V=I_{K\times K}$. Following a similar analysis to the 2nd statement of Lemma \ref{SVDPopulationLtau} gives $U=ZU(\mathcal{I},:)$,  which implies that running the K-means procedure to $U$'s rows returns the classification matrix $Z$. This suggests the following method called Ideal LCA-PCA. Input: $\mathscr{R}$ and $K$. Output: $Z$ and $\Theta$.
\begin{itemize}
  \item Obtain $U\Sigma V'$, the top $K$ SVD of $\mathscr{R}$.
  \item Run K-means algorithm on $U$'s rows with $K$ clusters to obtain $Z$.
  \item Recover $\Theta$ by setting $\Theta=\mathscr{R}'Z(Z'Z)^{-1}$.
\end{itemize}

Let $\hat{U}\hat{\Sigma}\hat{V}'$ be $R$'s top-K SVD. Algorithm \ref{alg:PCA} extends Ideal LCA-PCA to the real case immediately.
\begin{algorithm}
\caption{\textbf{LCA-PCA}}
\label{alg:PCA}
\begin{algorithmic}[1]
\Require $R$ and $K$.
\Ensure $\hat{Z}$ and $\hat{\Theta}$.
\State Obtain $\hat{U}\hat{\Sigma}\hat{V}'$, the top $K$ SVD of $R$.
\State Run K-means algorithm on $\hat{U}$'s rows with $K$ clusters to obtain $\hat{Z}$.
\State Obtain $\hat{\Theta}$ using Steps 4-5 of Algorithm \ref{alg:RSC}.
\end{algorithmic}
\end{algorithm}

Compared with LCA-RSC, LCA-RSCn, and LCA-RSCORS, we design LCA-PCA using the SVD of the observed response matrix instead of the regularized Laplacian matrix. Meanwhile, LCA-PCA has the same computational complexity as LCA-RSC, LCA-RSCn, and LCA-RSCORS. Following a similar analysis as Theorem \ref{mainRLCM} gets LCA-PCA's theoretical error rates, which are almost the same as that of LCA-RSC and we omit them here.
\subsubsection{LCA-RMK}
Because $\mathscr{R}=Z\Theta'$ under $\mathrm{LCM}(Z,\Theta)$, we have $\mathscr{R}=Z\mathscr{R}(\mathcal{I},:)$. Thus applying the K-means algorithm on $\mathscr{R}$'s rows with $K$ clusters returns $Z$. This gives the following algorithm called Ideal LCA-RMK, where RMK is short for response matrix by K-means. Input: $\mathscr{R}$ and $K$. Output: $Z$ and $\Theta$.
\begin{itemize}
  \item Run K-means algorithm on $\mathscr{R}$'s rows with $K$ clusters to obtain $Z$.
  \item Recover $\Theta$ by setting $\Theta=\mathscr{R}'Z(Z'Z)^{-1}$.
\end{itemize}

\begin{algorithm}
\caption{\textbf{LCA-RMK}}
\label{alg:RMK}
\begin{algorithmic}[1]
\Require $R$ and $K$.
\Ensure $\hat{Z}$ and $\hat{\Theta}$.
\State Run K-means algorithm on $R$'s rows with $K$ clusters to obtain $\hat{Z}$.
\State Obtain $\hat{\Theta}$ using Steps 4-5 of Algorithm \ref{alg:RSC}.
\end{algorithmic}
\end{algorithm}

Algorithm \ref{alg:RMK} called LCA-RMK extends the Ideal LCA-RMK to the real case naturally. Unlike LCA-RSC, LCA-RSCn, LCA-RSCORS, and LCA-PCA, there is no SVD step in LCA-RMK. The computational cost of LCA-RMK is $O(lNJK)$ with $l$ being the number of iterations in K-means. For the case $J=\beta N$, LCA-RMK's complexity is $O(\beta lKN^{2})$, which is slightly larger than $O(KN^{2})$ when $\beta l>1$. Therefore, LCA-RSC, LCA-RSCn, LCA-RSCORS, and LCA-PCA  run faster than LCA-RMK when $\beta l>1$. This conclusion is verified by the experimental results in this section.
\subsubsection{LCA-RLMK}
Similar to the Ideal-RMK, recall that $\mathcal{L}_{\tau}=\mathscr{D}^{-1/2}_{\tau}Z\Theta'$, we also have $\mathcal{L}_{\tau}=Z\mathcal{L}_{\tau}(\mathcal{I},:)$ by simple algebra, which suggests the following algorithm called Ideal LCA-RLMK, where RLMK is short for regularized Laplacian matrix by K-means. Input: $\mathscr{R},K$, and $\tau$. Output: $Z$ and $\Theta$.
\begin{itemize}
  \item Calculate $\mathcal{L}_{\tau}$ by Equation (\ref{PopulationL}).
  \item Run K-means algorithm on $\mathcal{L}_{\tau}$'s rows with $K$ clusters to obtain $Z$.
  \item Recover $\Theta$ by setting $\Theta=\mathscr{R}'Z(Z'Z)^{-1}$.
\end{itemize}

\begin{algorithm}
\caption{\textbf{LCA-RLMK}}
\label{alg:RLMK}
\begin{algorithmic}[1]
\Require $R, K$, and $\tau$ with default value $M\mathrm{max}(N,J)$.
\Ensure $\hat{Z}$ and $\hat{\Theta}$.
\State Calculate $L_{\tau}$ by Equation (\ref{realL}).
\State Run K-means algorithm on $L_{\tau}$'s rows with $K$ clusters to obtain $\hat{Z}$.
\State Obtain $\hat{\Theta}$ using Steps 4-5 of Algorithm \ref{alg:RSC}.
\end{algorithmic}
\end{algorithm}

Algorithm \ref{alg:RLMK} summarizes the real case of Ideal LCA-RLMK. This algorithm has the same computational complexity as LCA-RMK.
\section{A simple numerical example}
Here, we provide a toy example by generating only one response matrix. For this experiment, we set $K=2, M=3, N=16$, and $J=10$. In Figure  \ref{Ex6ZThetaRThetahat}, the 1st, 2nd, and 3rd matrices show $Z, \Theta$, and a response matrix $R$  generated from a Binomial distribution with expectation $Z\Theta'$ and 3 independent trials. Applying aforementioned methods to the observed response matrix $R$ (i.e., the 3rd matrix in Figure \ref{Ex6ZThetaRThetahat}) obtains their respective metrics, which are shown in Table \ref{ErrorRatesSimulatedR}. We see that our methods exactly recover $Z$ and exactly infer $K$ while KM performs poorer than our methods for this example. Our methods also return the same $\hat{\Theta}$ which is the 4th matrix in Figure \ref{Ex6ZThetaRThetahat}. Because $Z$ and $\Theta$ are provided in the first two matrices of Figure \ref{Ex6ZThetaRThetahat}, readers can run our algorithms or some other algorithms to the $R$ in the 3rd matrix of Figure \ref{Ex6ZThetaRThetahat} to verify their accuracies in estimating $Z, \Theta$, and $K$.
\begin{figure}
\centering
\includegraphics[width=1\textwidth]{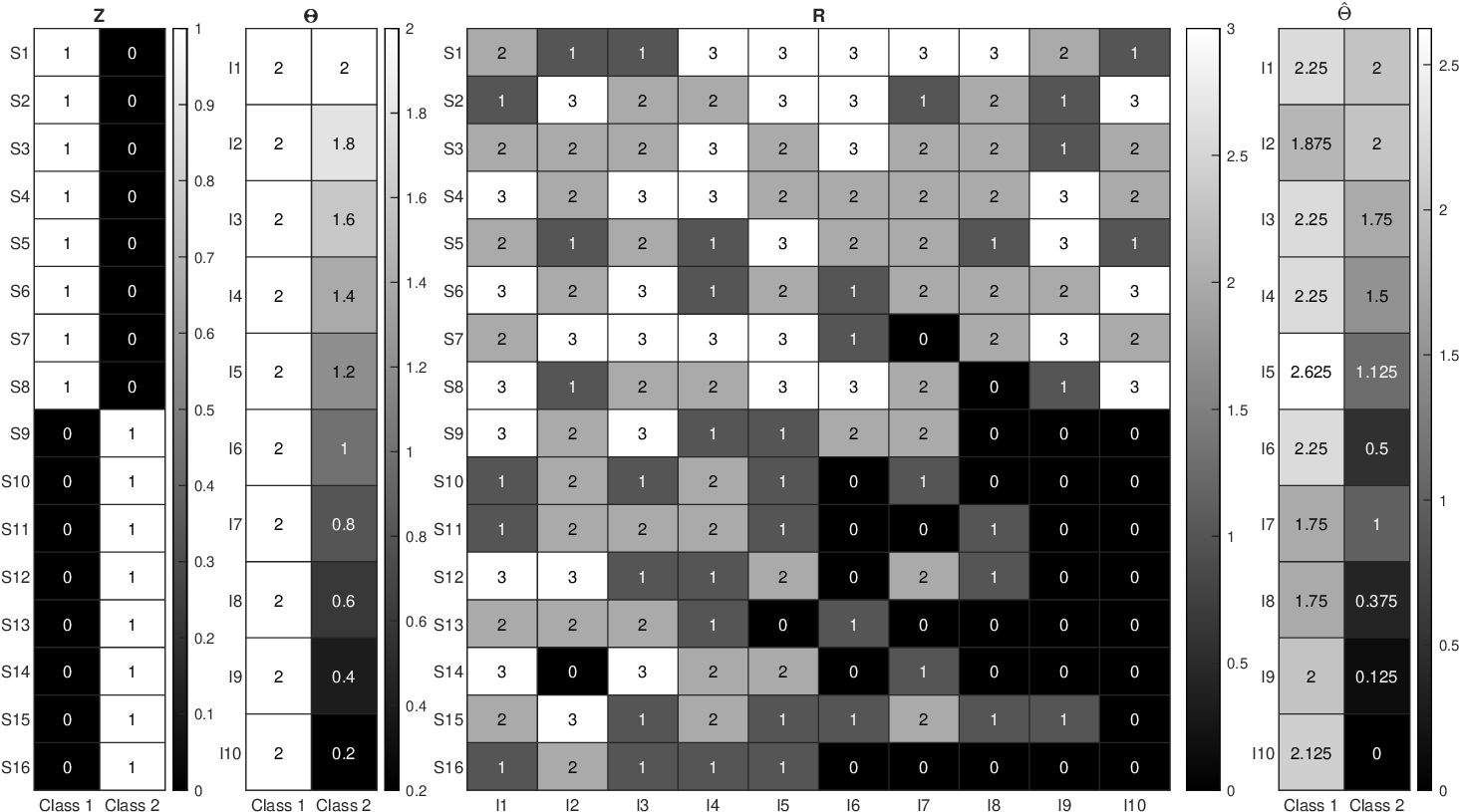}
\caption{The 1st and 2nd matrices are $Z$ and $\Theta$ for Experiment 6, respectively. The 3rd matrix is an observed response matrix $R$ generated from LCM using $Z$ and $\Theta$ in the first two matrices when the number of independent trials $M=3$ in a Binomial distribution. The last matrix is the estimated item parameter matrix $\hat{\Theta}$ returned by all algorithms. For all matrices, S$i$ represents subject $i$ for $i\in[16]$ and I$j$ represents item $j$ for $j\in[10]$.}
\label{Ex6ZThetaRThetahat} %% label for entire figure
\end{figure}

\begin{table}[h!]
\footnotesize
	\centering
	\caption{Metrics and $\hat{K}$ of our methods for $R$ in the 3rd matrix of Figure \ref{Ex6ZThetaRThetahat}, where $\hat{K}$ is the output of Equation (\ref{EstimateK}) for each method.}
	\label{ErrorRatesSimulatedR}
%	\resizebox{\columnwidth}{!}{
	\begin{tabular}{cccccccccccc}
\hline\hline&Clustering error&Hamming error&NMI&ARI&Relative $l_{1}$ error&Relative $l_{2}$ error&$\hat{K}$\\
\hline
LCA-RSC&0&0&1&1&0.1188&0.1550&2\\
LCA-RSCn&0&0&1&1&0.1188&0.1550&2\\
LCA-RSCORS&0&0&1&1&0.1188&0.1550&2\\
LCA-PCA&0&0&1&1&0.1188&0.1550&2\\
LCA-RMK&0&0&1&1&0.1188&0.1550&2\\
LCA-RLMK&0&0&1&1&0.1188&0.1550&2\\
KM&0.1250&0.0625&0.7210&0.7492&0.1286&0.1677&2\\
\hline\hline
\end{tabular}%}
\end{table}
\bibliographystyle{elsarticle-num}
\bibliography{refRSCLCA}

@article{hagenaars2002applied,
  title={Applied latent class analysis},
  author={Hagenaars, Jacques A and McCutcheon, Allan L},
  year={2002},
  journal={Cambridge University Press}
}

@article{sloane1996introduction,
  title={An introduction to categorical data analysis},
  author={Sloane, Douglas and Morgan, S Philip},
  journal={Annual Review of Sociology},
  volume={22},
  number={1},
  pages={351--375},
  year={1996},
  publisher={Annual Reviews 4139 El Camino Way, PO Box 10139, Palo Alto, CA 94303-0139, USA}
}

@article{agresti2012categorical,
  title={Categorical data analysis},
  author={Agresti, Alan},
  volume={792},
  year={2012},
  journal={John Wiley $\&$ Sons}
}

@article{chen2023spectral,
  title={A spectral method for identifiable grade of membership analysis with binary responses},
  author={Chen, Ling and Gu, Yuqi},
  journal={Psychometrika},
  volume={89},
  number={2},
  pages={626--657},
  year={2024},
  publisher={Springer}
}

@article{joseph2016impact,
	title="Impact of regularization on spectral clustering",
	author="Antony {Joseph} and Bin {Yu}",
	journal="Annals of Statistics",
	volume="44",
	number="4",
	pages="1765--1791",
	year="2016"
}

@article{tropp2012user,
	title="User-Friendly Tail Bounds for Sums of Random Matrices",
	author="Joel A. {Tropp}",
	journal="Foundations of Computational Mathematics",
	volume="12",
	number="4",
	pages="389--434",
	year="2012"
}

@article{zhou2019analysis,
	title="Analysis of spectral clustering algorithms for community detection: the general bipartite setting",
	author="Zhixin {Zhou} and Arash A. {Amini}",
	journal="Journal of Machine Learning Research",
	volume="20",
	number="47",
	pages="1--47",
	year="2019"
}

@article{strehl2002cluster,
  title={Cluster ensembles---a knowledge reuse framework for combining multiple partitions},
  author={Strehl, Alexander and Ghosh, Joydeep},
  journal={Journal of Mchine Learning Research},
  volume={3},
  number={Dec},
  pages={583--617},
  year={2002}
}

@article{danon2005comparing,
  title={Comparing community structure identification},
  author={Danon, Leon and Diaz-Guilera, Albert and Duch, Jordi and Arenas, Alex},
  journal={Journal of statistical mechanics: Theory and experiment},
  volume={2005},
  number={09},
  pages={P09008},
  year={2005},
  publisher={IOP Publishing}
}

@article{bagrow2008evaluating,
  title={Evaluating local community methods in networks},
  author={Bagrow, James P},
  journal={Journal of Statistical Mechanics: Theory and Experiment},
  volume={2008},
  number={05},
  pages={P05001},
  year={2008},
  publisher={IOP Publishing}
}

@article{luo2017community,
  title={Community detection by fuzzy relations},
  author={Luo, Wenjian and Yan, Zhenglong and Bu, Chenyang and Zhang, Daofu},
  journal={IEEE Transactions on Emerging Topics in Computing},
  volume={8},
  number={2},
  pages={478--492},
  year={2017},
  publisher={IEEE}
}

@article{hubert1985comparing,
  title={Comparing partitions},
  author={Hubert, Lawrence and Arabie, Phipps},
  journal={Journal of Classification},
  volume={2},
  pages={193--218},
  year={1985},
  publisher={Springer}
}

@inproceedings{vinh2009information,
  title={Information theoretic measures for clusterings comparison: is a correction for chance necessary?},
  author={Vinh, Nguyen Xuan and Epps, Julien and Bailey, James},
  booktitle={Proceedings of the 26th annual International Conference on Machine Learning},
  pages={1073--1080},
  year={2009}
}

@article{SCORE,
	title={{Fast community detection by SCORE}},
	author="Jiashun {Jin}",
	journal="Annals of Statistics",
	volume="43",
	number="1",
	pages="57--89",
	year="2015"
}

@article{xu2017aos,
author = {Gongjun Xu},
title = {{Identifiability of restricted latent class models with binary responses}},
volume = {45},
journal = {Annals of Statistics},
number = {2},
publisher = {Institute of Mathematical Statistics},
pages = {675 -- 707},
year = {2017},
}

@article{bakk2016robustness,
  title={Robustness of stepwise latent class modeling with continuous distal outcomes},
  author={Bakk, Zsuzsa and Vermunt, Jeroen K},
  journal={Structural Equation Modeling: A Multidisciplinary Journal},
  volume={23},
  number={1},
  pages={20--31},
  year={2016},
  publisher={Taylor \& Francis}
}

@article{chen2022beyond,
  title={Beyond the EM algorithm: constrained optimization methods for latent class model},
  author={Chen, Hao and Han, Lanshan and Lim, Alvin},
  journal={Communications in Statistics-Simulation and Computation},
  volume={51},
  number={9},
  pages={5222--5244},
  year={2022},
  publisher={Taylor \& Francis}
}

@article{woodbury1978mathematical,
  title={Mathematical typology: a grade of membership technique for obtaining disease definition},
  author={Woodbury, Max A and Clive, Jonathan and Garson Jr, Arthur},
  journal={Computers and Biomedical Research},
  volume={11},
  number={3},
  pages={277--298},
  year={1978},
  publisher={Elsevier}
}

@article{erosheva2005comparing,
  title={Comparing latent structures of the grade of membership, Rasch, and latent class models},
  author={Erosheva, Elena A},
  journal={Psychometrika},
  volume={70},
  number={4},
  pages={619--628},
  year={2005},
  publisher={Springer}
}

@article{lei2015consistency,
author = {Jing Lei and Alessandro Rinaldo},
title = {{Consistency of spectral clustering in stochastic block models}},
volume = {43},
journal = {Annals of Statistics},
number = {1},
publisher = {Institute of Mathematical Statistics},
pages = {215 -- 237},
year = {2015}
}

@article{mao2018overlapping,
  title={Overlapping clustering models, and one (class) svm to bind them all},
  author={Mao, Xueyu and Sarkar, Purnamrita and Chakrabarti, Deepayan},
  journal={Advances in Neural Information Processing Systems},
  volume={31},
  year={2018}
}

@article{newman2003mixing,
  title={Mixing patterns in networks},
  author={Newman, Mark EJ},
  journal={Physical Review E},
  volume={67},
  number={2},
  pages={026126},
  year={2003},
  publisher={APS}
}

@article{newman2006modularity,
  title={Modularity and community structure in networks},
  author={Newman, Mark EJ},
  journal={Proceedings of the National Academy of Sciences},
  volume={103},
  number={23},
  pages={8577--8582},
  year={2006},
  publisher={National Acad Sciences}
}

@article{newman2004finding,
  title={Finding and evaluating community structure in networks},
  author={Newman, Mark EJ and Girvan, Michelle},
  journal={Physical Review E},
  volume={69},
  number={2},
  pages={026113},
  year={2004},
  publisher={APS}
}

@article{nepusz2008fuzzy,
  title={Fuzzy communities and the concept of bridgeness in complex networks},
  author={Nepusz, Tam{\'a}s and Petr{\'o}czi, Andrea and N{\'e}gyessy, L{\'a}szl{\'o} and Bazs{\'o}, F{\"u}l{\"o}p},
  journal={Physical Review E},
  volume={77},
  number={1},
  pages={016107},
  year={2008},
  publisher={APS}
}

@article{wang2020spectral,
  title={Spectral algorithms for community detection in directed networks},
  author={Wang, Zhe and Liang, Yingbin and Ji, Pengsheng},
  journal={Journal of Machine Learning Research},
  volume={21},
  number={1},
  pages={6101--6145},
  year={2020},
  publisher={JMLRORG}
}

@article{qing2023community,
  title={Community detection for weighted bipartite networks},
  author={Qing, Huan and Wang, Jingli},
  journal={Knowledge-Based Systems},
  volume={274},
  pages={110643},
  year={2023},
  publisher={Elsevier}
}

@inproceedings{kunegis2013konect,
  title={Konect: the koblenz network collection},
  author={Kunegis, J{\'e}r{\^o}me},
  booktitle={Proceedings of the 22nd International Conference on World Wide Web},
  pages={1343--1350},
  year={2013}
}

@article{goodman1974exploratory,
  title={Exploratory latent structure analysis using both identifiable and unidentifiable models},
  author={Goodman, Leo A},
  journal={Biometrika},
  volume={61},
  number={2},
  pages={215--231},
  year={1974},
  publisher={Oxford University Press}
}

@article{garrett2000latent,
  title={Latent class model diagnosis},
  author={Garrett, Elizabeth S and Zeger, Scott L},
  journal={Biometrics},
  volume={56},
  number={4},
  pages={1055--1067},
  year={2000},
  publisher={Wiley Online Library}
}

@inproceedings{asparouhov2011using,
  title={Using {Bayesian} priors for more flexible latent class analysis},
  author={Asparouhov, Tihomir and Muth{\'e}n, Bengt},
  booktitle={Proceedings of the 2011 Joint Statistical Meeting, Miami Beach, FL},
  year={2011},
  organization={American Statistical Association Alexandria, VA}
}

@article{white2014bayeslca,
  title={Bayes{LCA: An R package for Bayesian} latent class analysis},
  author={White, Arthur and Murphy, Thomas Brendan},
  journal={Journal of Statistical Software},
  volume={61},
  pages={1--28},
  year={2014}
}

@article{li2018bayesian,
  title={Bayesian latent class analysis tutorial},
  author={Li, Yuelin and Lord-Bessen, Jennifer and Shiyko, Mariya and Loeb, Rebecca},
  journal={Multivariate Behavioral Research},
  volume={53},
  number={3},
  pages={430--451},
  year={2018},
  publisher={Taylor \& Francis}
}

@article{zeng2023tensor,
  title={A {Tensor-EM Method for Large-Scale Latent Class Analysis with Binary Responses}},
  author={Zeng, Zhenghao and Gu, Yuqi and Xu, Gongjun},
  journal={Psychometrika},
  volume={88},
  number={2},
  pages={580--612},
  year={2023},
  publisher={Springer}
}

@article{vermunt2004latent,
  title={Latent class analysis},
  author={Vermunt, Jeroen K and Magidson, Jay},
  journal={The SAGE Encyclopedia of Social Sciences Research Methods},
  volume={2},
  pages={549--553},
  year={2004},
  publisher={Sage Publications Thousand Oakes, CA}
}

@article{nylund2018ten,
  title={Ten frequently asked questions about latent class analysis.},
  author={Nylund-Gibson, Karen and Choi, Andrew Young},
  journal={Translational Issues in Psychological Science},
  volume={4},
  number={4},
  pages={440},
  year={2018},
  publisher={Educational Publishing Foundation}
}

@article{lanza2016latent,
  title={Latent class analysis for developmental research},
  author={Lanza, Stephanie T and Cooper, Brittany R},
  journal={Child Development Perspectives},
  volume={10},
  number={1},
  pages={59--64},
  year={2016},
  publisher={Wiley Online Library}
}

@article{mccutcheon1987latent,
  title={Latent class analysis},
  author={McCutcheon, Allan L},
  volume={64},
  year={1987},
  journal={Sage}
}

@article{lanza2013latent,
  title={Latent class analysis: an alternative perspective on subgroup analysis in prevention and treatment},
  author={Lanza, Stephanie T and Rhoades, Brittany L},
  journal={Prevention Science},
  volume={14},
  number={2},
  pages={157--168},
  year={2013},
  publisher={Springer}
}

@article{ng2001spectral,
  title={On spectral clustering: Analysis and an algorithm},
  author={Ng, Andrew and Jordan, Michael and Weiss, Yair},
  journal={Advances in Neural Information Processing Systems},
  volume={14},
  year={2001}
}

@article{von2007tutorial,
  title={A tutorial on spectral clustering},
  author={Von Luxburg, Ulrike},
  journal={Statistics and Computing},
  volume={17},
  pages={395--416},
  year={2007},
  publisher={Springer}
}

@article{qin2013regularized,
  title={Regularized spectral clustering under the degree-corrected stochastic blockmodel},
  author={Qin, Tai and Rohe, Karl},
  journal={Advances in Neural Information Processing Systems},
  volume={26},
  year={2013}
}

@article{KSB2011,
author = {Karl Rohe and Sourav Chatterjee and Bin Yu},
title = {{Spectral clustering and the high-dimensional stochastic blockmodel}},
volume = {39},
journal = {Annals of Statistics},
number = {4},
publisher = {Institute of Mathematical Statistics},
pages = {1878 -- 1915},
year = {2011}
}

@article{rohe2016co,
  title={Co-clustering directed graphs to discover asymmetries and directional communities},
  author={Rohe, Karl and Qin, Tai and Yu, Bin},
  journal={Proceedings of the National Academy of Sciences},
  volume={113},
  number={45},
  pages={12679--12684},
  year={2016},
  publisher={National Acad Sciences}
}

@article{huang1998extensions,
  title={Extensions to the k-means algorithm for clustering large data sets with categorical values},
  author={Huang, Zhexue},
  journal={Data Mining and Knowledge Discovery},
  volume={2},
  number={3},
  pages={283--304},
  year={1998},
  publisher={Springer}
}

@article{qing2023regularized,
  title={Regularized spectral clustering under the mixed membership stochastic block model},
  author={Qing, Huan and Wang, Jingli},
  journal={Neurocomputing},
  volume={550},
  pages={126490},
  year={2023},
  publisher={Elsevier}
}

@article{qing2025community,
  title={Community detection in multi-layer networks by regularized debiased spectral clustering},
  author={Qing, Huan},
  journal={Engineering Applications of Artificial Intelligence},
  volume={152},
  pages={110627},
  year={2025},
  publisher={Elsevier}
}

@article{qing2025discovering,
  title={Discovering overlapping communities in multi-layer directed networks},
  author={Qing, Huan},
  journal={Chaos, Solitons $\&$ Fractals},
  volume={194},
  pages={116175},
  year={2025},
  publisher={Elsevier}
}

@article{lyu2025degree,
  title={Degree-heterogeneous Latent Class Analysis for high-dimensional discrete data},
  author={Lyu, Zhongyuan and Chen, Ling and Gu, Yuqi},
  journal={Journal of the American Statistical Association},
  pages={1--14},
  year={2025},
  publisher={Taylor \& Francis}
}

@article{al2023enhancement,
  title={Enhancement of K-means clustering in big data based on equilibrium optimizer algorithm},
  author={Al-Kababchee, Sarah Ghanim Mahmood and Algamal, Zakariya Yahya and Qasim, Omar Saber},
  journal={Journal of Intelligent Systems},
  volume={32},
  number={1},
  pages={20220230},
  year={2023},
  publisher={De Gruyter}
}

@article{lei2016goodness,
  title={A goodness-of-fit test for stochastic block models},
  author={Lei, Jing},
  journal={Annals of Statistics},
  volume={44},
  number={1},
  pages={401--424},
  year={2016},
  publisher={Institute of Mathematical Statistics}
}

@article{wu2024spectral,
  title={A spectral based goodness-of-fit test for stochastic block models},
  author={Wu, Qianyong and Hu, Jiang},
  journal={Statistics \& Probability Letters},
  volume={209},
  pages={110104},
  year={2024},
  publisher={Elsevier}
}

@article{qing2025joint,
  title={Joint estimation of asymmetric community numbers in directed networks},
  author={Qing, Huan},
  journal={arXiv preprint arXiv:2508.14816},
  year={2025}
}

@article{holland1983stochastic,
  title={{Stochastic blockmodels: First steps}},
  author={Holland, Paul W and Laskey, Kathryn Blackmond and Leinhardt, Samuel},
  journal={Social Networks},
  volume={5},
  number={2},
  pages={109--137},
  year={1983},
  publisher={Elsevier}
}

@article{qing2024grade,
  title={Grade of Membership Analysis for Multi-Layer Ordinal Categorical Data},
  author={Qing, Huan},
  journal={Statistica Sinica},
  volume={38},
  number={2},
  year={2024}
}

@article{qing2025mixed,
  title={Mixed membership estimation for categorical data with weighted responses},
  author={Qing, Huan},
  journal={TEST},
  pages={1--48},
  year={2025},
  publisher={Springer}
}
\end{document}